\relax
\documentclass{article} %DO NOT CHANGE THIS
\usepackage[accepted]{icml2019}  %Required
\usepackage{times}  %Required
\usepackage{helvet}  %Required
\usepackage{courier}  %Required
\usepackage{url,stmaryrd}  %Required
\usepackage{graphicx}  %Required
\usepackage{amsfonts}
\usepackage{bigcenter}
\usepackage{amsthm}
\usepackage{dsfont}
\usepackage{bm}
\usepackage{mathtools}
\usepackage{float}
\usepackage{times}
\usepackage[algo2e]{algorithm2e}
\usepackage[subrefformat=parens,labelformat=parens,caption=false]{subfig}
\usepackage{tikz}
\usetikzlibrary{arrows,automata}
\tikzset{terminal state/.style={draw,rectangle,minimum size=.3in}}
\usepackage{enumitem}
\DeclareMathOperator*{\argmax}{argmax}

\tikzstyle{startstop} = [rectangle, rounded corners, minimum width=3cm, minimum height=1cm,text centered, draw=black, fill=red!30]
\tikzstyle{io} = [trapezium, trapezium left angle=70, trapezium right angle=110, minimum width=3cm, minimum height=1cm, text centered, draw=black, fill=blue!30]
\tikzstyle{process} = [rectangle, minimum width=3cm, minimum height=1cm, text centered, draw=black, fill=orange!30]
\tikzstyle{decision} = [diamond, minimum width=3cm, minimum height=1cm, text centered, draw=black, fill=green!30]
\tikzstyle{arrow} = [thick,->,>=stealth]

\newtheorem{theorem}{Theorem}%[section]
\newtheorem{proposition}{Proposition}

\newtheorem{lemma}{Lemma}

\floatname{algorithm}{Algorithm}
\newenvironment{pseudocode}[1][htb]
{% Update algorithm name
	\begin{algorithm}[#1]%
	}{\end{algorithm}}

\icmltitlerunning{Safe Policy Improvement with Baseline Bootstrapping}

\begin{document}

\twocolumn[
\icmltitle{Safe Policy Improvement with Baseline Bootstrapping}



\icmlsetsymbol{equal}{*}

\begin{icmlauthorlist}
\icmlauthor{Romain Laroche}{to}
\icmlauthor{Paul Trichelair}{to}
\icmlauthor{Remi Tachet des Combes}{to}
\end{icmlauthorlist}

\icmlaffiliation{to}{Microsoft Research, Montr\'eal, Canada}
% \icmlaffiliation{goo}{OWKIN, France}

\icmlcorrespondingauthor{Romain Laroche}{romain.laroche@microsoft.com}

\icmlkeywords{Reinforcement Learning, Batch RL, Safe Policy Improvement}

\vskip 0.3in
]



\printAffiliationsAndNotice{}  % leave blank if no need to mention equal contribution

	\begin{abstract}
	    This paper considers Safe Policy Improvement (SPI) in Batch Reinforcement Learning (Batch RL): from a fixed dataset and without direct access to the true environment, train a policy that is guaranteed to perform at least as well as the baseline policy used to collect the data. 
	    Our approach, called SPI with Baseline Bootstrapping (SPIBB), is inspired by the knows-what-it-knows paradigm: it bootstraps the trained policy with the baseline when the uncertainty is high. 
	    Our first algorithm, $\Pi_b$-SPIBB, comes with SPI theoretical guarantees. 
	    We also implement a variant, $\Pi_{\leq b}$-SPIBB, that is even more efficient in practice. 
	    We apply our algorithms to a motivational stochastic gridworld domain and further demonstrate on randomly generated MDPs the superiority of SPIBB with respect to existing algorithms, not only in safety but also in mean performance. 
	    Finally, we implement a model-free version of SPIBB and show its benefits on a navigation task with deep RL implementation called SPIBB-DQN, which is, to the best of our knowledge, the first RL algorithm relying on a neural network representation able to train efficiently and reliably from batch data, without any interaction with the environment.
	\end{abstract}

	\section{Introduction}
	\label{sec:introduction}
% 	Reinforcement Learning (RL, \cite{Sutton1998}) consists in discovering through \textit{trial-and-error} in an unknown uncertain environment, which action is the most valuable in a particular situation. This \textit{optimism in the face of uncertainty} (OFU~\cite{Szita2008}), or at the very least, this recklessness in an online learning setting: algorithms based on OFU, such as R-MAX or UCRL, are provably efficient. Indeed, a good outcome brings a policy improvement, and even an error leads to learning not to do it again at a lesser cost~\cite{brafman2002r,Auer2007}. However, most real-world algorithms are to be widely deployed on independent devices/systems, and as such their policies cannot be updated as often as online learning would require. Therefore, the same mistake may be repeated during a time long enough to lose the user's trust or to cause irreversible damages. In this offline setting, batch RL algorithms are one approach that has been recommended~\cite{Lange2012}. But, the OFU paradigm shows its limits when the policy updates are rare, because the commitment to optimism is too strong and the error impact may be severe. In this paper, we endeavour to build batch RL algorithms that are safe in this regard.

	Most real-world Reinforcement Learning agents \citep[RL]{Sutton1998} are to be deployed simultaneously on numerous independent devices and cannot be patched quickly. In other practical applications, such as crop management or clinical tests, the outcome of a treatment can only be assessed after several years. Consequently, a bad update could be in effect for a long time, potentially hurting the user's trust and/or causing irreversible damages. Devising safe algorithms with guarantees on the policy performance is a key challenge of modern RL that needs to be tackled before any wide-scale adoption.
	
	Batch RL is an existing approach to such offline settings and consists in training a policy on a fixed set of observations without access to the true environment~\cite{Lange2012}. It should not be mistaken with the multi-batch setting where the learner trains successive policies from small batches of interactions with the environment~\cite{Duan2016}. Current Batch RL algorithms are however either unsafe or too costly computationally to be used in real-world applications. Safety in RL \cite{garcia2015comprehensive} is an overloaded term, as it may be considered with respect to parametric uncertainty~\cite{thomas2015high,ghavamzadeh2016safe}, internal uncertainty~\cite{Altman1999,Carrara2019}, interruptibility~\cite{Orseau2016,Guerraoui2017}, or as exploration in a hazardous environment \cite{Schulman2015,Schulman2017,Fatemi2019}. We focus on the former.
	
% 	Safety in RL has been defined in several contexts \cite{garcia2015comprehensive} and traditionally aims at dealing with internal or parametric uncertainty. Following \citet{Ghavamzadeh2016bayes}, \textit{internal uncertainty} reflects the uncertainty due to stochastic transitions and rewards in a known environment, while \textit{parametric uncertainty} reflects the uncertainty about an unknown environment.
	
	In this paper, we develop novel \textit{safe and efficient} Batch RL algorithms. Our methodology for Safe Policy Improvement (SPI), called SPI with Baseline Bootstrapping (SPIBB), is introduced in Section~\ref{sec:spibb}. It consists in bootstrapping the trained policy with the behavioral policy, called \textit{baseline}, in the state-action pair transitions that were not probed enough in the dataset. It therefore assumes access to the baseline, an assumption already made in the SPI literature~\cite{ghavamzadeh2016safe}. Other SPI algorithms assume knowledge of the baseline performance instead~\cite{thomas2015high,thomas2015high2}. We argue that our assumption is more natural since SPI aims to improve an existing policy. This scenario is typically encountered when a policy is trained in a simulator and then run in its real environment, for instance in Transfer RL~\cite{Taylor2009}; or when a system is designed with expert knowledge and then optimized, for example in Dialogue applications~\cite{Laroche2010}.
	
    Still in Section \ref{sec:spibb}, we implement a computationally efficient algorithm, $\Pi_b$-SPIBB, that provably approximately outperforms the baseline with high confidence. At the expense of theoretical guarantees, we design a variant, $\Pi_{\leq b}$-SPIBB, that is even more efficient in practice. Moreover, we implement an equivalent model-free version. Coupled with a pseudo-count implementation~\cite{Bellemare2016}, it allows applying SPIBB algorithms to tasks requiring a neural network representation. Finally, we position our algorithm with respect to competing SPI algorithms found in the literature. 
    
    Then, in Section~\ref{sec:results}, we motivate our approach on a small stochastic gridworld domain and further demonstrate on randomly generated MDPs the superiority of SPIBB compared to existing algorithms, not only in safety but also in mean performance. Furthermore, we apply the model-free version to a continuous navigation task. It is, to the best of our knowledge, the first RL algorithm relying on a neural network representation able to train efficiently and reliably from batch data, without any interaction with the environment~\cite{Duan2016}.

    Finally, Section~\ref{sec:conclusion} concludes the paper. The appendix includes the proofs, thorough experiment details, and the complete results of experiments. The code may be found at \url{https://github.com/RomainLaroche/SPIBB} and \url{https://github.com/rems75/SPIBB-DQN}.

	\section{SPI with Baseline Bootstrapping}
	\label{sec:spibb}
% 	\subsection{Notations}
	A proper introduction to Markov Decision Processes \citep[MDPs]{bellman1957markovian} and Reinforcement Learning \citep[RL]{Sutton1998} is available in Appendix~\ref{sup:MDP}. Due to space constraint, we only define our notations here.
	
	An MDP is denoted by $M=\langle \mathcal{X}, \mathcal{A},R,P,\gamma \rangle$, where $\mathcal{X}$ is the state space, $\mathcal{A}$ is the action space, $R^*(x,a)\in[-R_{max},R_{max}]$ is the bounded stochastic reward function, $P^*(\cdot|x,a)$ is the transition distribution, and $\gamma\in[0,1)$ is the discount factor. The true environment is modelled as an unknown finite MDP $M^*=\langle \mathcal{X}, \mathcal{A},R^*,P^*,\gamma \rangle$ with $R^*(x,a)\in[-R_{max},R_{max}]$. $\Pi=\{\pi:\mathcal{X}\rightarrow \Delta_{\mathcal{A}}\}$ is the set of stochastic policies, where $\Delta_{\mathcal{A}}$ denotes the set of probability distributions over the set of actions $\mathcal{A}$.
	
	The state and state-action value functions are respectively denoted by $V_M^\pi(x)$ and $Q_M^\pi(x,a)$.	We define the performance of a policy by its expected return, starting from the initial state $x_0$: $\rho(\pi,M) = V^\pi_M(x_0)$. Given a policy subset $\Pi'\subseteq\Pi$, a policy $\pi'$ is said to be $\Pi'$-optimal for an MDP $M$ when it maximizes its performance on $\Pi'$: $\rho(\pi',M)=\max_{\pi\in\Pi'}\rho(\pi,M)$. We will also make use of the notation $V_{max}$ as a known upper bound of the return's absolute value: $V_{max}\leq\frac{R_{max}}{1-\gamma}$.
	
	In this paper, we focus on the batch RL setting where the algorithm does its best at learning a policy from a fixed set of experience. Given a dataset of transitions $\mathcal{D}=\langle x_j,a_j,r_j,x'_j\rangle_{j\in\llbracket 1,N\rrbracket}$, we denote by $N_{\mathcal{D}}(x,a)$ the state-action pair counts; and by $\widehat{M}=\langle \mathcal{X}, \mathcal{A},\widehat{R},\widehat{P},\gamma \rangle$ the Maximum Likelihood Estimation (MLE) MDP of the environment, where $\widehat{R}$ is the reward mean and $\widehat{P}$ is the transition statistics observed in the dataset. Vanilla batch RL, referred hereinafter as Basic RL, looks for the optimal policy in $\widehat{M}$. This policy may be found indifferently using dynamic programming on the explicitly modelled MDP $\widehat{M}$, $Q$-learning with experience replay until convergence~\cite{Sutton1998}, or Fitted-$Q$ Iteration with a one-hot vector representation of the state space~\cite{ernst2005tree}.
	
	\subsection{Percentile criterion and Robust MDPs}
	\label{sec:percentile}
	We start from the \textit{percentile criterion}~\cite{Delage2010} on the safe policy improvement over the baseline $\pi_b$:
	\begin{align}
	&\pi_C = \argmax_{\pi\in\Pi} \mathbb{E}\left[\rho(\pi,M) \,\middle\vert\, M \sim \mathbb{P}_\textsc{mdp}(\cdot|\mathcal{D})\right], \label{eq:percentilecriterion} \\
	\textnormal{s.t. } &\mathbb{P}\left(\rho(\pi,M)\geq \rho(\pi_b,M) - \zeta \,\middle\vert\, M \sim \mathbb{P}_\textsc{mdp}(\cdot|\mathcal{D})\right)\geq 1-\delta, \nonumber	
	\end{align}
	where $\mathbb{P}_\textsc{mdp}(\cdot|\mathcal{D})$ is the posterior probability of the MDP parameters, $1-\delta$ is the high probability meta-parameter, and $\zeta$ is the approximation meta-parameter. \citep{ghavamzadeh2016safe} use Robust MDP~\cite{Iyengar2005,Nilim2005} to bound from below the constraint in \eqref{eq:percentilecriterion} by considering a set of admissible MDPs $\Xi=\Xi(\widehat{M},e)$ defined as:
	\begin{align}
	\Xi(\widehat{M},e) := &\left\{ M = \langle \mathcal{X}, \mathcal{A}, R, P, \gamma\rangle \quad\textnormal{s.t. } \forall(x,a)\in\mathcal{X}\times\mathcal{A}, \right. \nonumber\\
	&\left. \begin{array}{ll}
    	||P(\cdot|x,a)-\widehat{P}(\cdot|x,a)||_1 \leq e(x,a), \\
    	|R(x,a)-\widehat{R}(x,a)| \leq e(x,a)R_{max} 
	\end{array}
	\right\}
	\end{align}
	where $e:\mathcal{X}\times\mathcal{A}\rightarrow\mathbb{R}$ is an error function depending on $\mathcal{D}$ and $\delta$.
	In place of the intractable expectation in Equation~\eqref{eq:percentilecriterion}, Robust MDP classically consider optimizing the policy performance $\rho(\pi,M)$ of the worst-case scenario in $\Xi$:
	\begin{align}
	\pi_R = \argmax_{\pi\in\Pi} \min_{M\in\Xi} \rho(\pi,M). \label{eq:robustcriterion}
	\end{align}
	
	In our benchmarks, we use the Robust MDP solver described in \citet{ghavamzadeh2016safe}. \citet{ghavamzadeh2016safe} also contemplate the policy improvement worst-case scenario:
	\begin{align}
	\pi_S = \argmax_{\pi\in\Pi} \min_{M\in\Xi} \left(\rho(\pi,M) - \rho(\pi_b,M)\right). \label{eq:SPIcriterion}
	\end{align}
	
	They prove that this optimization is an NP-hard problem and propose an algorithm approximating the solution without any formal proof: Approximate Robust Baseline Regret Minimization (ARBRM). There are three problems with ARBRM. First, it assumes that there is no error in the transition probabilities of the baseline, which is very restrictive and amounts to Basic RL when the support of the baseline is the full action space in each state (as is the case in all our experiments). Second, given its high complexity, it is difficult to empirically assess its percentile criterion safety except on simple tasks. Third, in order to retain safety guarantees, ARBRM requires a conservative safety test that consistently fails in our experiments. These are the reasons why our benchmarks do not include ARBRM.
	
	\RestyleAlgo{boxruled}
	\setlength{\algomargin}{2pt}
	\subsection{SPIBB methodology}
	In this section, we reformulate the percentile criterion to make searching for an efficient and provably-safe policy tractable in terms of computer time. Our new criterion consists in optimizing the policy with respect to its performance in the MDP estimate $\widehat{M}$, while guaranteeing it to be $\zeta$-approximately at least as good as $\pi_b$ in the admissible MDP set $\Xi$. Formally, we write it as follows:
	\begin{align}
	\max_{\pi\in\Pi} \rho(\pi,\widehat{M}), \textnormal{ s.t. } \forall M\in\Xi, \rho(\pi,M) \geq \rho(\pi_b,M) - \zeta.	\label{eq:criterion}
	\end{align}
	
	From Theorem 8 of \citet{ghavamzadeh2016safe}, it is direct to guarantee that, if all the state-action pair counts satisfy:
	\begin{align}
	    N_\mathcal{D}(x,a) \geq N_{\wedge} = \frac{8V_{max}^2}{\zeta^2(1-\gamma)^2}\log\frac{2\lvert\mathcal{X}\rvert\lvert\mathcal{A}\rvert2^{|\mathcal{X}|}}{\delta},  \label{eq:petrikconstraint}
	\end{align}
	and if $\widehat{M}$ is the Maximum Likelihood Estimation (MLE) MDP, then, with high probability $1-\delta$, the optimal policy $\pi^\odot = \argmax_{\pi\in\Pi}\rho(\pi,\widehat{M})$ in $\widehat{M}$ is $\zeta$-approximately safe with respect to the true environment $M^*$:
	\begin{align}
	    \rho(\pi^\odot,M^*) \geq \rho(\pi^*,M^*) - \zeta \geq \rho(\pi_b,M^*) - \zeta.
	\end{align}

	In the following, we extend this result by allowing constraint \eqref{eq:petrikconstraint} to be violated on a subset of the state-action pairs $\mathcal{X}\times\mathcal{A}$, called the bootstrapped set and denoted by $\mathfrak{B}$. $\mathfrak{B}$ is the set of state-action pairs with counts smaller than $N_\wedge$.
	
	\subsection{$\Pi_b$-SPIBB}
	\label{sec:pibootstrap}
	In this section, we develop two novel algorithms based on policy bootstrapping and prove associated SPI bounds. More precisely, when a state-action pair $(x,a)$ is rarely seen in the dataset, we propose to rely on the baseline by copying its probability to take action $a$: 
	\begin{align}
	    \pi^\odot_{spibb}(a|x)=\pi_b(a|x)\;\textnormal{if}\;(x,a)\in\mathfrak{B}. \label{eq:policyiterationconstraint}
	\end{align}
	
	We let $\Pi_b$ denote the set of policies that verify \eqref{eq:policyiterationconstraint} for all state-action pairs. Our first algorithm, coined $\Pi_b$-SPIBB, consists in the usual policy optimization of the expected return $\rho(\pi, \widehat{M})$ under constraint \eqref{eq:policyiterationconstraint}. In practice, it may be achieved in a model-based manner by explicitly computing the MDP model $\widehat{M}$, constructing the set of allowed policies $\Pi_b$ and finally searching for the $\Pi_b$-optimal policy $\pi^\odot_{spibb}$ in $\widehat{M}$ using policy iteration over $\Pi_b$~\cite{Howard1966,Puterman1979}. In the policy evaluation step, the current policy $\pi^{(i)}$ is evaluated as $Q^{(i)}_{\widehat{M}}$. In the policy improvement step, $\pi^{(i+1)}$ is defined as the greedy policy with respect to $Q^{(i)}$ under the constraint of belonging to $\Pi_b$ (Algorithm~\ref{alg:Pibproj} describes how to enforce this constraint in linear time).

	\begin{pseudocode}[ht!]
		\caption{Greedy projection of $Q^{(i)}$ on $\Pi_{b}$}
		\KwIn{Baseline policy $\pi_{b}$}
		\KwIn{Last iteration value function $Q^{(i)}$}
		\KwIn{Set of bootstrapped state-action pairs $\mathfrak{B}$}
		\KwIn{Current state $x$ and action set $\mathcal{A}$}
		\BlankLine
		Initialize $\pi^{(i)}_{spibb} = 0$
		
		\lFor{$(x,a) \in \mathfrak{B}$}{
			$\pi^{(i)}_{spibb}(a|x) = \pi_b(a|x)$
		}
		
		$\displaystyle\pi^{(i)}_{spibb}\left(x,\argmax_{a|(x,a)\notin\mathfrak{B}}Q^{(i)}(x,a)\right) = \sum_{a|(x,a)\notin\mathfrak{B}} \pi_b(a|x)$
		
		\Return $\pi^{(i)}_{spibb}$
		\label{alg:Pibproj}
	\end{pseudocode}
	
	The following theorems prove that $\Pi_b$-SPIBB converges to a $\Pi_b$-optimal policy $\pi^\odot_{spibb}$, and that $\pi^\odot_{spibb}$ is a safe policy improvement over the baseline in the true MDP $M^*$.
	
	\begin{theorem}[Convergence]
		$\Pi_b$-SPIBB converges to a policy $\pi^\odot_{spibb}$ that is $\Pi_b$-optimal in the MLE MDP $\widehat{M}$.
		\label{th:pib-spibb-convergence}
	\end{theorem}
	\begin{theorem}[Safe policy improvement]
		Let $\Pi_b$ be the set of policies under the constraint of following $\pi_b$ when $(x,a)\in\mathfrak{B}$. Then, $\pi^\odot_{spibb}$ is a $\zeta$-approximate safe policy improvement over the baseline $\pi_b$ with high probability $1-\delta$, where:
		\begin{equation*}
		\zeta = \cfrac{4 V_{max}}{1-\gamma} \sqrt{\cfrac{2}{N_{\wedge}}\log\cfrac{2|\mathcal{X}||\mathcal{A}|2^{|\mathcal{X}|}}{\delta}}  - \rho(\pi^\odot_{spibb}, \widehat{M}) + \rho(\pi_b, \widehat{M})
		\end{equation*}
		\label{th:safepolicyimprovement-pi}
	\end{theorem}
	
	Proofs of both theorems are available in Appendix \ref{sup:th-proofs}. Theorem \ref{th:pib-spibb-convergence} is a direct application of the classical policy iteration theorem. Theorem \ref{th:safepolicyimprovement-pi} is a generalization of Theorem~8 in \citet{ghavamzadeh2016safe}. The resulting bounds may look very similar at first. The crucial difference is that, in our case, $N_\wedge$ is not a property of the dataset, but a hyper-parameter of the algorithm. In all our experiments, $\lVert e\rVert_\infty$ from Theorem 8 would be equal to 2, leading to a trivial bound. In comparison, $\Pi_b$-SPIBB allows safe improvement if $N_\wedge$ is chosen large enough to ensure safety and small enough to ensure improvement.
	
	SPIBB takes inspiration from \citet{ghavamzadeh2016safe}'s idea of finding a policy that is guaranteed to be an improvement for any realization of the uncertain parameters, and similarly estimates the error on those parameters, as a function of the state-action pair counts. But instead of searching for the analytic optimum, SPIBB looks for a solution that improves the baseline when it can guarantee improvement and falls back on the baseline when the uncertainty is too high. One can see it as a \textit{knows-what-it-knows} algorithm, asking for help from the baseline when it \textit{does not know whether it knows}~\cite{li2008knows}. As such, our algorithms can be seen as pessimistic, the flip side of \textit{optimism in the face of uncertainty}~\cite{Szita2008}. As a consequence, $\Pi_b$-SPIBB is not optimal with respect to the criterion in Equation \eqref{eq:criterion}. But in return, it is inherently safe as it only allows to search in a set of policies for which the improvement over the baseline can be safely evaluated~\cite{Thomas2017}. It is also worth mentioning that SPIBB is computationally simple, which allows us to develop the SPIBB-DQN algorithm in the next section.
	
	\subsection{Model-free $\Pi_b$-SPIBB and SPIBB-DQN}
	The $\Pi_b$-SPIBB policy optimization may indifferently be achieved in a model-free manner by fitting the $Q$-function to the following target $y^{(i)}_j$ over the transition samples in the dataset $\mathcal{D}=\langle x_j,a_j,r_j,x'_j\rangle_{j\in\llbracket 1,N\rrbracket}$:
	\begin{align}
	   % y^{(i)}_j &= r_j + \gamma \max_{\pi\in\Pi_b} \sum_{a'\in\mathcal{A}} \pi(a'|x_j') Q^{(i)}(x_j',a')\\
	    y^{(i)}_j &= r_j + \gamma\sum_{a'|(x_j',a')\in\mathfrak{B}} \pi_b(a'|x_j') Q^{(i)}(x_j',a') \label{eq:spibb-DQN} \\
	    &+ \gamma\left(\sum_{a'|(x_j',a')\notin\mathfrak{B}} \pi_b(a'|x_j')\right) \max_{a'|(x_j',a')\notin\mathfrak{B}} Q^{(i)}(x_j',a') \nonumber
	\end{align}
	
	The first term $r_j$ is the immediate reward observed during the recorded transition, the second term is the return estimate of the bootstrapped actions (where the trained policy is constrained to the baseline policy), and the third term is the return estimate maximized over the non-bootstrapped actions. SPIBB-DQN is the DQN algorithm fitted to these targets $y^{(i)}_j$~\cite{Mnih2015}. Note that computing the SPIBB targets requires determining the bootstrapped set $\mathfrak{B}$, which relies on an estimate of the state-action counts $\widetilde{N}_{\mathcal{D}}(x,a)$, also called pseudo-counts~\cite{Bellemare2016,Fox2018,Burda2019}.
	
	\begin{theorem}
	    In finite MDPs, Equation \ref{eq:spibb-DQN} admits a unique fixed point that coincides with the $Q$-value of the policy trained with model-based $\Pi_b$-SPIBB.
	\end{theorem}
	
	\subsection{$\Pi_{\leq b}$-SPIBB} 
	In our empirical evaluation, we consider a variant of $\Pi_b$-SPIBB: the space of policies to search is relaxed to $\Pi_{\leq b}$, the set of policies that do not to give more weight than $\pi_b$ to bootstrapped actions. As a consequence, in comparison with $\Pi_b$-SPIBB, it allows to cut off bad performing actions even when their estimate is imprecise:
	\begin{equation}
		\Pi_{\leq b} = \left\{\pi\in\Pi \,\middle\vert\,\pi(a|x) \leq \pi_b(a|x)\;\textnormal{if}\;(x,a)\in\mathfrak{B}\right\}
	\end{equation}

	The resulting algorithm is referred as $\Pi_{\leq b}$-SPIBB and amounts, as for $\Pi_b$-SPIBB, to perform a policy iteration under the policy constraint to belong to $\Pi_{\leq b}$. The convergence guarantees of Theorem~\ref{th:pib-spibb-convergence} still apply to $\Pi_{\leq b}$-SPIBB, but we lose the SPI ones.

	Algorithm~\ref{alg:pileqbproj} in Appendix~\ref{sup:algos}, describes the greedy projection of $Q^{(i)}$ on $\Pi_{\leq b}$. Appendix~\ref{sup:pibvspi<b} also includes a comprehensive example that illustrates the difference between the $\Pi_b$-SPIBB and $\Pi_{\leq b}$-SPIBB policy improvement steps. Despite the lack of safety guarantees, our experiments show $\Pi_{\leq b}$-SPIBB to be even safer than $\Pi_{b}$-SPIBB while outperforming it in most scenarios. Multi-batch settings -- where it may be better to keep exploring the bootstrapped pairs -- might be an exception~\cite{Lange2012}.

	\subsection{Other related works}	
	High-Confidence PI refers to the family of algorithms introduced in \citet{Paduraru2013,Mandel2014,thomas2015high}, which rely on the ability to produce high-confidence lower bound on the Importance Sampling (IS) estimate of the trained policy performance. IS and SPIBB approaches are very different in nature: IS provides frequentist bounds, while SPIBB provides Bayesian bounds. In comparison to SPIBB, IS has the advantage of not depending on the MDP model and as a consequence may be applied to infinite MDPs with guarantees. However, the IS estimates are known to be high variance. %The algorithm in \citet{Mandel2014}, based on concentration inequality, tends to be conservative and requires hyper parameters optimization. The algorithms in \citet{thomas2015high2} rely on the assumption that the IS estimate is normally distributed which is false when the number of trajectories is small. The algorithm in \citet{Paduraru2013} is based on bias corrected and accelerated bootstrap and tends to be too optimistic. 
	%As a consequence, only one of the proposed methods is reliable but is likely to produce a too pessimistic lower bound preventing policy improvement without a large dataset. 
	Another drawback of the IS approach is that it fails for long horizon problem. Indeed, \citet{Guo2017} show that the amount of data required by IS-based SPI algorithms scales exponentially with the horizon of the MDP. Regarding the dependency in the horizon of SPIBB algorithms, the discount factor $\gamma$ is often translated as a planning horizon: $H=\frac{1}{1-\gamma}$. This is the case in UCT for instance~\cite{Kocsis2006}.
	As a consequence, Theorem \ref{th:safepolicyimprovement-pi} tells us that the safety is linear in the horizon (given a fixed $V_{max}$).
	
	In \citet{Kakade2002}, Conservative Policy Iteration (CPI) not only assumes access to the environment, but also to a $\mu$-restart mechanism which can basically sample at will from the environment according to a distribution $\mu$. This is used in step (2) of the CPI algorithm to build an estimate of the advantage function precise enough to ensure policy improvement with high probability. SPIBB does not have access to the true environment: all it sees are the finite samples from the batch. Similarly, \citet{Pirotta2013b,Pirotta2013a} consider a single safe policy improvement in order to speed up training of policy gradients (use less policy iterations). These are however not safe in the sense of finding a policy that improves a previous policy with high confidence: they will converge to the same policy asymptotically, the optimal one in the MLE MDP. Additionally, they are not considering the batch setting.

	\begin{figure*}[t]
		\centering
    	\newcommand*{\xMin}{0}
    	\newcommand*{\xMax}{5}
    	\newcommand*{\yMin}{0}
    	\newcommand*{\yMax}{5}	
    	\newcommand*{\tikzfigscale}{0.85}
		\subfloat[Gridworld domain.]{
			\begin{tikzpicture}[scale=\tikzfigscale, every node/.style={scale=\tikzfigscale}]
			\foreach \i in {\xMin,...,\xMax} {
				\draw [very thin,gray] (\i,\yMin) -- (\i,\yMax)  node [below] at (\i,\yMin) {};
			}
			\foreach \i in {\yMin,...,\yMax} {
				\draw [very thin,gray] (\xMin,\i) -- (\xMax,\i) node [left] at (\xMin,\i) {};
			}
			\draw[red,very thick] (3,0) -- (3,3);
			\draw[red,very thick] (4,2) -- (4,4);
			\draw[red,very thick] (0,0) -- (0,5) -- (5,5) -- (5,0) -- (0,0);
			
			\node[draw] at (0.5,0.5) {S};
			\node[draw] at (4.5,4.5) {G};
			
			\node[draw] at (0.5,0.5) {S};
			\node[draw] at (4.5,4.5) {G};
			\draw[->,blue] (0.5,0.7) -- (0.5,1.3);
			\draw[->,blue] (0.7,1.5) -- (1.3,1.5);
			\draw[->,blue] (1.5,1.7) -- (1.5,2.3);
			\draw[->,blue] (1.5,2.7) -- (1.5,3.3);
			\draw[->,blue] (1.7,3.5) -- (2.3,3.5);
			\draw[->,blue] (2.5,3.7) -- (2.5,4.3);
			\draw[->,blue] (2.7,4.5) -- (3.3,4.5);
			\draw[->,blue] (3.7,4.5) -- (4.3,4.5);
			\end{tikzpicture}
			\label{fig:gridworld}
		}
		\subfloat[1\%-CVaR heatmap: $\Pi_{b}$-SPIBB.]{
			\includegraphics[trim = 10pt 40pt 45pt 60pt, clip, width=0.33\textwidth]{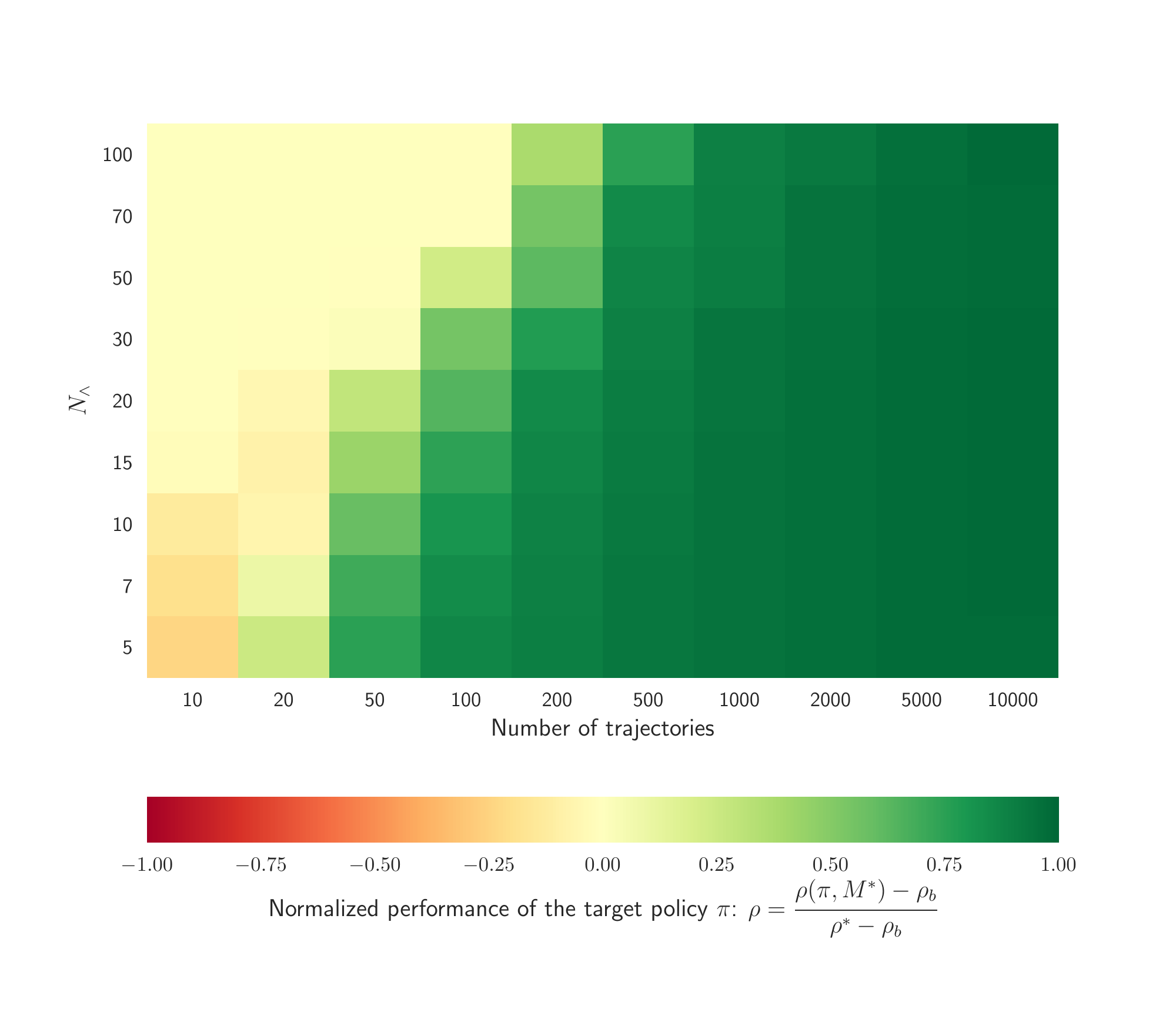}
			\label{fig:heatmap_maze_pib}
		}
		\subfloat[1\%-CVaR heatmap: $\Pi_{\leq b}$-SPIBB.]{
			\includegraphics[trim = 10pt 40pt 45pt 60pt, clip, width=0.33\textwidth]{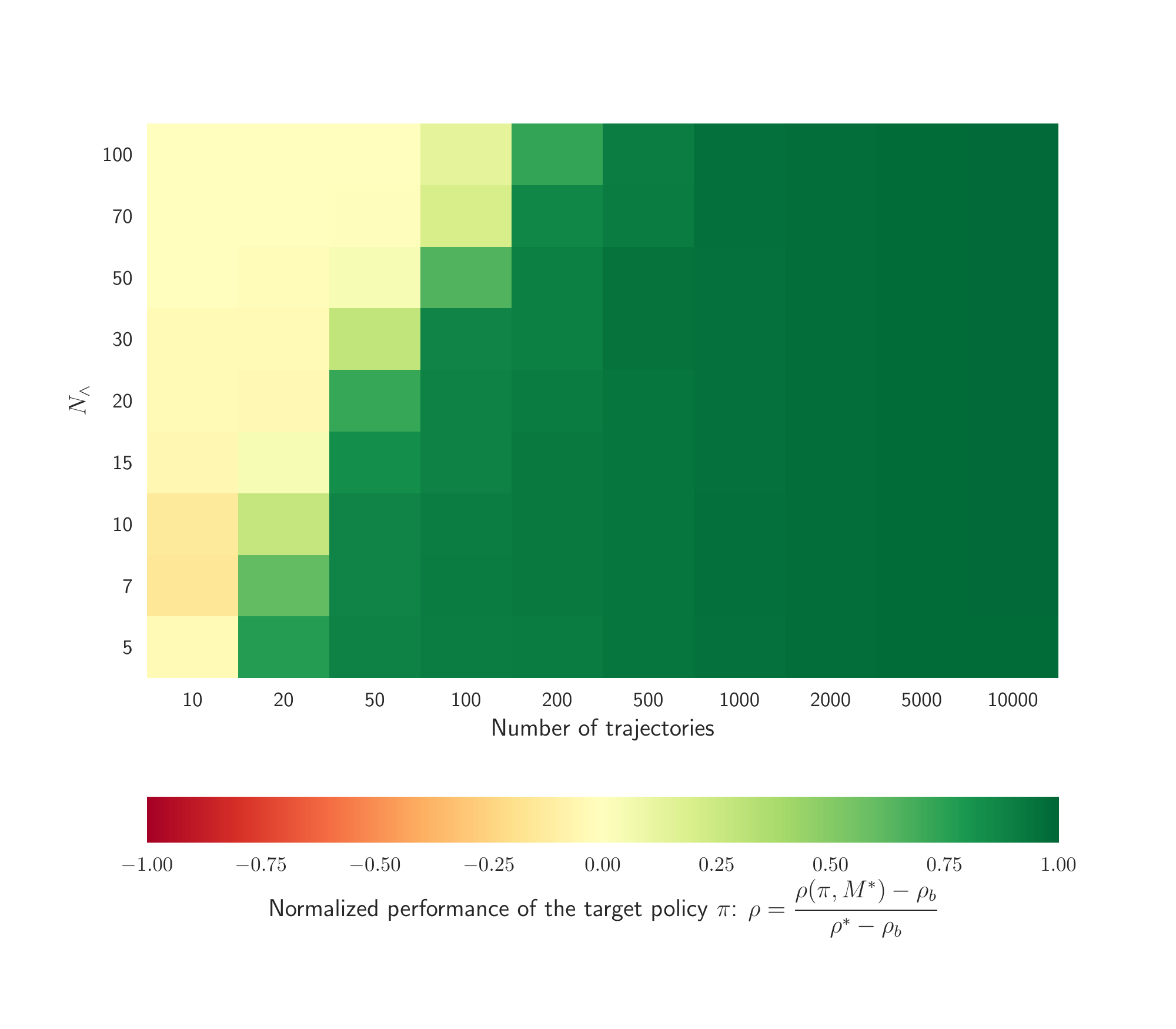}
			\label{fig:heatmap_maze_pi<b}
		} \\
		\bigcentering
		\subfloat[Mean: benchmark with $N_\wedge=5$.]{
			\includegraphics[trim = 5pt 5pt 5pt 5pt, clip, width=0.33\textwidth]{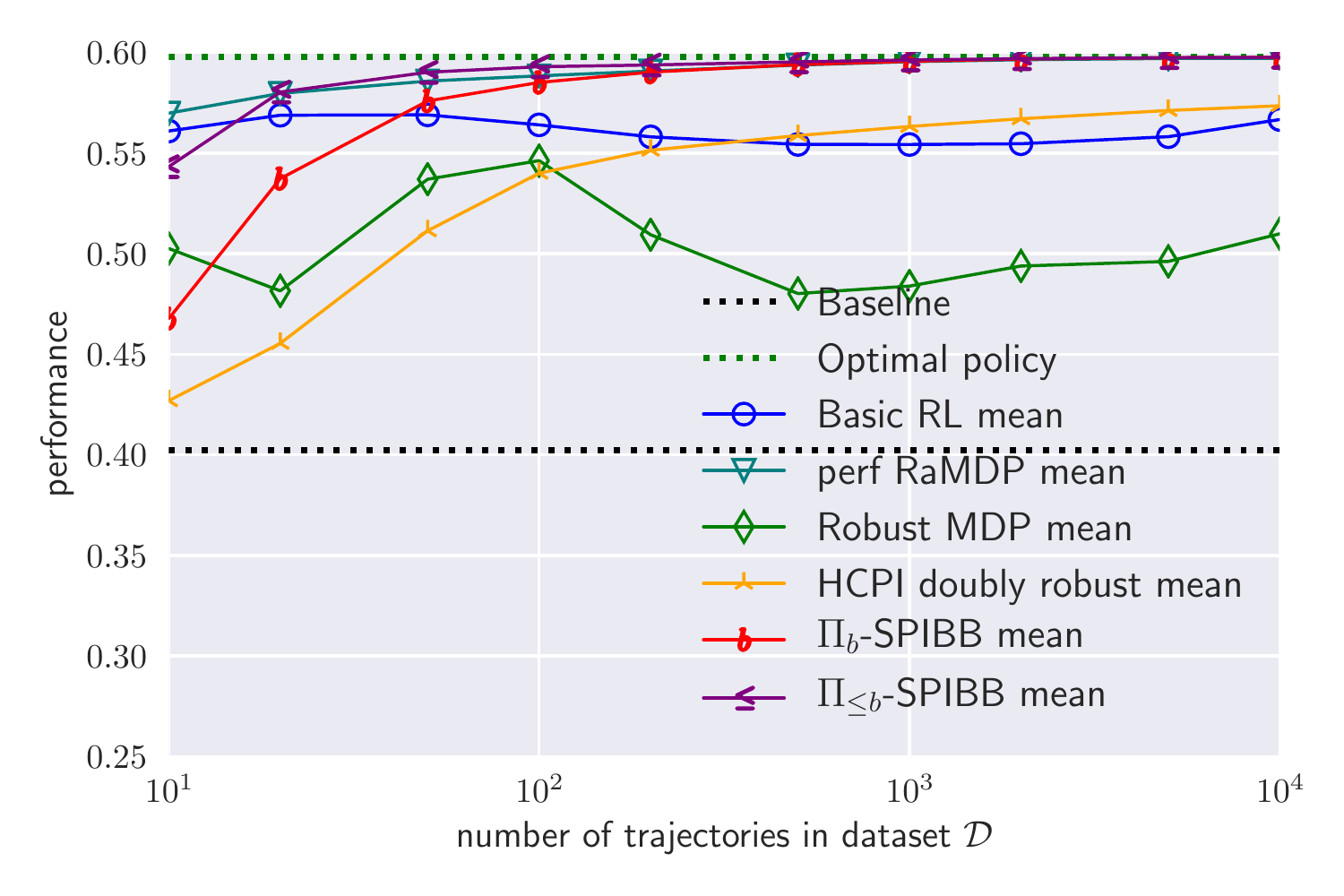}
			\label{fig:benchmark_mean_maze_N=5}
		} 
		\subfloat[1\%-CVaR: benchmark with $N_\wedge=5$.]{
			\includegraphics[trim = 5pt 5pt 5pt 5pt, clip, width=0.33\textwidth]{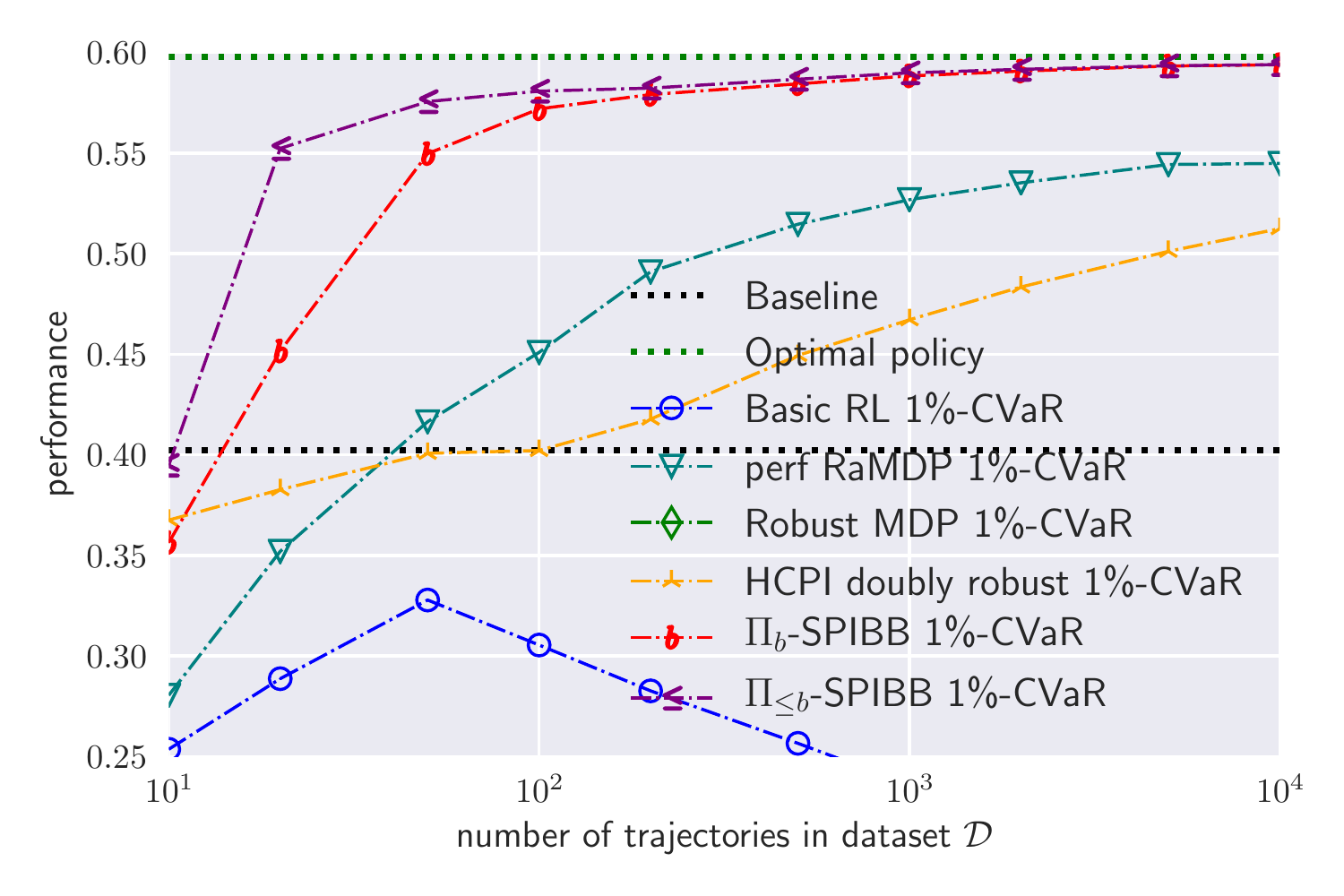}
			\label{fig:benchmark_CVaR_maze_N=5}
		}
		\subfloat[Mean \& 1\%-CVaR: SPIBB w. $N_\wedge=20$.]{
			\includegraphics[trim = 5pt 5pt 5pt 5pt, clip, width=0.33\textwidth]{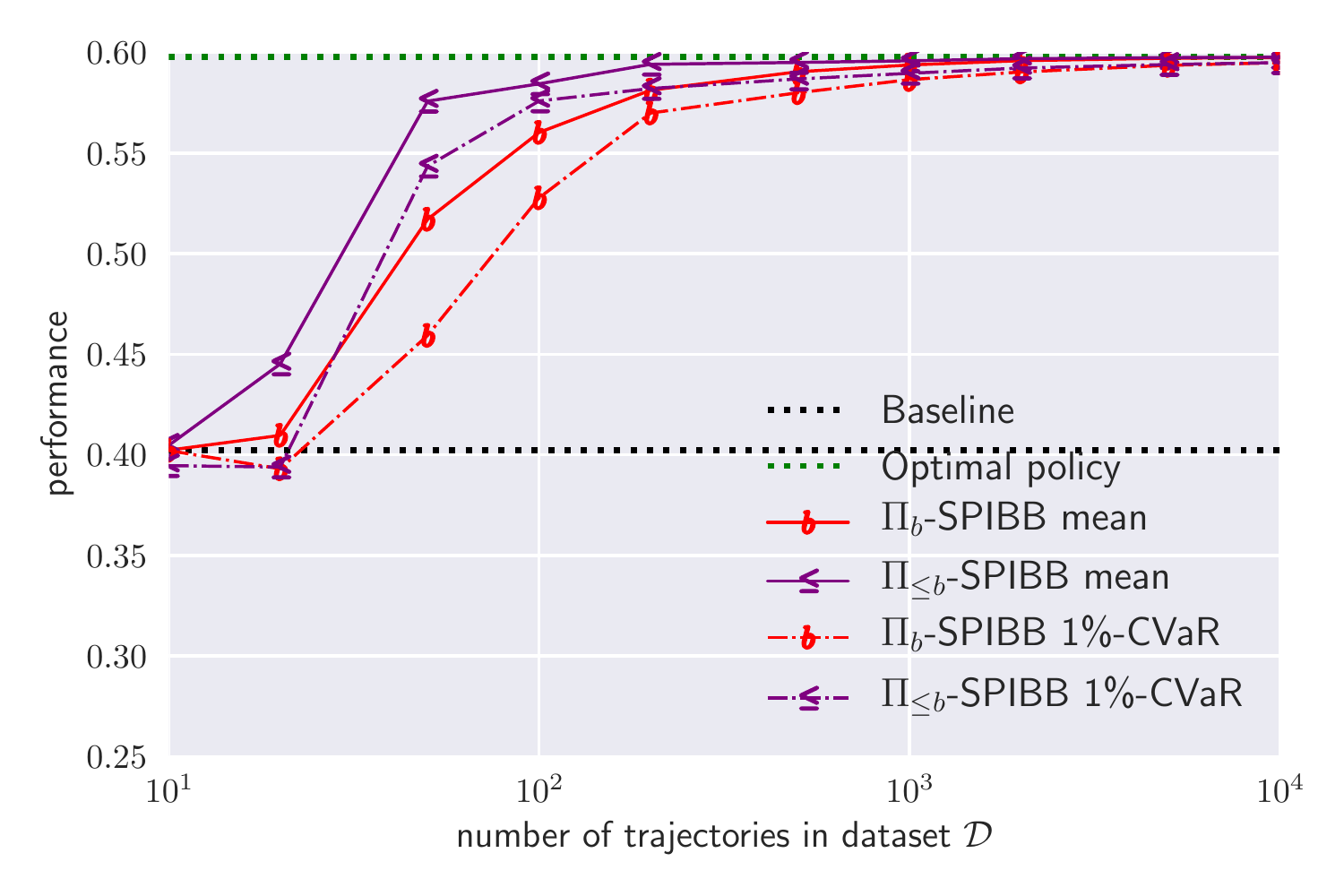}
			\label{fig:maze_spibb_N=20}
		}
		\caption{Gridworld experiment: Figure (a) illustrates the domain with an optimal trajectory. Figures (b-c) are heatmaps of the 1\%-CVaR normalized performance of the SPIBB algorithms as a function of $N_\wedge$. Figures (d-e) show the benchmark for the mean and 1\%-CVaR performance. Figure (f) displays additional curves for another value of $N_\wedge$. }
		\label{fig:maze_main}
% 	\vspace{-10pt}
	\end{figure*}
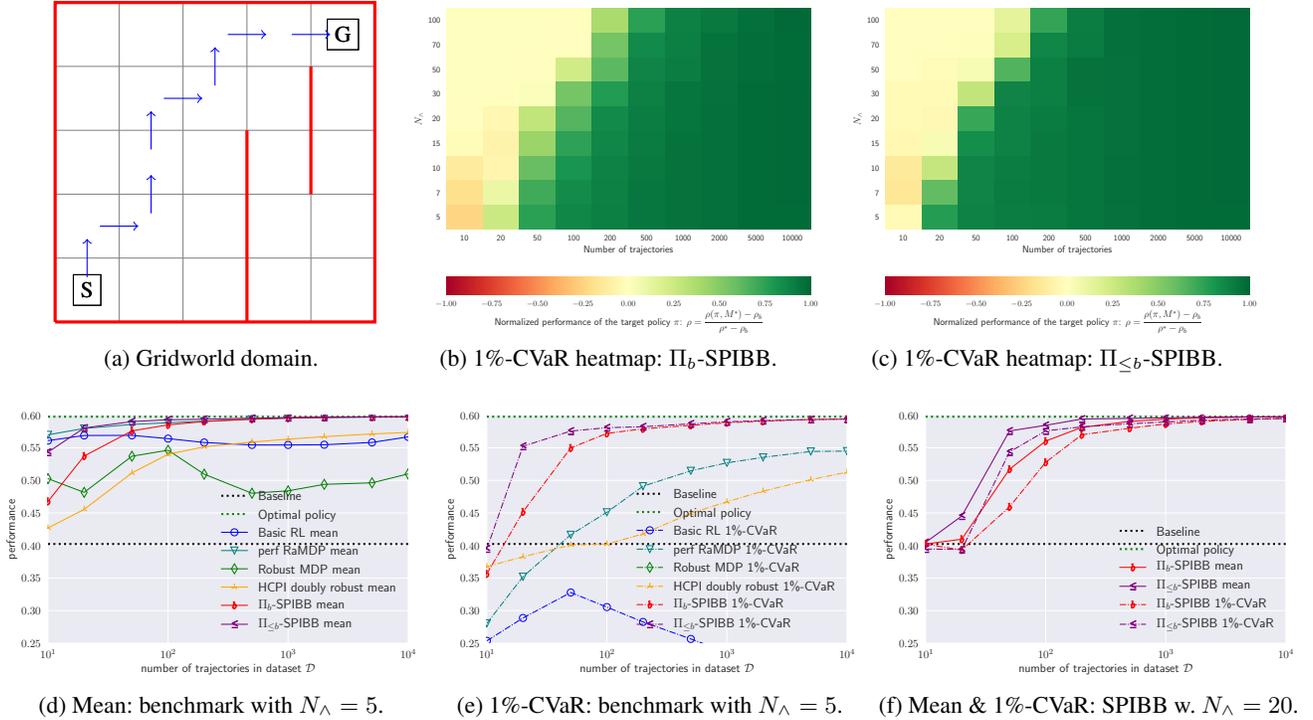
	
	\section{SPIBB Empirical Evaluation}
	\label{sec:results}
	
	The performance of Batch RL algorithms can vary greatly from one dataset to another. To properly assess existing and SPIBB algorithms, we evaluate their ability to generate policies that consistently outperform the baseline. Practically, we repeated 100k times the following procedure on various environments: randomly generate a dataset, train a policy on that dataset using each algorithm and each hyper-parameter in the benchmark and compute the performance of the trained policy (with $\gamma=0.95$). We formalize this experimental protocol in Appendix \ref{sup:gridworld-protocol}. The algorithms are then evaluated using the mean performance and conditional value at risk performance (CVaR, also called expected shortfall) of the policies they produced. The $X\%$-CVaR is the mean performance over the $X\%$ worst runs. Given the high number of runs, all the results that are visible to the naked eye are significant.

	%In this section, we answer various questions about the empirical performance of SPIBB algorithms. The RL performance of a policy is the average return of the trajectories it generated (with $\gamma=0.95$). But we are interested in evaluating algorithms. An algorithm trains policies that may be performing more or less well, depending on the dataset it was trained on. We focus therefore on assessing their consistency to generate policies that are outperforming the baseline.
	
	%The algorithms are evaluated on their mean performance and their conditional value at risk performance (CVaR), sometimes also called the expected shortfall in the literature: the mean performance over the $X\%$ worst runs. Unless stated otherwise, a run comprises: the random data generation and the training of the algorithms in the benchmark. Each point of each curve is a mean/CVaR based on 200k runs. All the results that are visible to the naked eye are therefore significant.

	In addition to the SPIBB algorithms, our finite MDP benchmark contains four algorithms: Basic RL, HCPI~\cite{thomas2015high}, Robust MDP, and RaMDP~\cite{ghavamzadeh2016safe}. RaMDP stands for Reward-adjusted MDP and applies an exploration penalty when performing actions rarely observed in the dataset. At the exception of Basic RL, they all rely on one hyper-parameter: $\delta_{hcpi}$, $\delta_{rob}$ and $\kappa_{adj}$ respectively. We performed a grid search on those parameters and for HCPI compared 3 versions. In the main text, we only report the best performance we found ($\delta_{hcpi}=0.9$, $\delta_{rob}=0.1$, and $\kappa_{adj}=0.003$), the full results can be found in Appendix~\ref{sup:benchmarkalgos}. Additionally, Robust MDP and RaMDP depend on a safety test that always failed in our experiments. We still report their performance.
	
	\subsection{Does SPIBB outperform existing algorithms?}
    \label{sec:gridworld_exp}
% 	In this section, we investigate which algorithms are succeeding at safely improving the baselines. We also study the sensitivity of SPIBB algorithms with respect to $N_\wedge$. 
	
 	\begin{figure*}[t]
		\bigcentering
		\subfloat[1\%-CVaR benchmark with $N_\wedge=20$.]{
			\includegraphics[trim = 5pt 5pt 5pt 5pt, clip, width=0.33\textwidth]{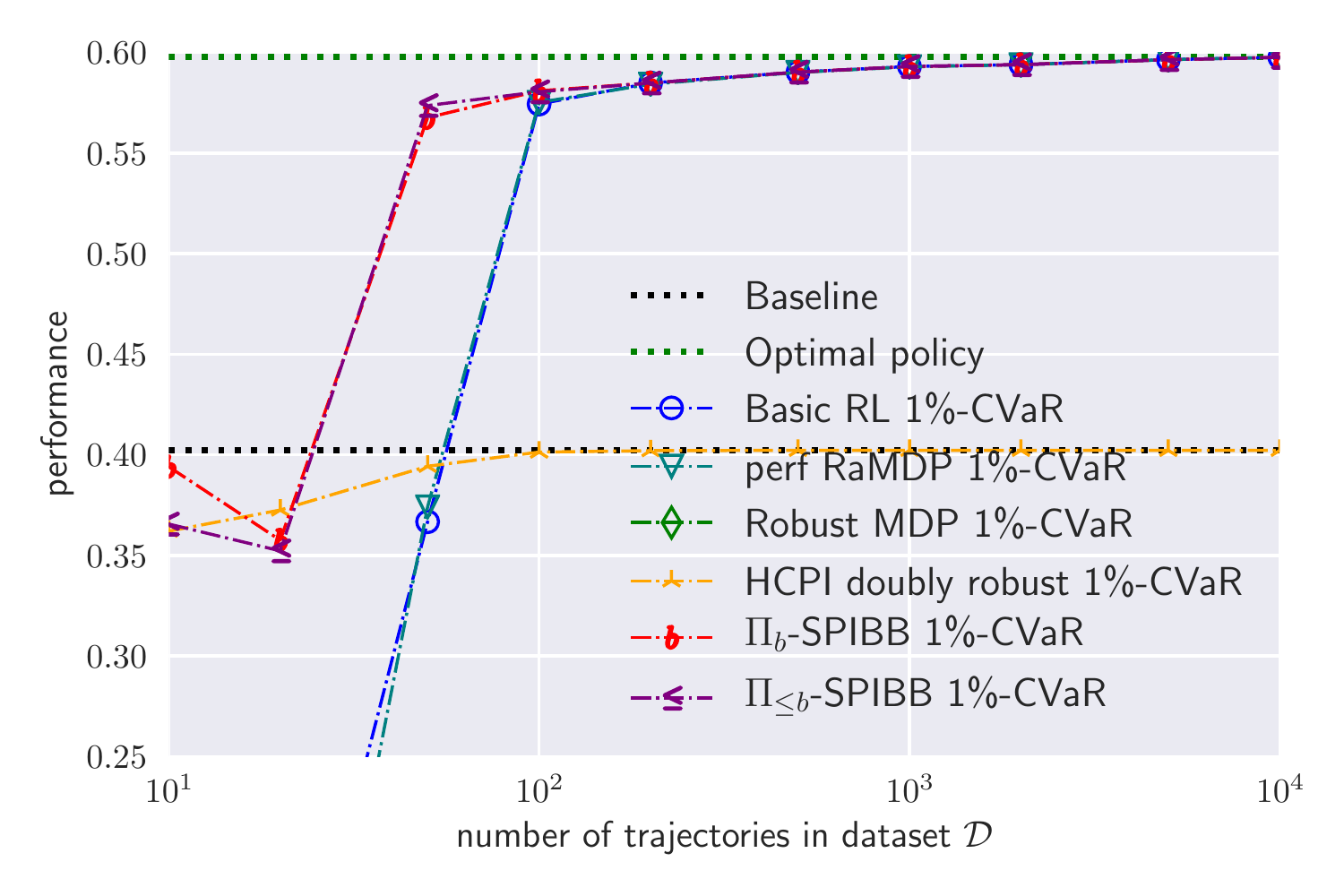}
			\label{fig:benchmark_CVaR_rand_N=20}
		} 
		\subfloat[1\%-CVaR heatmap: $\Pi_{b}$-SPIBB.]{
			\includegraphics[trim = 10pt 130pt 45pt 60pt, clip, width=0.33\textwidth]{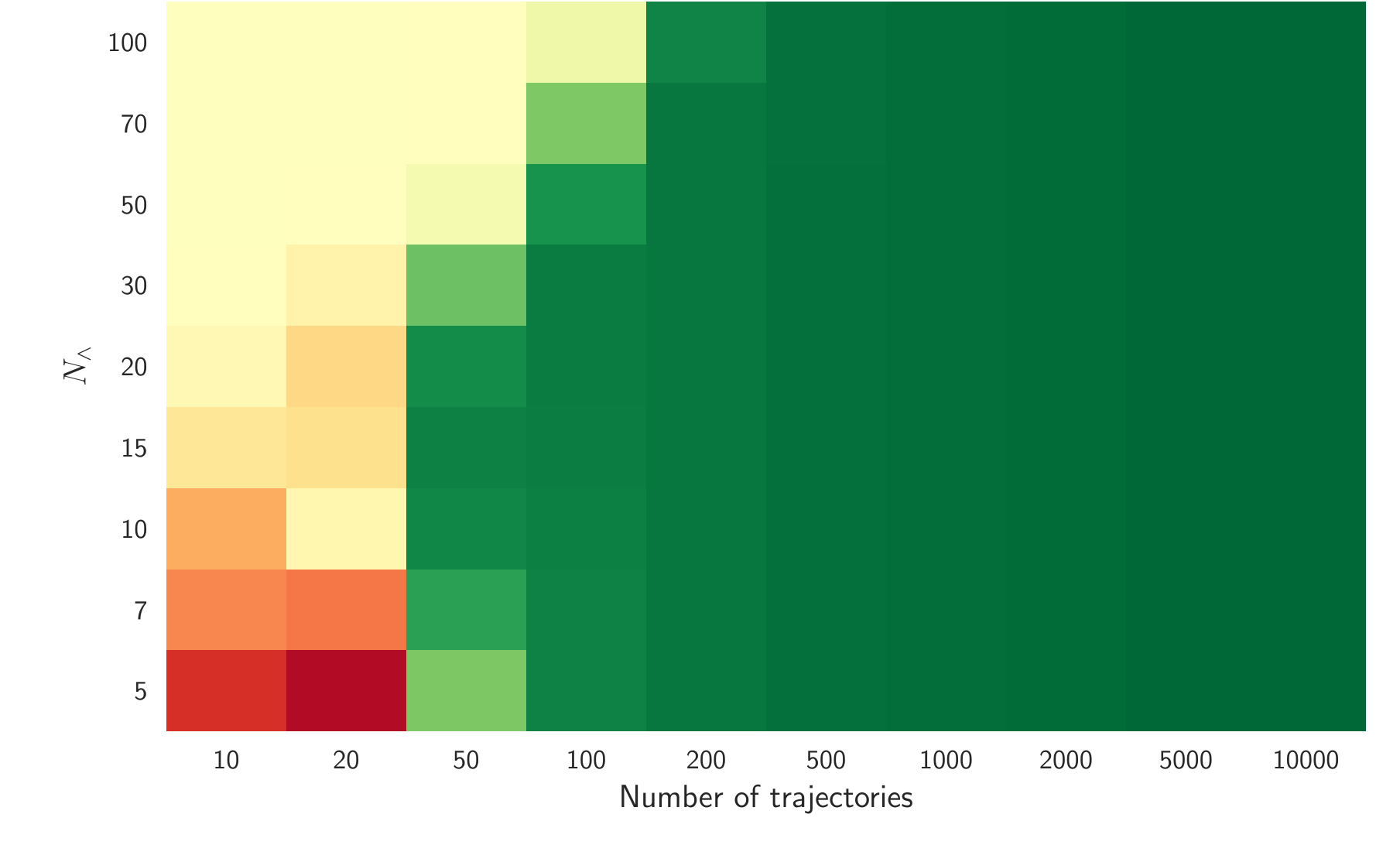}
			\label{fig:heatmap_CVaR_rand_pi_b}
		}
		\subfloat[1\%-CVaR heatmap: $\Pi_{\leq b}$-SPIBB.]{
			\includegraphics[trim = 10pt 130pt 45pt 60pt, clip, width=0.33\textwidth]{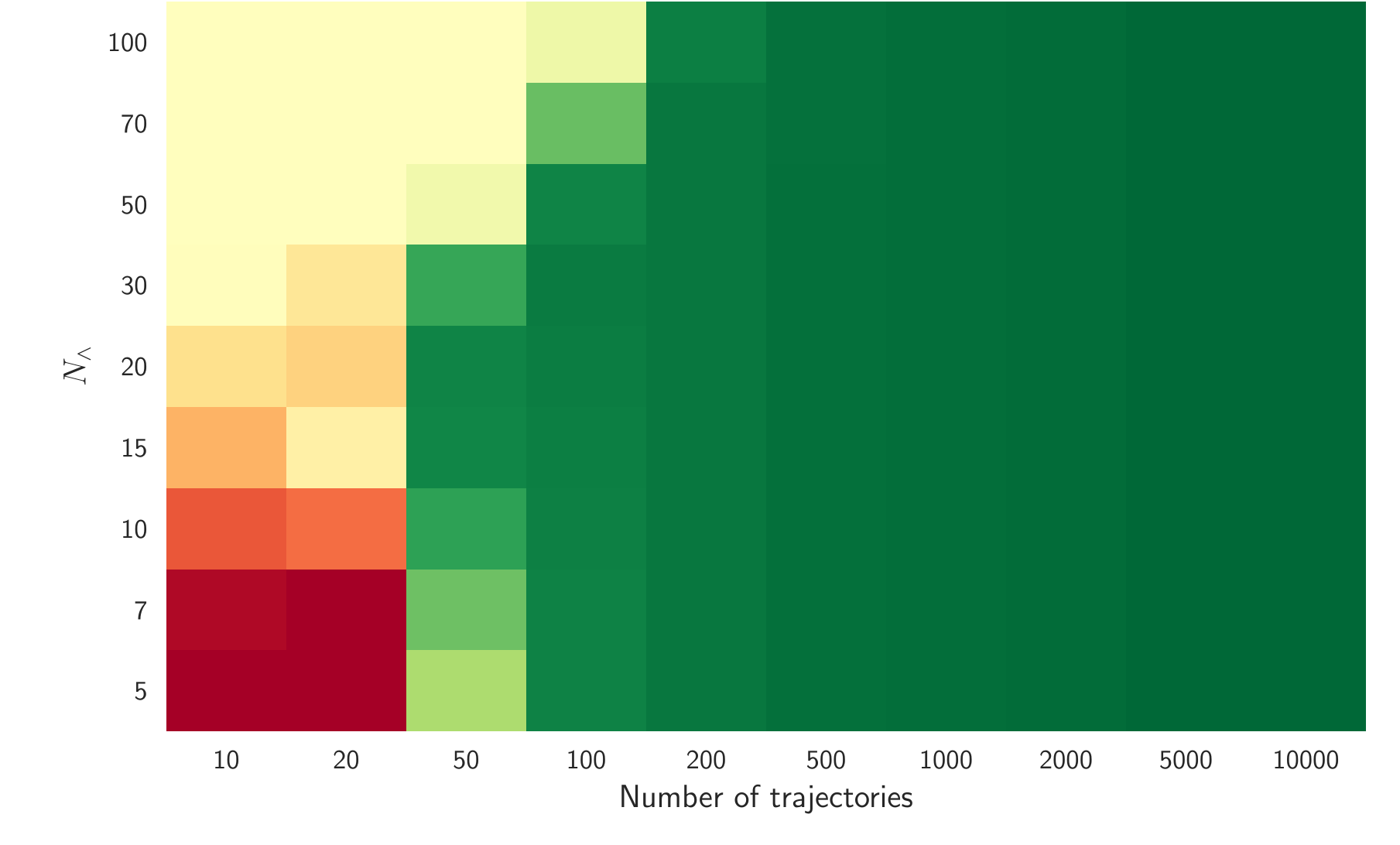}
			\label{fig:heatmap_CVaR_rand_pi_<b}
		}
		\caption{Gridworld experiment with random behavioural policy:  Figure (a) shows the benchmark for the 1\%-CVaR performance, with SPIBB using $N_\wedge=20$. Figures (b-c) are the heatmaps of the 1\%-CVaR normalized performance of the SPIBB algorithms as a function of $N_\wedge$ (same heat colours as in Figures \subref*{fig:heatmap_maze_pib} and \subref*{fig:heatmap_maze_pi<b}).}
		\label{fig:rand_behavioural_main}
% 		\vspace{-10pt}
	\end{figure*}
	
	Our first domain is a discrete, stochastic, $5 \times 5$ gridworld (see Figure \subref*{fig:gridworld}), with 4 actions: up, down, left and right. The transitions are stochastic: the agent moves in the requested direction with $75\%$ chance, in the opposite one with $5\%$ chance and to either side with $10\%$ chance each. The initial and final states are respectively the bottom left and top right corners. The reward function is $+1$ when the final state is reached and 0 everywhere else. The baseline we use in this experiment is a fixed stochastic policy with a 0.4 performance, the optimal policy has a 0.6 performance.
	 
	We start by analysing the sensitivity of $\Pi_{b}$-SPIBB and $\Pi_{\leq b}$-SPIBB with respect to $N_\wedge$. We visually represent the results as two 1\%-CVaR heatmaps: Figures \subref*{fig:heatmap_maze_pib} and \subref*{fig:heatmap_maze_pi<b} for $\Pi_{b}$-SPIBB and $\Pi_{\leq b}$-SPIBB. They read as follows: the colour of a cell indicates the improvement over the baseline normalized with respect to the optimal performance: red, yellow, and green respectively mean below, equal to, and above baseline performance. We observe for SPIBB algorithms that the policy improvement is safe (at the slight exception of $\Pi_{b}$-SPIBB with a low $N_\wedge$ on 10-trajectory datasets), that the bigger the $N_\wedge$, the more conservative SPIBB gets, and that $\Pi_{\leq b}$-SPIBB outperforms $\Pi_{b}$-SPIBB.
	 
	In Figure~\subref*{fig:benchmark_mean_maze_N=5}, we see that Basic RL improves the baseline on average, but not monotonically with the size of the dataset, and remains quite far from optimal. That fact is explained by the fairly frequent learning of catastrophic policies and will be analyzed in details with the 1\%-CVaR results. HCPI is more conservative for small datasets but slightly outperforms Basic RL for bigger ones, still remaining away from optimal. We also observe that Robust MDPs do even worse than Basic RL; in fact, they learn policies that remain at the center of the grid where the dataset contains a maximum of transitions and therefore where the Robust MDPs have a minimal estimate error, and completely ignore the goal. A similar behaviour is observed with RaMDP when its hyper-parameter is set too high ($\geq 0.004$). Inversely, when it is set too low ($\leq 0.002$), RaMDP behaves like Basic RL. But in the tight spot of 0.003, RaMDP is very efficient. We refer the interested reader to Appendix~\ref{sup:benchmarkalgos} for the analysis of hyper-parameter search for the benchmark algorithms. Overall, RaMDP and $\Pi_{\leq b}$-SPIBB win this benchmark based on mean performance, with $\Pi_{b}$-SPIBB not far behind.
	
	Figure \subref*{fig:benchmark_CVaR_maze_N=5} displays the 1\%-CVaR performance of the algorithms. We observe that the very good mean performance of RaMDP hides some catastrophic runs where the trained policy under-performs for small datasets. In contrast, $\Pi_{\leq b}$-SPIBB's curve remains over the baseline. $\Pi_{b}$-SPIBB is again a bit behind. HCPI also proves to be near safe. We explained in the previous paragraph why Robust MDP often generates bad policies. It actually does it so often, and the policies are so bad, that its curve does not even show on the graph. Let us now consider Basic RL and explain why it does so poorly, even at times on very large datasets (considering that the MDP has 25 states and 4 actions). The dataset is collected using a baseline that performs some actions only very rarely. As a consequence, even in big datasets, some state-action pairs are observed only once or twice. Given the stochasticity of the environment, the MLE MDP might be quite different from the true MDP in those states, leading to policies falsely taking advantage of those chimaeras. SPIBB algorithms are not allowed to jump to conclusions without sufficient proof and have to conservatively reproduce the baseline policy in those configurations.
	
	Figure \subref*{fig:maze_spibb_N=20} shows the SPIBB curves for a higher value of $N_\wedge = 20$. There, the algorithms are more conservative and therefore safe, while still achieving near optimality on big datasets.  Full results may be found in Appendix~\ref{sup:maze_full_results}.

	\begin{figure*}[t]
		\centering
		\bigcentering
		\subfloat[Mean: with $\eta=0.1$ and $N_\wedge=10$.]{
			\includegraphics[trim = 5pt 5pt 5pt 5pt, clip, width=0.33\textwidth]{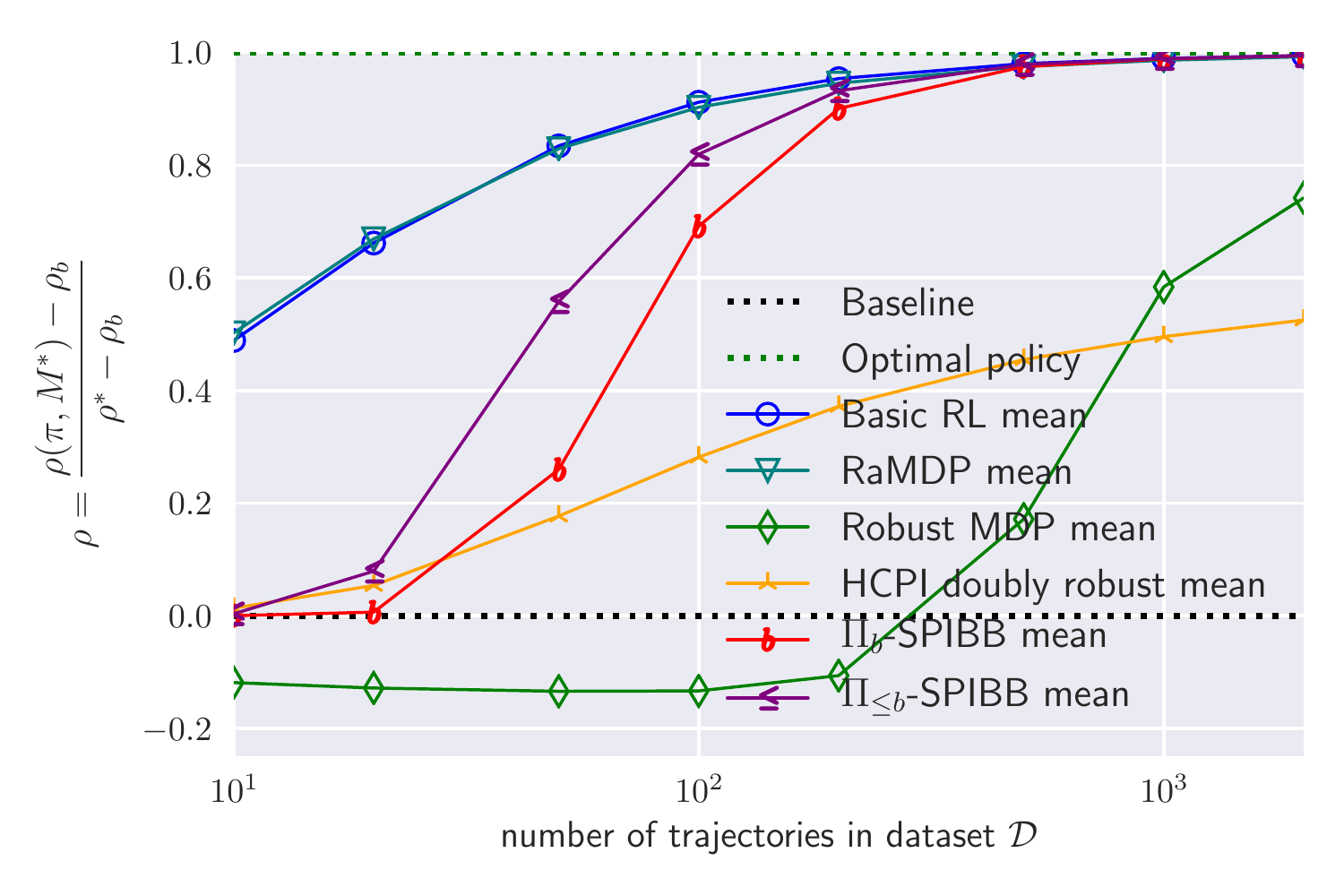}
			\label{fig:benchmark_randomMDP_mean_ratio=0.1}
		}
		\subfloat[Mean: with $\eta=0.9$ and $N_\wedge=10$.]{
			\includegraphics[trim = 5pt 5pt 5pt 5pt, clip, width=0.33\textwidth]{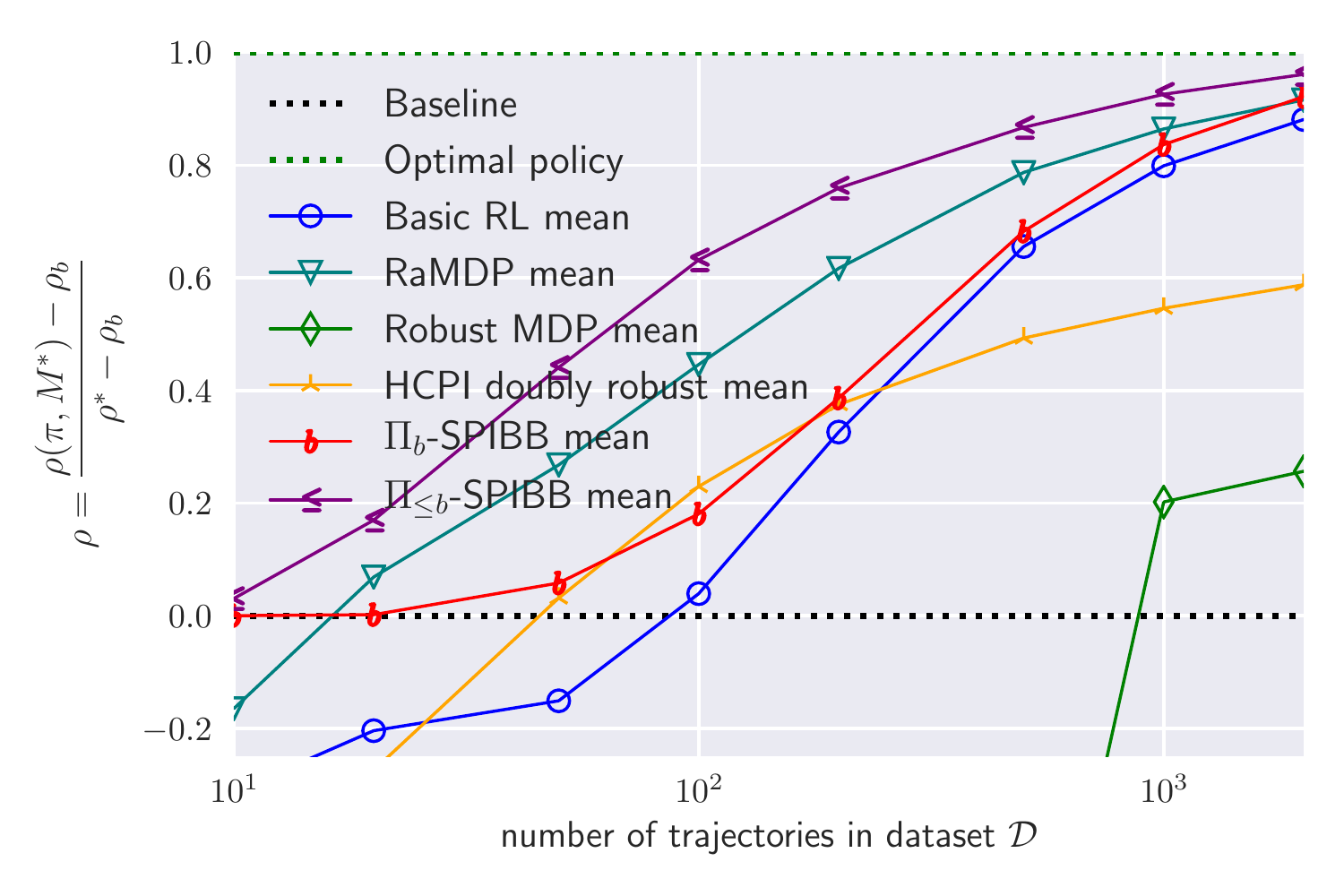}
			\label{fig:benchmark_randomMDP_mean_ratio=0.9}
		}
		\subfloat[1\%-CVaR: with $\eta=0.9$ and $N_\wedge=10$.]{
			\includegraphics[trim = 5pt 5pt 5pt 5pt, clip, width=0.33\textwidth]{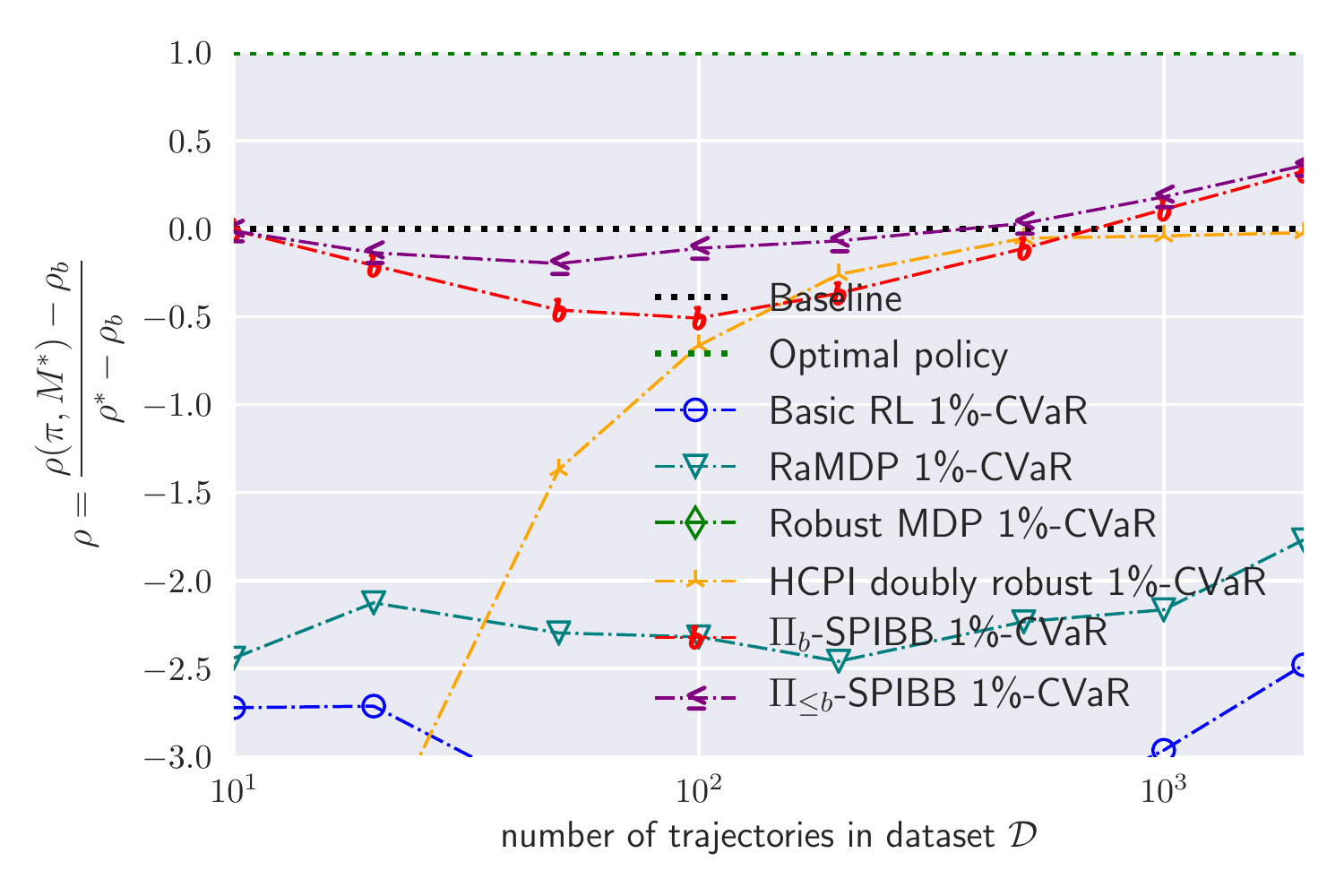}
			\label{fig:benchmark_randomMDP_percentile_ratio=0.9}
		} \\
		\subfloat[1\%-CVaR: RaMDP with $\delta_{adj}=0.003$.]{
			\includegraphics[trim = 10pt 130pt 45pt 60pt, clip, width=0.33\textwidth]{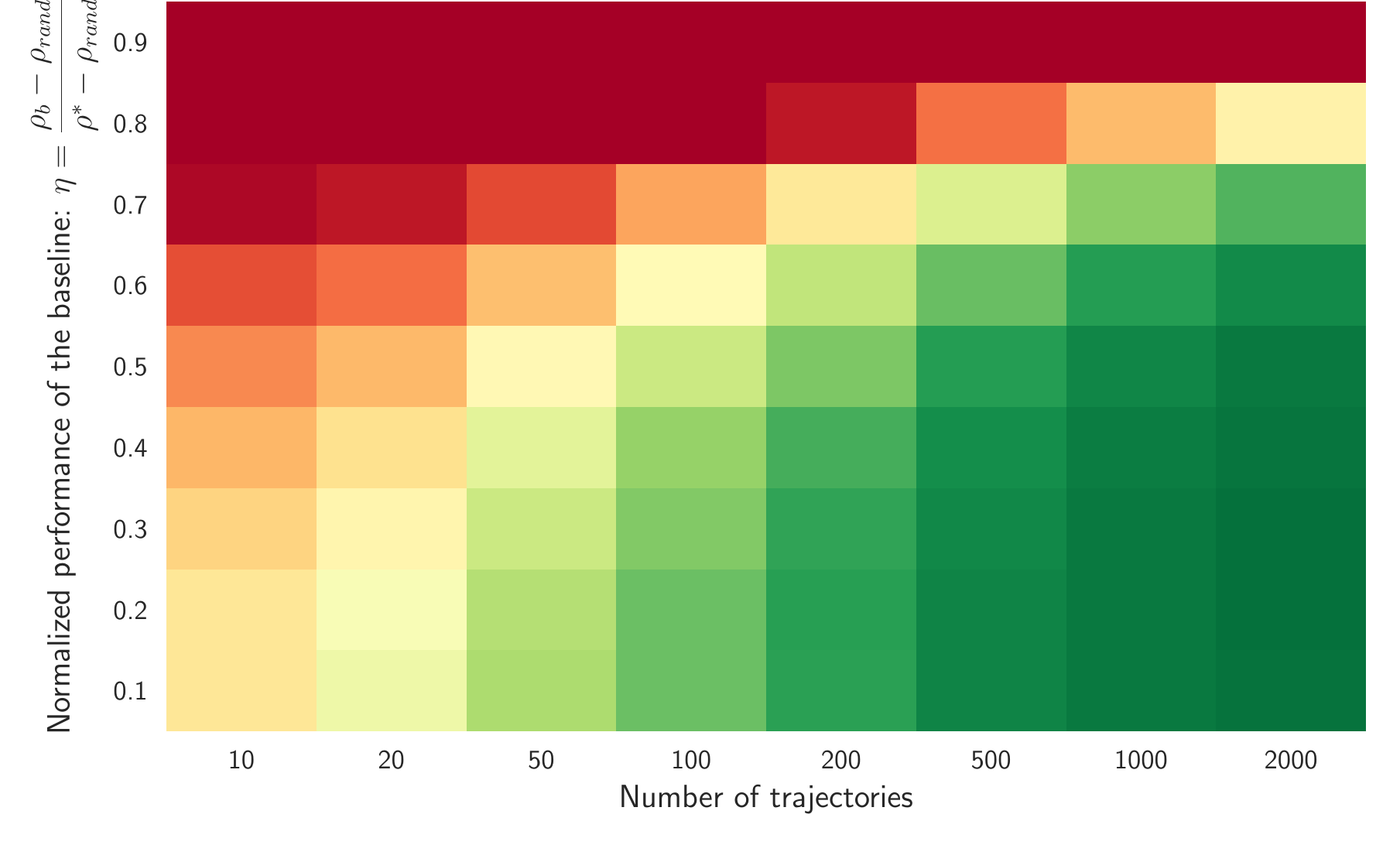}
			\label{fig:heatmap_RaMDP}
		}
		\subfloat[1\%-CVaR: $\Pi_{\leq b}$-SPIBB with $N_\wedge=10$.]{
			\includegraphics[trim = 10pt 130pt 45pt 60pt, clip, width=0.33\textwidth]{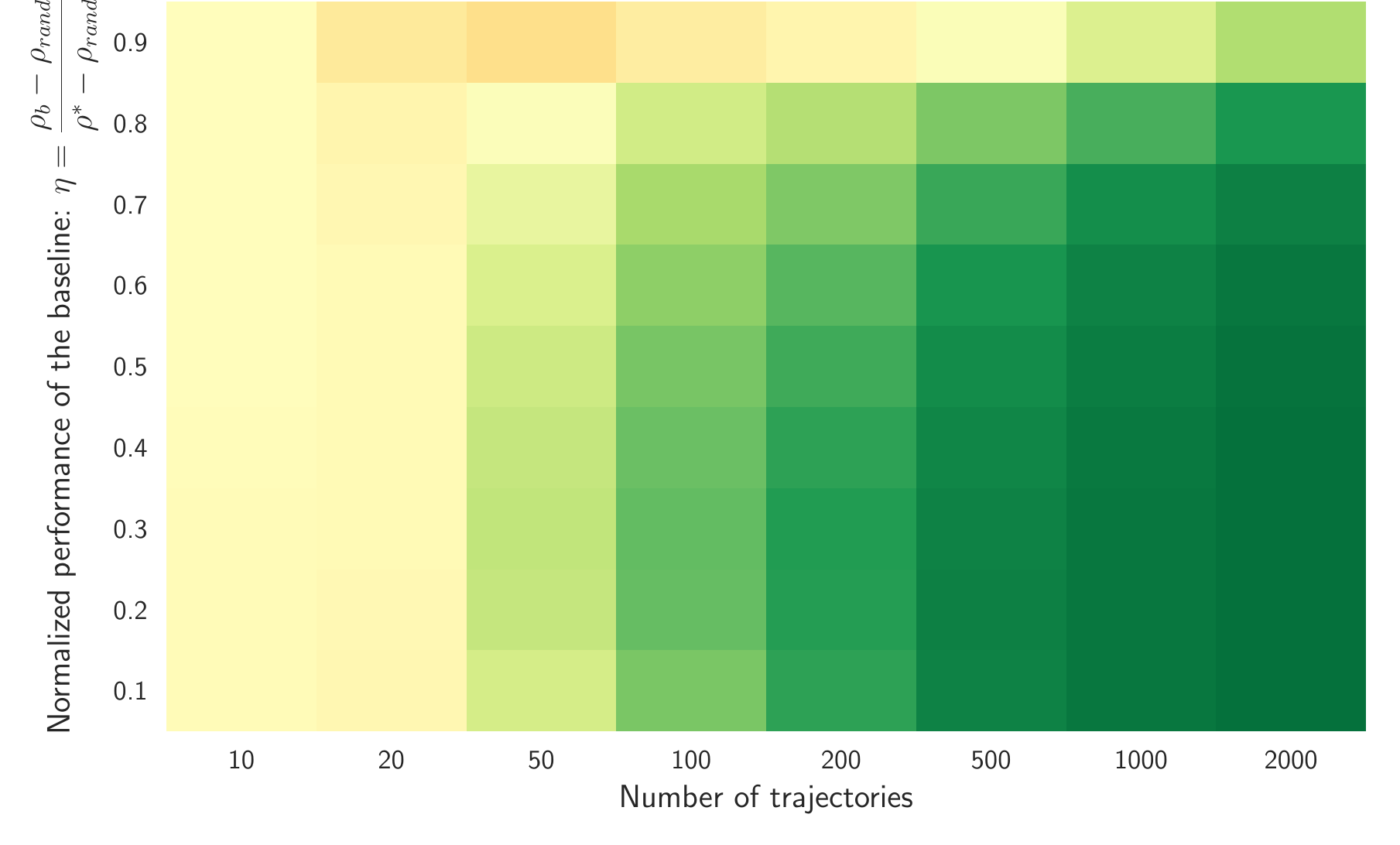}
			\label{fig:heatmap_pi_<b_randomMDP_N=10}
		}
		\subfloat[1\%-CVaR: $\Pi_{\leq b}$-SPIBB with $\eta=0.9$.]{
			\includegraphics[trim = 10pt 130pt 45pt 60pt, clip, width=0.33\textwidth]{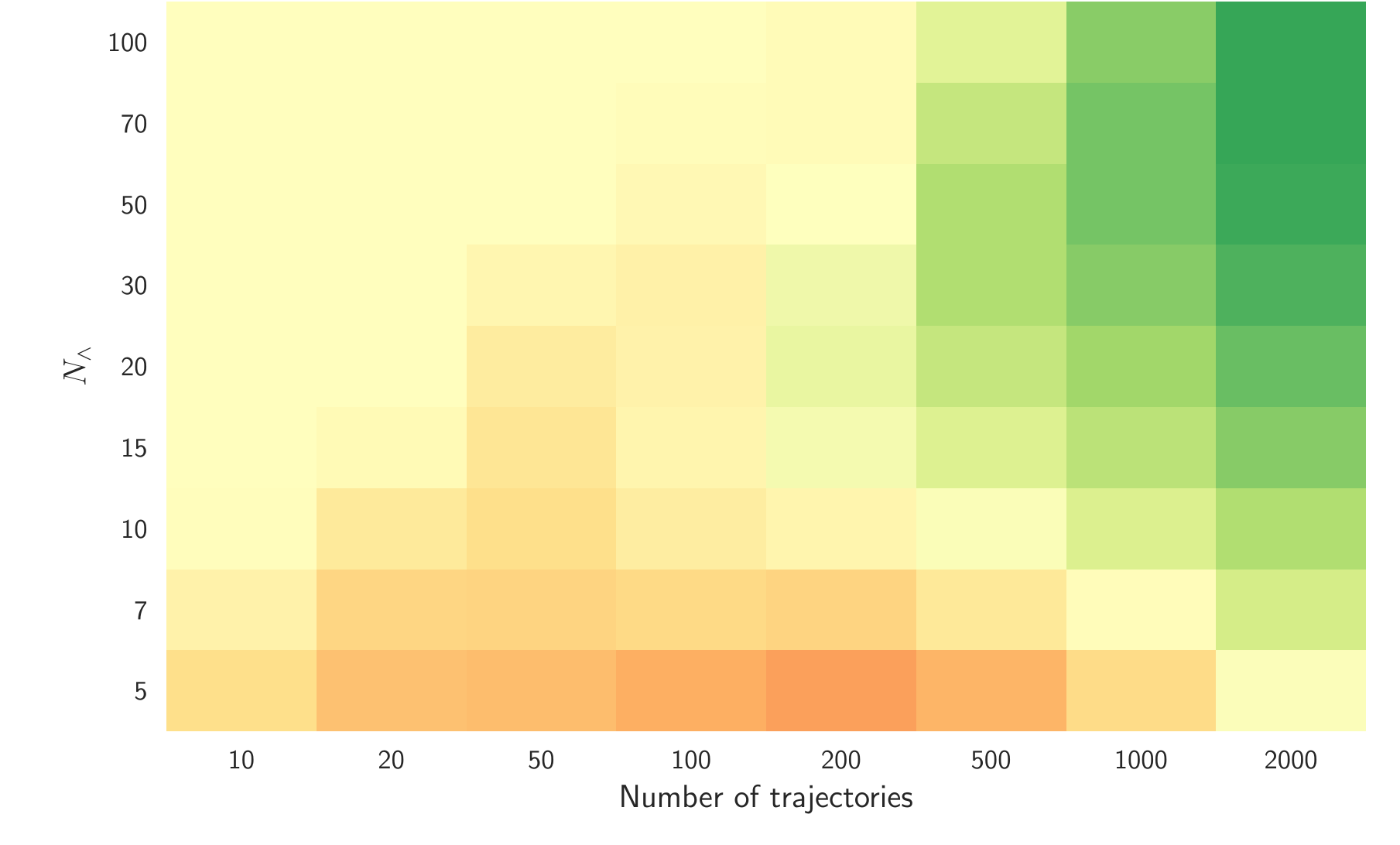}
			\label{fig:heatmap_pi_<b_randomMDP_ratio=0.9}
		}
		\caption{Random MDPs domain: Figures (a-c) show the mean and 1\%-CVaR performances for $\eta$ values of 0.1 and 0.9 and SPIBB with $N_\wedge=10$. Figures (d-e) are the 1\%-CVaR as a function of $\eta$ for RaMDP and $\Pi_{\leq b}$-SPIBB respectively. Figure (f) is the 1\%-CVaR heatmap for $\Pi_{\leq b}$-SPIBB as a function of $N_\wedge$ with $\eta=0.9$.}
		\label{fig:randomMDP_main}
% 		\vspace{-10pt}
	\end{figure*}
	
	\subsection{Must the dataset be collected with the baseline?} 
    \label{sec:gridworld_exp_random_baseline}
% 	Benchmark 1%
	SPIBB theory relies on the assumption that the baseline was used for the data collection, which is a limiting factor of the method. In practice, this assumption simply ensures that the preferential trajectories of the baseline are experienced in the batch of trajectories used for training. We modify the previous experiment by producing datasets using a uniform random policy, while keeping the same Gridworld environment and the same baseline for bootstrapping. In this setting, Basic RL does not have its non-monotonic behaviour anymore, but both our algorithms, $\Pi_b$-SPIBB and $\Pi_{\leq b}$-SPIBB, still significantly outperform their competitors (see Figure \subref*{fig:benchmark_CVaR_rand_N=20}). Note however the following differences: Basic RL becomes safe with 100 trajectories, RaMDP does not improve Basic RL anymore, and HCPI has more difficulty improving the baseline. Robust MDP still does not show on the 1\%-CVaR figure. Focusing more specifically on the SPIBB algorithms and their $N_\wedge$ sensitivity, Figures \subref*{fig:heatmap_CVaR_rand_pi_b} and \subref*{fig:heatmap_CVaR_rand_pi_<b} show that they fail to be completely safe when $N_\wedge\leq 10$ and $|\mathcal{D}|\leq 20$; and that $\Pi_b$-SPIBB slightly outperforms $\Pi_{\leq b}$-SPIBB. Indeed, $\Pi_{\leq b}$-SPIBB cannot take advantage anymore of the bias that the behavioural policy tends to take actions that are better than average. Full results may be found in Appendix~\ref{sup:random_baseline_full_results}.

    % More generally, the baseline policy does not need to be stochastic for the theory to apply, but the behavioural policy has to incorporate some stochasticity in order to allow some improvement: if the same action is deterministically taken in the same states, it is impossible for any algorithm to infer that one action is better than another. The more stochastic the behavioural policy is, the more informative the dataset will be.
    
 	\begin{figure*}[t]
		\bigcentering
		\subfloat[Helicopter environment.]{
			\includegraphics[trim = 5pt 5pt 5pt 5pt, clip, width=0.27\textwidth]{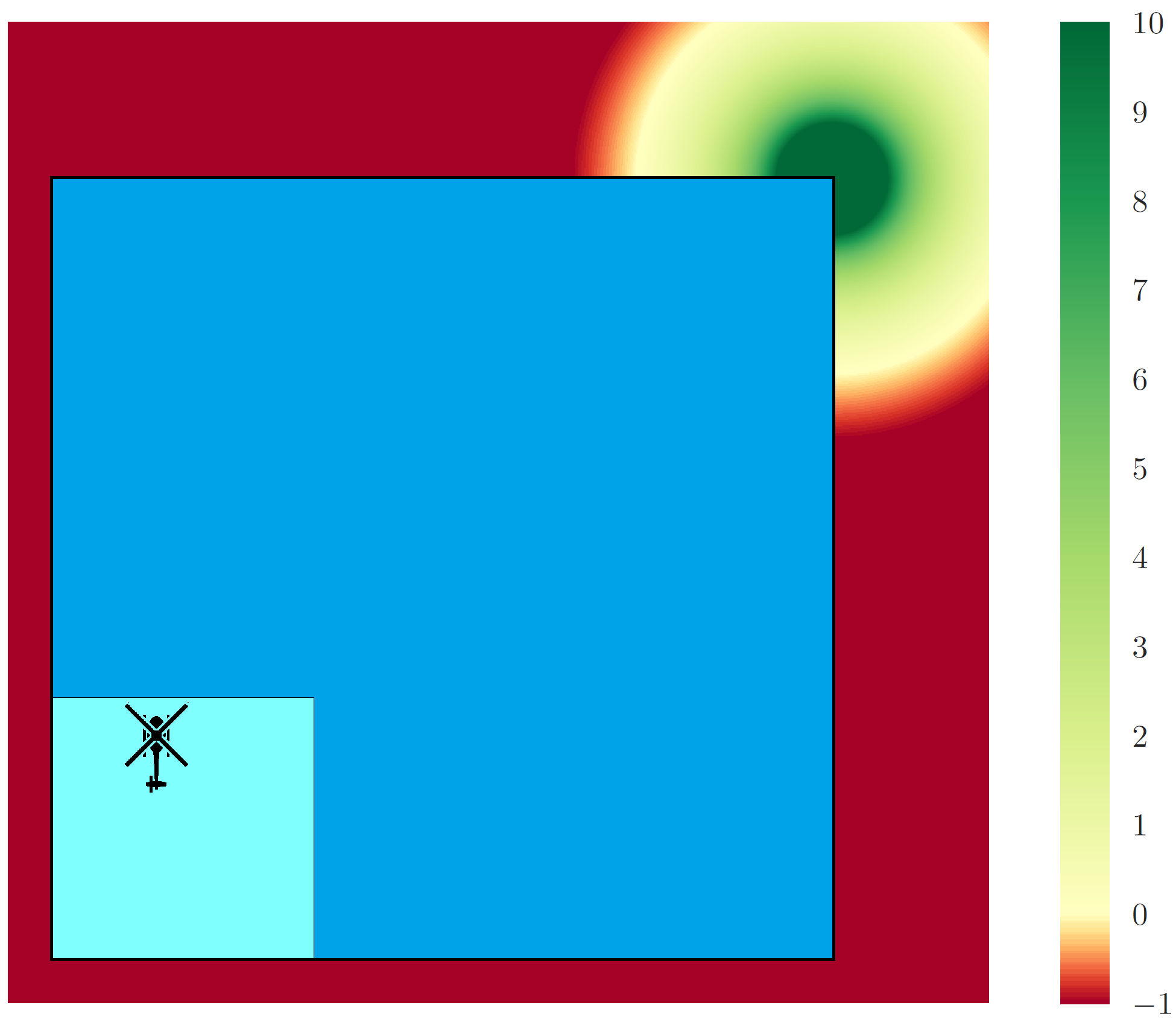}
			\label{fig:helicopter}
		} 
		\subfloat[Mean and 10\%-CVaR in function of $N_\wedge$.]{
			\includegraphics[trim = 5pt 5pt 5pt 5pt, clip, width=0.35\textwidth]{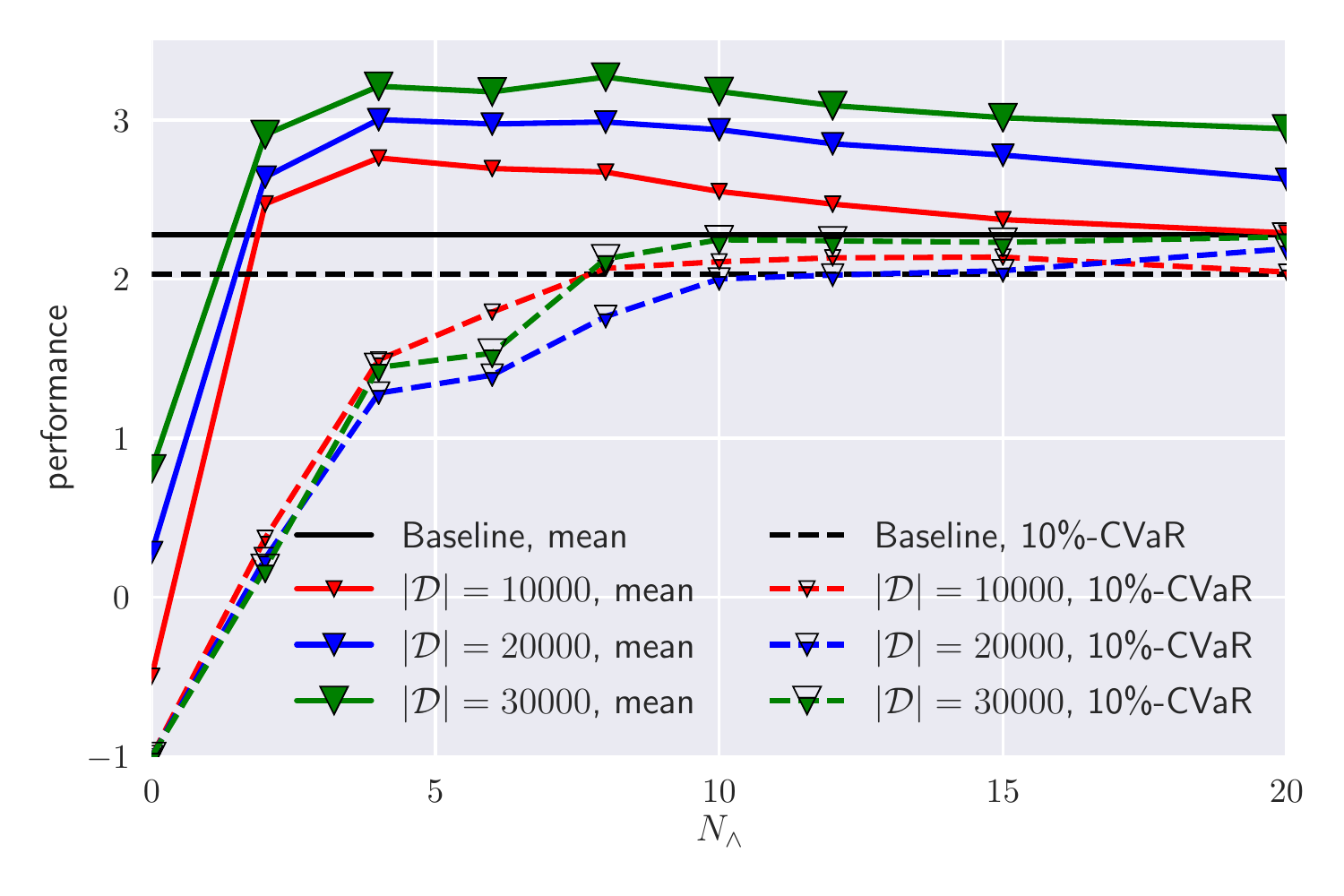}
			\label{fig:10k-20k-30k-datasets}
		}
		\subfloat[Performance in function of the noise factor.]{
			\includegraphics[trim = 5pt 5pt 5pt 5pt, clip, width=0.35\textwidth]{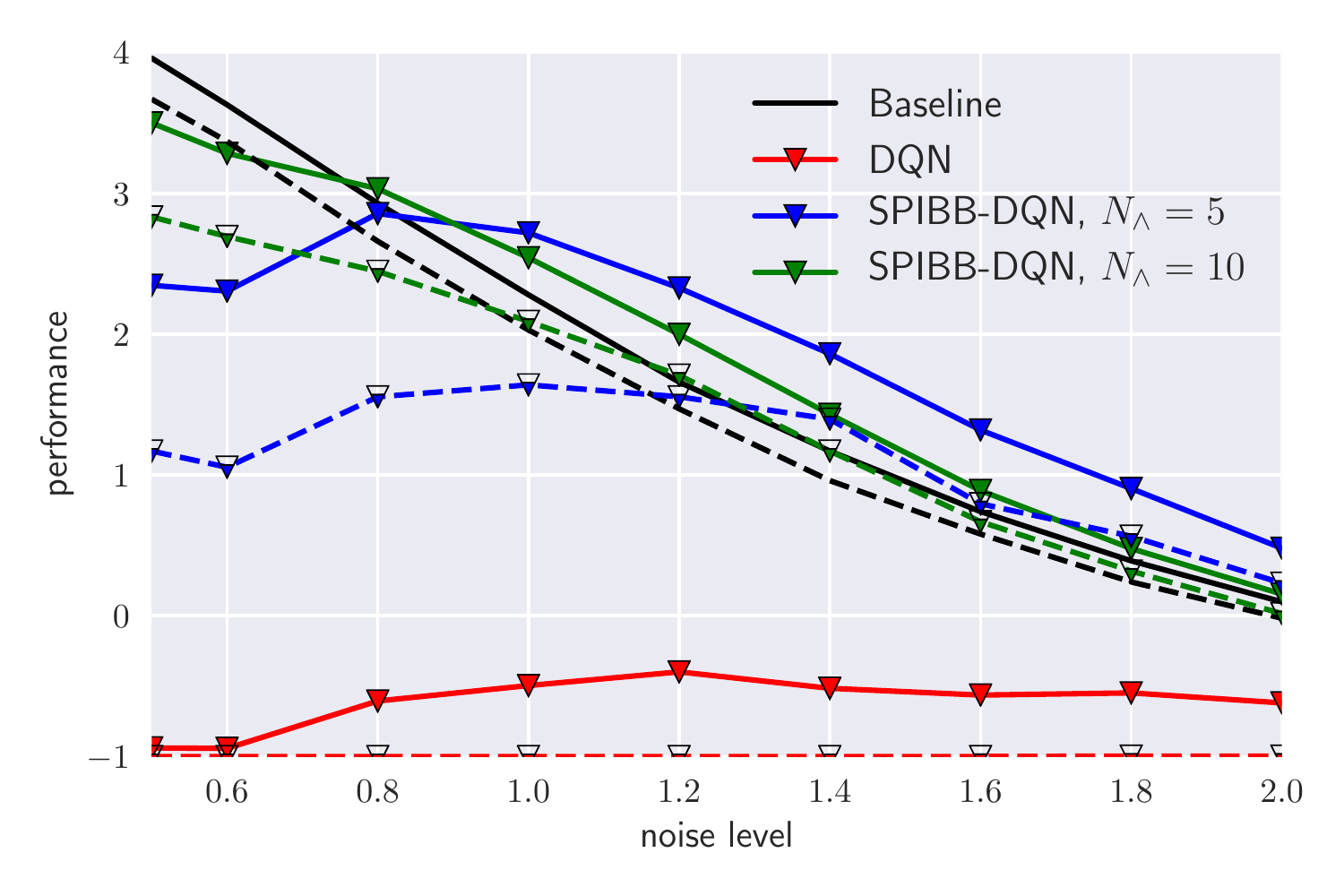}
			\label{fig:noise-factor}
		}
		\caption{SPIBB-DQN experiments: Figure (a) is an illustration of the environment. Figure (b) displays the mean and 10\%-CVaR performance as a function of $N_\wedge$ for three dataset sizes. Figure (c) displays the mean and 10\%-CVaR performance for the baseline, vanilla DQN, RaMDP with $\kappa_{adj}=0.01$, SPIBB-DQN with $N_\wedge = 5$, and with $N_\wedge = 10$, as a function of the transition noise factor.}
		\label{fig:spibb_dqn}
% 		\vspace{-10pt}
	\end{figure*}
	
    \subsection{Does SPIBB achieve SPI in most domains?}
    \label{sec:RandomMDPs_exp}
    % benchmark mean ratio = 0.9
    % benchmark 1% ratio = 0.1
    % benchmark 1% ratio = 0.9
    % heatmap ratios 1% pi_<b N_wedge=10
    % heatmap ratios 1% RaMDP
    % heatmap N_wedge 1% pi_<b ratio = 0.9
    In this section, we study the conditions required on the environment and on the baseline for SPIBB to be helpful. To do so, we use a generator of Random MDPs where the number of states has been fixed to $|\mathcal{X}|=50$, the number of actions to $|\mathcal{A}|=4$ and the connectivity of the transition function to 4. This means that for a given state-action pair $(x,a)$, its transition function $P(x'|x,a)$ is non-zero on four states $x'$ only. The initial state is fixed at $x_0$. The reward function is 0 everywhere except when entering the terminal state, where it equals 1. The terminal state is chosen in such a way that the optimal value function is minimal. It coarsely amounts to choosing the state that is the hardest to reach/farthest from $x_0$.
    %Thus, following the optimal policy, the average length of the trajectories is around 10.
    For a randomly generated MDP $M$, we generate baselines with different levels of performance (the process is detailed in Appendix~\ref{sup:baselinegen}). Specifically, we set a target performance for the baseline based on a hyper-parameter $\eta\in \left\{0.1,0.2,0.3,0.4,0.5,0.6,0.7,0.8,0.9\right\}$: $\rho(\pi_b,M) = \eta\rho(\pi^*,M) + (1-\eta)\rho(\tilde{\pi},M)$, where $\pi^*$ and $\tilde{\pi}$ are respectively the optimal and the uniform policies.
    
    Figure \subref*{fig:benchmark_randomMDP_mean_ratio=0.1} shows the mean results with a bad, highly stochastic baseline ($\eta=0.1$). Since, the baseline is bad, it is an easy task to safely improve it. Basic RL and RaMDP dominate the benchmark in mean, but also in safety (not shown). SPIBB algorithms are too conservative for small datasets but catch up on the bigger ones. Figure \subref*{fig:benchmark_randomMDP_mean_ratio=0.9} shows the mean results with a very good baseline, therefore  very hard task to safely improve. On average, the podium is composed by $\Pi_{\leq b}$-SPIBB, RaMDP, $\Pi_{b}$-SPIBB, followed closely by Basic RL. But, when one considers more specifically the 1\%-CVaR performance, all fail to be safe but the SPIBB algorithms. Note that a -0.5 normalized performance is still a good performance, and that this loss is actually predicted by the theory: Theorem \ref{th:safepolicyimprovement-pi} proves a $\zeta$-approximate safe policy improvement.
    
    The heatmaps shown in Figures \subref*{fig:heatmap_RaMDP} and \subref*{fig:heatmap_pi_<b_randomMDP_N=10} allow us to compare more globally the 1\%-CVaR performance of RaMDP and $\Pi_{\leq b}$-SPIBB. One observes that the former is unsafe in a large area of the map (where it is red, for high $\eta$ or small datasets), while the latter is safe everywhere. Figure \subref*{fig:heatmap_pi_<b_randomMDP_ratio=0.9} displays a heatmap of the $\Pi_{\leq b}$-SPIBB 1\%-CVaR performance in the hardest scenario ($\eta=0.9$) in function of its $N_\wedge$ hyper-parameter. Unsurprisingly, the algorithm becomes slightly unsafe when $N_\wedge$ gets too low. As it increases, the red stains disappear meaning that it becomes completely safe. The green sections show that it still allows for some policy improvement. Full results may be found in Appendix~\ref{sup:random_MDPs_full_results}.

	\subsection{Does SPIBB scale to larger tasks?}
	\label{sec:spibb-dqn-exp}
% 	DQN unstability on 1 dataset
% 	10,000 multi-centile avec N_\wedge qui varie
%   une courbe par taille de dataset mean
%   une courbe par taille de dataset decile

    For the sake of simplicity and to be able to repeat several runs of each experiment efficiently, instead of applying pseudo-count methods from the literature~~\cite{Bellemare2016,Fox2018,Burda2019}, we consider here a pseudo-count heuristic based on the Euclidean state-distance, and a task where it makes sense to do so. The pseudo-count of a state-action $(x,a)$ is defined as the sum of its similarity with the state-action pairs $(x_i,a_i)$ found in the dataset. The similarity between $(x,a)$ and $(x_i,a_i)$ is equal to 0 if $a_i\neq a$, and to $\max(0,1-d(x,x_i))$ otherwise, where $d(\cdot,\cdot)$ is the Euclidean distance between two states.
	
	We consider a helicopter navigation task (see Figure \subref*{fig:helicopter}). The helicopter starts from a random position in the teal area, with a random initial velocity. The 9 available actions consist in applying thrust: backward, no, or forward acceleration, along the two dimensions. The episode terminates when the velocity exceeds some maximal value, in which case it gets a -1 reward, or when the helicopter leaves the blue area, in which case it gets a reward as chromatically indicated on Figure \subref*{fig:helicopter}. The dynamics of the domain follow the basic laws of physics with a Gaussian centered additive noise both on the position and the velocity, see Appendix~\ref{sup:dummy-parameters} for full details of the domain. To train our algorithms, we use a discount factor $\gamma=0.9$, but we report in our results the undiscounted final reward. The baseline is generated as follows: we first train a policy with online DQN, stop before full convergence and then apply a softmax on the obtained $Q$-network. Our experiments consist in 300 runs on SPIBB-DQN with a range of $N_\wedge$ values and for different dataset sizes. SPIBB-DQN with $N_\wedge=0$ is equivalent to vanilla DQN. We also tried RaMDP with several values of $\kappa_{adj}\in[0.001,0.1]$ without any success. For figure clarity, we do not report RaMDP in the Main Document figures. The set of used parameters and the results of the preliminary experiments are reported in Appendices~\ref{sup:dqn-parameters} and \ref{sup:DQN-figures}.
	
	Figure \subref*{fig:10k-20k-30k-datasets} displays the mean and 10\%-CVaR performances in function of $N_\wedge$ for three dataset sizes (10k, 20k, and 30k). We observe that vanilla DQN ($N_\wedge=0$) significantly worsens the baseline in mean and achieves the worst possible 10\%-CVaR performance. SPIBB-DQN not only significantly improves the baseline in mean performance for $N_\wedge\geq 1$, but also in 10\%-CVaR when $N_\wedge\geq 8$. The discerning reader might wonder about the CVaR curve for the baseline. It is explained by the fact that the evaluation of the policies are not exact. The curve accounts for the evaluation errors, errors also obviously encountered with the trained policies.
	
    We performed an additional experiment. Keeping the baseline identical, we trained on 10k-transitions datasets obtained from environments with a different transition noise. Figure \subref*{fig:noise-factor} shows the mean and 10\%-CVaR performances for the baseline, vanilla DQN, and SPIBB-DQN with $N_\wedge\in\left\{5,10\right\}$. First, we observe that vanilla DQN performs abysmally. Second, we see that the baseline quickly gets more efficient when the noise is removed making the safe policy improvement task harder for SPIBB-DQN. SPIBB is efficient at dealing with stochasticity, the noise attenuation reduces its usefulness. Third, as we get to higher noise factors, the stochasticity becomes too high to efficiently aim at the goal, but SPIBB algorithms still succeed at safely improving the baseline.
	
% 	Our first experiment consists in generating a single dataset of 10k trajectories with the baseline and train on 300 different seeds DQN and $\Pi_b$-SPIBB-DQN with different values of $N_\wedge$. Results presented on Figure \subref*{fig:dqn-single-dataset} show that DQN is unstable and learns policies with inconsistent performance, which on average deteriorate the baseline policy. In comparison, $\Pi_b$-SPIBB-DQN with $5 \leq N_\wedge \leq 10$ is stable and safely improves the baseline.
	
% 	Our second experiment consists in generating 20 datasets of 10k trajectories and training the algorithms on 15 different seeds. We did so instead of generating 300 datasets because the dataset generation and their pseudo-counts are actually what take the most time. Results presented on Figure \subref*{fig:10k-datasets} show that DQN performs really bad. In comparison, $\Pi_b$-SPIBB-DQN with $5 \leq N_\wedge \leq 10$ is still stable and safely improves the baseline. We repeated the experiment for different dataset sizes. It appeared that smaller datasets (1k, 2k) do not allow $\Pi_b$-SPIBB-DQN to improve the baseline, but it still does so in a safe way. From 5k trajectories, we start seeing some improvement. Datasets of lager size: 20k, 30k, and 50k enable further improvement in $\Pi_b$-SPIBB-DQN but not for basic DQN (see Figure \subref*{fig:50k-datasets}).
	
	\section{Conclusion and Future Work}
	\label{sec:conclusion}
	In this paper, we tackle the problem of safe Batch Reinforcement Learning. We reformulate the percentile criterion without compromising its safety. We lose optimality that way but keep a PAC-style guarantee of policy improvement. It allows the implementation of an algorithm $\Pi_{b}$-SPIBB that run as fast as a vanilla model-based RL algorithm, while generating a provably safe policy improvement over a known baseline $\pi_b$. A variant algorithm $\Pi_{\leq b}$-SPIBB is shown to perform better and safer on a wide range of domains, but does not come with safety guarantees. Basic Batch RL and the other benchmark competitors are shown to fall short on at least one, and generally two, of the following criteria: mean performance, safety, or domain-dependent hyper-parameter sensitivity. Finally, we implement a DQN version of SPIBB that is the first deep batch algorithm allowing policy improvement in a safe manner.

% 	Future work includes developing model-free versions of our algorithms in order to ease their use in continuous state MDPs and complex real-world applications, with state representation approximation, using density networks~\cite{Veness2012} to compute pseudo-counts~\cite{Bellemare2016}, in a similar way to that of optimism-motivated online RL~\cite{Osband2016,Laroche2017aaai}.

	% A bad dialogue model may easily lead an unsafe RL algorithm to try actions that are not relevant in the context: for example, the good-bye action is generally performed in late stages, when the task has been successful, a bad generalisation would classically be to say good-bye at the beginning of the dialogue.
	
	\newpage
% 	\small
	\bibliographystyle{icml2019}
	\bibliography{ICML-SPIBB}
	\normalsize
	
%\end{document}
	\newpage
	\onecolumn
	\appendix
	\setcounter{proposition}{0}
	\setcounter{lemma}{0}
	\setcounter{theorem}{0}
	
	\section{SPIBB Theory Complements}
	\label{sup:proofs}
	
	\subsection{The MDP framework}
	\label{sup:MDP}
	Markov Decision Processes \citep[MDPs]{bellman1957markovian} are a widely used framework to address the problem of optimizing a sequential decision making problem. In this work, we assume that the true environment is modelled as an unknown finite MDP $M^*=\langle \mathcal{X}, \mathcal{A},R^*,P^*,\gamma \rangle$, where $\mathcal{X}$ is the finite state space, $\mathcal{A}$ is the finite action space, $R^*(x,a)\in[-R_{max},R_{max}]$ is the true bounded stochastic reward function, $P^*(\cdot|x,a)$ is the true transition distribution, and $\gamma\in[0,1)$ is the discount factor. %We consider a Dirichlet prior on the transitions. 
	Without loss of generality, we assume that the process deterministically begins in state $x_{0}$. The agent then makes a decision about which action $a_{0}$ to select. This action leads to a new state that depends on the transition probability and the agent receives a reward $R^*(x_{0},a_{0})$. This process is then repeated until the end of the episode. We denote by $\pi$ the policy which corresponds to the decision making mechanism that assigns actions to states. $\Pi=\{\pi:\mathcal{X}\rightarrow \Delta_{\mathcal{A}}\}$ denotes the set of stochastic policies, and $\Delta_{\mathcal{A}}$ denotes the set of probability distributions over the set of actions $\mathcal{A}$.
	
	The state value function $V_M^\pi(x)$ (resp. state-action value function $Q_M^\pi(x,a)$) is the expectation of the discounted sum of rewards when following $\pi\in\Pi$, starting from state $x\in\mathcal{X}$ (resp. performing action $a\in\mathcal{A}$ in state $x\in\mathcal{X}$) in the MDP $M=\langle \mathcal{X}, \mathcal{A},R,P,\gamma  \rangle$:
	\begin{align}
	    V_M^\pi(x) = \sum_{a\in\mathcal{A}} \pi(a|x) Q_M^\pi(x,a) = \mathbb{E}_{M,\pi,x}\left[\sum_t \gamma^t r_t\right]\!\!. %\nonumber
	\end{align}
	
	The goal of a reinforcement learning algorithm is to discover the unique optimal state value function $V^*_M$ (resp. action-state value function $Q_M^*$) and/or a policy that implements it. We define the performance of a policy by its expected return, starting from the initial state: $\rho(\pi,M) = V^\pi_M(x_0)$. Given a policy subset $\Pi'\subseteq\Pi$, a policy $\pi'$ is said to be $\Pi'$-optimal for an MDP $M$ when it maximizes its performance on $\Pi'$: $\rho(\pi',M)=\max_{\pi\in\Pi'}\rho(\pi,M)$. We will also make use of the notation $V_{max}$ as a known upper bound of the return's absolute value: $V_{max}\leq\frac{R_{max}}{1-\gamma}$.
	
	In this paper, we focus on the batch RL setting where the algorithm does its best at learning a policy from a fixed set of experience. Given a dataset of transitions $\mathcal{D}=\langle x_j,a_j,r_j,x'_j\rangle_{j\in\llbracket 1,N\rrbracket}$, we denote by $N_{\mathcal{D}}(x,a)$ the state-action pair counts; and by $\widehat{M}=\langle \mathcal{X}, \mathcal{A},\widehat{R},\widehat{P},\gamma \rangle$ the Maximum Likelihood Estimation (MLE) MDP of the environment, where $\widehat{R}$ is the reward mean and $\widehat{P}$ is the transition statistics observed in the dataset. Vanilla batch RL, referred hereinafter as Basic RL, looks for the optimal policy in $\widehat{M}$. This policy may be found indifferently using dynamic programming on the explicitly modelled MDP $\widehat{M}$, $Q$-learning with experience replay until convergence~\cite{Sutton1998}, or Fitted-$Q$ Iteration with a one-hot vector representation of the state space~\cite{ernst2005tree}.
	
	\subsection{Matrix notations for the proofs}
	
	The proofs make use of the matrix representation for the $Q$-function, $V$-function, the policy, the transition, the reward and the discount rate (when dealing with semi-MDPs) functions. 
	
	The $Q$-functions matrices have 1 row and $|\mathcal{X}||\mathcal{A}|$ columns.
	
	The $V$-functions matrices have 1 row and $|\mathcal{X}|$ columns. 
	
	The policy matrices $\pi$ have $|\mathcal{X}||\mathcal{A}|$ row and $|\mathcal{X}|$ columns. Even though a policy is generally defined as a function from $\mathcal{X}$ to $\mathcal{A}$ and should be represented by a compact matrix with $|\mathcal{A}|$ rows and $|\mathcal{X}|$ columns, in order to use simple matrix operators, we need the policy matrix to output a distribution over the state-action pairs. Consequently, our policy matrix obtained through the following expansion through the diagonal:
	\begin{equation*}
		\begin{bmatrix*}
			\pi_{11} &  \dots & \pi_{1j} &  \dots & \pi_{1|\mathcal{X}|} \\
			\vdots &  & \vdots &  & \vdots \\
			\pi_{i1} &  \dots & \pi_{ij} &  \dots & \pi_{i|\mathcal{X}|} \\
			\vdots &  & \vdots &  & \vdots \\
			\pi_{|\mathcal{A}|1} &  \dots & \pi_{|\mathcal{A}|j} &  \dots & \pi_{|\mathcal{A}||\mathcal{X}|}
		\end{bmatrix*}
		=
		\begin{bmatrix*}
			\bm{\pi}_{\cdot 1} &  \dots & \bm{\pi}_{\cdot j} & \dots & \bm{\pi}_{\cdot|\mathcal{X}|} 
		\end{bmatrix*} 
		\longrightarrow 
		\begin{bmatrix*}
			\bm{\pi}_{\cdot 1} & & 0 & & 0 \\
			& \ddots & & &  \\
			0 & & \bm{\pi}_{\cdot j} & & 0 \\
			& & & \ddots & \\
			0 & & 0 & & \bm{\pi}_{\cdot|\mathcal{X}|}.
		\end{bmatrix*}
	\end{equation*}
		
	The transition matrices $P$ have $|\mathcal{X}|$ rows and $|\mathcal{X}||\mathcal{A}|$ columns. 
	
	The reward matrices $R$ have 1 row and $|\mathcal{X}||\mathcal{A}|$ columns. 
	
	The discount rate matrices $\Gamma$ have $|\mathcal{X}|$ rows and $|\mathcal{X}||\mathcal{A}|$ columns. 
	
	The expression $AB$ is the matrix product between matrices $A$ and $B$ for which column and row dimensions coincide. 
	
	The expression $(A\circ B)$ is the element-wise product between matrices $A$ and $B$ of the same dimension. 
	
	$\mathbb{I}$ denotes the identity matrix (the diagonal matrix with only ones), which dimension is given by the context. 
		
	$\mathds{1}(y)$ denotes the column unit vector with zeros everywhere except for the element of index $y$ which equals 1. For instance $Q\mathds{1}_{x,a}$ denotes the value of performing action $a$ in state $x$.
	
	The regular and option Bellman equations are therefore respectively written as follows:
	\begin{align}
	Q &= R + \gamma Q\pi P \\
	Q &= R + Q\pi (\Gamma\circ P)
	\end{align}

\subsection{Convergence and safe policy improvement of $\Pi_b$-SPIBB}
\label{sup:th-proofs}

\begin{lemma}[$Q$-function error bounds with $\Pi_b$-SPIBB]
	Consider two semi-MDPs $M_1=\langle \mathcal{X},\Omega_{\mathcal{A}},P_1,R_1,\Gamma_1\rangle$ and $M_2=\langle \mathcal{X},\Omega_{\mathcal{A}},P_2,R_2,\Gamma_2\rangle$. Consider a policy $\pi$. Also, consider $Q_1$ and $Q_2$ be the state-action value function of the policy $\pi$ in $M_1$ and $M_2$, respectively. If:
	\begin{align}
		\forall a\in\mathcal{A}, \forall x\in\mathcal{I}_a, 
		\left\{
		\begin{array}{l}
			|R_1\mathds{1}_{x,o_a} - R_2\mathds{1}_{x,o_a}| \leq \epsilon R_{max} \\
			||(\Gamma_1 \circ P_1)\mathds{1}_{x,o_a} - (\Gamma_2 \circ P_2)\mathds{1}_{x,o_a} ||_1 \leq \epsilon,
		\end{array}
		\right.
	\end{align}
	then, we have:	
	\begin{equation}
		\forall a\in\mathcal{A}, \forall x\in\mathcal{I}_a, \quad |Q_1\mathds{1}_{x,o_a}-Q_2\mathds{1}_{x,o_a}| \leq \cfrac{2\epsilon V_{max}}{1-\gamma},
	\end{equation}
	where $V_{max}$ is the known maximum of the value function.
	\label{lem:q1-q2}
\end{lemma}
\begin{proof}
	We adopt the matrix notations. The difference between the two state-option value functions can be written:
	\begin{align}
	Q_1 - Q_2 &= R_1 + Q_1\pi (\Gamma_1\circ P_1) - R_2 - Q_2\pi (\Gamma_2\circ P_2) \\
	&= R_1 + Q_1\pi (\Gamma_1\circ P_1) - R_2 - Q_2\pi (\Gamma_2\circ P_2) + Q_2\pi (\Gamma_1\circ P_1) - Q_2\pi (\Gamma_1\circ P_1) \\
	&= R_1 - R_2 + (Q_1-Q_2)\pi (\Gamma_1\circ P_1) + Q_2\pi ((\Gamma_1\circ P_1) - (\Gamma_2\circ P_2)) \\
	&= \left[R_1 - R_2 + Q_2\pi ((\Gamma_1\circ P_1) - (\Gamma_2\circ P_2))\right](\mathbb{I}-\pi(\Gamma_1\circ P_1))^{-1}.
	\label{eq:q1-q2}
	\end{align}
	
	Now using Holder's inequality and the second assumption, we have:
	\begin{equation}
	\lvert Q_2\pi ((\Gamma_1\circ P_1) - (\Gamma_2\circ P_2))\mathds{1}_{x,o_a} \rvert \leq \lVert Q_2\rVert_\infty \lVert \pi \rVert_\infty \lVert (\Gamma_1\circ P_1)\mathds{1}_{x,o_a} - (\Gamma_2\circ P_2)\mathds{1}_{x,o_a} \rVert_1 \leq \epsilon V_{max}.
	\label{eq:p1-p2-pi}
	\end{equation}
	
	Inserting~\eqref{eq:p1-p2-pi} into Equation~\eqref{eq:q1-q2} and using the first assumption, we obtain:
	\begin{align}
	|Q_1\mathds{1}_{x,o_a}-Q_2\mathds{1}_{x,o_a}| &\leq \left[\epsilon R_{max} + \epsilon V_{max}\right] \lVert(\mathbb{I}-\pi(\Gamma_1\circ P_1))^{-1}\mathds{1}_{x,o_a}\rVert_1\\
	&\leq \cfrac{2\epsilon V_{max}}{1-\gamma},
	\end{align}
	which proves the lemma. There is a factor 2 that might require some discussion. It comes from the fact that we do not control that the maximum $R_{max}$ might be as big as $V_{max}$ in the semi-MDP setting and we do not control the $\gamma$ factor in front of the second term. As a consequence, we surmise that a tighter bound down to $\frac{\epsilon V_{max}}{1-\gamma}$ holds, but this still has to be proven.
\end{proof}

	\begin{proposition}
		Consider an environment modelled with a semi-MDP~\cite{Parr1998,Sutton1999} $\ddot{M}=\langle \mathcal{X},\Omega_{\mathcal{A}},\ddot{P}^*,\ddot{R}^*,\Gamma^*\rangle$,  where $\Gamma^*$ is the discount rate inferior or equal to $\gamma$ that varies with the state action transitions and the empirical semi-MDP $\widehat{\ddot{M}}=\langle \mathcal{X},\Omega_{\mathcal{A}},\widehat{\ddot{P}},\widehat{\ddot{R}},\widehat{\Gamma}\rangle$ estimated from a dataset $\mathcal{D}$. If in every state $x$ where option $o_a$ may be initiated: $x\in \mathcal{I}_a$, we have:
		\begin{equation}
			\sqrt{\cfrac{2}{N_{\mathcal{D}}(x,a)}\log\cfrac{2|\mathcal{X}||\mathcal{A}|2^{|\mathcal{X}|}}{\delta}}\leq\epsilon,
		\end{equation}
		then, with probability at least $1-\delta$:
		\begin{align}
		\forall a\in\mathcal{A}, \forall x\in\mathcal{I}_a, \left\{
		\begin{array}{ll}
			\lVert \Gamma^* \ddot{P}^*(x,o_a)-\widehat{\Gamma}\widehat{\ddot{P}}(x,o_a)\rVert_1 \leq \epsilon \\
			\lvert \ddot{R}^*(x,o_a)-\widehat{\ddot{R}}(x,o_a)\rvert \leq \epsilon \ddot{R}_{max}
		\end{array}
		\right.
		\end{align}
		\label{prop:eps_pib}
	\end{proposition}
	\begin{proof}
		The proof is similar to that of Proposition 9 in \citet{ghavamzadeh2016safe}.
	\end{proof}
	
\begin{lemma}[Safe policy improvement of $\pi^\odot_{spibb}$ over any policy $\pi\in\Pi_b$]
	Let $\Pi_b$ be the set of policies under the constraint of following $\pi_b$ when $(x,a)\in\mathfrak{B}$. Let $\pi^\odot_{spibb}$ be a $\Pi_b$-optimal policy of the reward maximization problem of an estimated MDP $\widehat{M}$. Then, for any policy $\pi\in\Pi_b$, the difference of performance between $\pi^\odot_{spibb}$ and $\pi$ is bounded as follows with high probability $1-\delta$ in the true MDP $M^*$:
	\begin{equation}
		\rho(\pi^\odot_{spibb}, M^*) - \rho(\pi, M^*) \geq  \rho(\pi^\odot_{spibb}, \widehat{M}) - \rho(\pi, \widehat{M}) - \cfrac{4 \epsilon V_{max}}{1-\gamma}.
	\end{equation}
	\label{lem:nearopt-pi}
\end{lemma}
\begin{proof}
	We transform the true MDP $M^*$ and the MDP estimate $\widehat{M}$, to their bootstrapped semi-MDP counterparts $\ddot{M}^*$ and the MDP estimate $\widehat{\ddot{M}}$. In these semi-MDPs, the actions $\mathcal{A}$ are replaced by options $\Omega_{\mathcal{A}} = \left\{o_a\right\}_{a\in\mathcal{A}}$ constructed as follows:
	\begin{align}
		o_a = \langle \mathcal{I}_a, {a\!\!:\!\!\pi_b},\beta \rangle = \left\{
		\begin{array}{l}
		\mathcal{I}_a=\left\{x\in\mathcal{X}\text{, such that } (x,a)\notin\mathfrak{B}\right\}\\
		{a\!\!:\!\!\tilde{\pi}_b} = \text{ perform } a \text{ at initialization, then follow } \tilde{\pi}_b \\
		\beta(x) = \lVert \dot{\pi}_b(x,\cdot) \rVert_1
		\end{array}
		\right.
	\end{align}
	where $\pi_b$ has been decomposed as the aggregation of two partial policies: $\pi_b = \dot{\pi}_b + \tilde{\pi}_b$, with $\dot{\pi}_b(a|x)$ containing the non-bootstrapped actions probabilities in state $x$, and $\tilde{\pi}_b(a|x)$ the bootstrapped actions probabilities:
	\begin{align}
		\forall a\in\mathcal{A}, 
		\left\{
			\begin{array}{ll}
				\dot{\pi}(a|x) = \pi(a|x) &\textnormal{if } (x,a) \notin \mathfrak{B} \\ \dot{\pi}(a|x) = 0 &\textnormal{if } (x,a) \in \mathfrak{B}
			\end{array}
		\right. \\
		\forall a\in\mathcal{A}, 
		\left\{
			\begin{array}{ll}
				\tilde{\pi}(a|x) = \pi(a|x) &\textnormal{if } (x,a) \in \mathfrak{B} \\ \tilde{\pi}(a|x) = 0 &\textnormal{if } (x,a) \notin \mathfrak{B}
			\end{array}
		\right.
	\end{align}
	
	Let $\ddot{\Pi}$ denote the set of policies over the bootstrapped semi MDPs. $\mathcal{I}_a$ is the initialization function: it determines the set of states where the option is available. ${a\!\!:\!\!\pi_b}$ is the option policy being followed during the length of the option. Finally, $\beta(x)$ is the termination function defining the probability of the option to terminate in each state.
	
	Notice that all options have the same termination function. Please, also notice that some states might have no available options, but this is okay since every option has a termination function equal to 0 in those states, meaning that they are unreachable. This to avoid being in this situation at the beginning of the trajectory, we use the notion of starting option: a trajectory starts with a void option $o_\emptyset = \langle \left\{x_0\right\}, \pi_b,\beta \rangle$.
	
	By construction $x\in\mathcal{I}_a$ if and only if $(x,a)\notin\mathfrak{B}$, \emph{i.e.} if and only if the condition on the state-action counts of Proposition~\ref{prop:eps_pib} is fulfilled\footnote{Also, note that there is the requirement here that the trajectories are generated under policy $\pi_b$, so that the options are consistent with the dataset.}. Also, any policy $\pi\in\Pi_b$ is implemented by a policy $\ddot{\pi}\in\ddot{\Pi}$ in a bootstrapped semi-MDP. Reciprocally, any policy $\ddot{\pi}\in\ddot{\Pi}$ admits a policy $\pi\in\Pi_b$ in the original MDP.
	
	Note also, that by construction, the transition and reward functions are only defined for $(x,o_a)$ pairs such that $x\in\mathcal{I}_a$. By convention, we set them to 0 for the other pairs. Their corresponding $Q$-functions are therefore set to 0 as well. 
	
	This means that Lemma~\ref{lem:q1-q2} may be applied with $\pi=\pi^\odot_{spibb}$ and $M_1 = \ddot{M}^*$ and $M_2 = \widehat{\ddot{M}}$. We have:
	\begin{align}
	\lvert\rho(\pi^\odot_{spibb}, M^*) - \rho(\pi^\odot_{spibb}, \widehat{M})\rvert &= \lvert\rho(\pi^\odot_{spibb}, \ddot{M}^*) - \rho(\pi^\odot_{spibb}, \widehat{\ddot{M}})\rvert \\
	&= \lvert V^{\pi^\odot_{spibb}}_{\ddot{M}^*}(x_0) - V^{\pi^\odot_{spibb}}_{\widehat{\ddot{M}}}(x_0)\rvert \\
	&= \lvert Q^{\pi^\odot_{spibb}}_{\ddot{M}^*}(x_0,o_\emptyset) - Q^{\pi^\odot_{spibb}}_{\widehat{\ddot{M}}}(x_0,o_\emptyset)\rvert \\
	&\leq \cfrac{2\epsilon V_{max}}{1-\gamma}
	\label{eq:Qodot}
	\end{align}
	
	Analogously to~\ref{eq:Qodot}, for any $\pi\in\Pi_b$, we also have:
	\begin{equation}
	\lvert\rho(\pi, M^*) - \rho(\pi, \widehat{M})\rvert 
	\leq \cfrac{2\epsilon V_{max}}{1-\gamma}
	\label{eq:Qstar}
	\end{equation}
	
	Thus, we may write:
	\begin{align}
	\rho(\pi^\odot_{spibb}, M^*) - \rho(\pi, M^*) &\geq \rho(\pi^\odot_{spibb}, \widehat{M}) - \rho(\pi, \widehat{M}) - \cfrac{4\epsilon V_{max}}{1-\gamma},
	\end{align}
	where the inequality is directly obtained from equations~\ref{eq:Qodot} and~\ref{eq:Qstar}.
\end{proof}

\begin{theorem}[Convergence of $\Pi_b$-SPIBB]
	$\Pi_b$-SPIBB converges to a policy $\pi^\odot_{spibb}$ that is $\Pi_b$-optimal in the MLE MDP $\widehat{M}$.
\end{theorem}
\begin{proof}
	We use the same transformation of $\widehat{M}$ as in Lemma~\ref{lem:nearopt-pi}. Then, the problem is cast without any constraint in a well defined semi-MDP, and Policy Iteration is known to converge in semi-MDPs to the policy optimizing the value function~\cite{Gosavi2004}.
\end{proof}

\begin{theorem}[Safe policy improvement of $\Pi_b$-SPIBB]
	Let $\Pi_b$ be the set of policies under the constraint of following $\pi_b$ when $(x,a)\in\mathfrak{B}$. Then, $\pi^\odot_{spibb}$ is a $\zeta$-approximate safe policy improvement over the baseline $\pi_b$ with high probability $1-\delta$, with:
	\begin{equation}
	\zeta = \cfrac{4 V_{max}}{1-\gamma} \sqrt{\cfrac{2}{N_{\wedge}}\log\cfrac{2|\mathcal{X}||\mathcal{A}|2^{|\mathcal{X}|}}{\delta}}  - \rho(\pi^\odot_{spibb}, \widehat{M}) + \rho(\pi_b, \widehat{M})
	\end{equation}
\end{theorem}

\begin{proof}
	It is direct to observe that $\pi_b \in \Pi_b$, and therefore that Lemma~\ref{lem:nearopt-pi} can be applied to $\pi_b$. We infer that, with high probability $1-\delta$:
	\begin{equation}
	\rho(\pi^\odot_{spibb}, M^*) - \rho(\pi_b, M^*) \geq  \rho(\pi^\odot_{spibb}, \widehat{M}) - \rho(\pi_b, \widehat{M}) - \cfrac{4 \epsilon V_{max}}{1-\gamma}.
	\end{equation}
	with:
	\begin{equation}
		\epsilon = \sqrt{\cfrac{2}{N_\wedge}\log\cfrac{2|\mathcal{X}||\mathcal{A}|2^{|\mathcal{X}|}}{\delta}}
	\end{equation}
	
	Therefore, we obtain:
	\begin{align}
	\zeta &= \cfrac{4 \epsilon V_{max}}{1-\gamma} -\left(\rho(\pi^\odot_{spibb}, \widehat{M}) - \rho(\pi_b, \widehat{M})\right) \\
	&= \cfrac{4 V_{max}}{1-\gamma} \sqrt{\cfrac{2}{N_\wedge}\log\cfrac{|\mathcal{X}||\mathcal{A}|2^{|\mathcal{X}|}}{\delta}} - \rho(\pi^\odot_{spibb}, \widehat{M}) + \rho(\pi_b, \widehat{M})
	\end{align}
	
	\textit{Quod erat demonstrandum}.
\end{proof}

	\begin{theorem}
	    In finite MDPs, Equation \ref{eq:spibb-DQN} admits a unique fixed point that coincides with the $Q$-value of the policy trained with model-based $\Pi_b$-SPIBB.
	    \label{th:free}
	\end{theorem}
	\begin{proof}
	    Unicity of the fixed point is a classical result in RL, obtained from the fact that the Bellman operator is a contraction. 
	    
	    Let $\pi_t$ denote the policy trained with model-based $\Pi_b$-SPIBB. By construction, we know that $\pi_t$ satisfies the optimal Bellman equation in the MLE MDP, under the $\Pi_b$-constraint:
	    \begin{align}
	        Q^{\pi_t}_{\widehat{M}} = \widehat{R} + \gamma Q^{\pi_t}_{\widehat{M}} \pi_t \widehat{P}
	    \end{align}
	    
	    Moreover, $\pi_t$ may be decomposed by its component $\tilde{\pi}_t$ on $\mathfrak{B}$ and its complementary $\dot{\pi}_t$:
	    \begin{align}
	        \pi_t(a|x) &= \tilde{\pi}_t(a|x) + \dot{\pi}_t(a|x) \\
	        &\textnormal{with: } \left\{\begin{array}{l}
	             \tilde{\pi}_t(a|x) = \left\{
	                \begin{array}{l}
	                    \pi_b(a|x) \textnormal{ if } (x,a)\in\mathfrak{B} \\
    	                0 \textnormal{ otherwise}
        	        \end{array}
        	        \right.  \\
	             \dot{\pi}_t(a|x) = \left\{
	                \begin{array}{l}
	                    \sum_{a'|(x,a')\notin\mathfrak{B}} \pi_b(a'|x) \textnormal{ if } a = \argmax_{a'|(x,a')\notin\mathfrak{B}}  Q^{\pi_t}_{\widehat{M}}(x,a) \\
    	                0 \textnormal{ otherwise}
        	        \end{array}
        	        \right.
	        \end{array}\right.
	    \end{align}
	    
	    As a consequence, we obtain:
	    \begin{align}
	        Q^{\pi_t}_{\widehat{M}}(x,a) &= \widehat{R}(x,a) + \gamma \sum_{x'\in\mathcal{X}} \sum_{a'\in\mathcal{A}}Q^{\pi_t}_{\widehat{M}}(x',a') \pi_t(a'|x') \widehat{P}(x'|x,a)\\
	        &= \widehat{R}(x,a) + \gamma \sum_{x'\in\mathcal{X}} \sum_{a'\in\mathcal{A}}Q^{\pi_t}_{\widehat{M}}(x',a') \left(\tilde{\pi}_t(a'|x') + \dot{\pi}_t(a'|x')\right) \widehat{P}(x'|x,a)\\
	        &= \widehat{R}(x,a) + \gamma\sum_{x'\in\mathcal{X}} \widehat{P}(x'|x,a) \left[ \sum_{a'|(x',a')\in\mathfrak{B}} \pi_b(a'|x')Q^{\pi_t}_{\widehat{M}}(x',a') \right] \\
	        &+ \gamma\sum_{x'\in\mathcal{X}} \widehat{P}(x'|x,a) \left[ \left(\sum_{a'|(x',a')\notin\mathfrak{B}} \pi_b(a'|x')\right) \max_{a'|(x',a')\notin\mathfrak{B}}Q^{\pi_t}_{\widehat{M}}(x',a')\right] \nonumber \\
	        &= \cfrac{\displaystyle\sum_{\langle x_j=x,a_j=a, r_j, x'_j\rangle \in \mathcal{D}} r_j}{N_\mathcal{D}(x,a)} + \gamma\sum_{x'\in\mathcal{X}} \cfrac{\displaystyle\sum_{\langle x_j=x,a_j=a, r_j, x'_j=x'\rangle \in \mathcal{D}} 1}{N_\mathcal{D}(x,a)} \left[ \sum_{a'|(x',a')\in\mathfrak{B}} \pi_b(a'|x')Q^{\pi_t}_{\widehat{M}}(x',a') \right] \\
	        &+ \gamma\sum_{x'\in\mathcal{X}} \cfrac{\displaystyle\sum_{\langle x_j=x,a_j=a, r_j, x'_j=x'\rangle \in \mathcal{D}} 1}{N_\mathcal{D}(x,a)} \left[ \left(\sum_{a'|(x',a')\notin\mathfrak{B}} \pi_b(a'|x')\right) \max_{a'|(x',a')\notin\mathfrak{B}}Q^{\pi_t}_{\widehat{M}}(x_j',a')\right] \nonumber 
	    \end{align}
	        
        where $\displaystyle\sum_{\langle x_j=x,a_j=a, r_j, x'_j\rangle \in \mathcal{D}}$ denotes the sum over all transitions in the dataset that start from the state-action pair $(x,a)$ and $\displaystyle\sum_{\langle x_j=x,a_j=a, r_j, x'_j=x'\rangle \in \mathcal{D}}$ is the sum over all transitions that start from the state-action pair $(x,a)$ and transition to $x'$. 
        
        We then see that:
	        
	    \begin{align}
	        Q^{\pi_t}_{\widehat{M}}(x,a) &= \cfrac{\displaystyle\sum_{\langle x_j=x,a_j=a, r_j, x'_j\rangle \in \mathcal{D}} r_j}{N_\mathcal{D}(x,a)} + \cfrac{\gamma}{N_\mathcal{D}(x,a)} \sum_{\langle x_j=x,a_j=a, r_j, x'_j\rangle \in \mathcal{D}}  \sum_{a'|(x'_j,a')\in\mathfrak{B}} \pi_b(a'|x'_j)Q^{\pi_t}_{\widehat{M}}(x'_j,a') \\
	        &+ \cfrac{\gamma}{N_\mathcal{D}(x,a)} \sum_{\langle x_j=x,a_j=a, r_j, x'_j\rangle \in \mathcal{D}} \left(\sum_{a'|(x'_j,a')\notin\mathfrak{B}} \pi_b(a'|x'_j)\right) \max_{a'|(x'_j,a')\notin\mathfrak{B}}Q^{\pi_t}_{\widehat{M}}(x_j',a') \nonumber \\
	        &= \cfrac{1}{N_\mathcal{D}(x,a)}\sum_{\langle x_j=x,a_j=a, r_j, x'_j\rangle \in \mathcal{D}}\left[r_j + \gamma  \sum_{a'|(x'_j,a')\in\mathfrak{B}} \pi_b(a'|x'_j)Q^{\pi_t}_{\widehat{M}}(x'_j,a')\right. \\
	        &\qquad\qquad\qquad\qquad\qquad\qquad\qquad\quad + \left. \gamma \left(\sum_{a'|(x'_j,a')\notin\mathfrak{B}} \pi_b(a'|x'_j)\right) \max_{a'|(x'_j,a')\notin\mathfrak{B}}Q^{\pi_t}_{\widehat{M}}(x_j',a')\right] \nonumber \\
	        &= \cfrac{1}{N_\mathcal{D}(x,a)} \sum_{\langle x_j=x,a_j=a, r_j, x'_j\rangle \in \mathcal{D}} y^{\pi_t}_j \textnormal{ when } N_\mathcal{D}(x,a)>0 \textnormal{ and is undefined otherwise.} 
	    \end{align}
	    
	    This concludes the proof that $Q^{\pi_t}_{\widehat{M}}$ is the fixed point of Equation \ref{eq:spibb-DQN}.
	\end{proof}
	
	\subsection{Algorithms for the greedy projection of $Q^{(i)}$ on $\Pi_{b}$ and $\Pi_{\leq b}$}
	\label{sup:algos}
	The policy-based SPIBB algorithms rely on a policy iteration process that requires a policy improvement step under the constraint that the generated policy belongs to $\Pi_b$ or $\Pi_{\leq b}$. Those are respectively described in Algorithms~\ref{alg:Pibproj} (main document) and~\ref{alg:pileqbproj} (see below).
	
	\begin{pseudocode}[ht!]
		\caption{Greedy projection of $Q^{(i)}$ on $\Pi_{\leq b}$}
		\KwIn{Baseline policy $\pi_{b}$}
		\KwIn{Last iteration value function $Q^{(i)}$}
		\KwIn{Set of bootstrapped state-action pairs $\mathfrak{B}$}
		\KwIn{Current state $x$ and action set $\mathcal{A}$}
		\BlankLine
		Sort $\mathcal{A}$ in decreasing order of the action values: $Q^{(i)}(x,a)$
		
		Initialize $\pi^{(i)}_{spibb} = 0$
		
		\For{$a\in\mathcal{A}$}{
			\eIf{$(x,a) \in \mathfrak{B}$}{
				\eIf{$\pi_b(a|x) \geq 1 - \sum_{a'\in\mathcal{A}}\pi^{(i)}_{spibb}(a'|x)$}{
					$\pi^{(i)}_{spibb}(a|x) = 1 - \sum_{a'\in\mathcal{A}}\pi^{(i)}_{spibb}(a'|x)$
					
					\Return $\pi^{(i)}_{spibb}$
				}{
					$\pi^{(i)}_{spibb}(a|x) = \pi_b(a|x)$
				}
			}{
				$\pi^{(i)}_{spibb}(a|x) = 1 - \sum_{a'\in\mathcal{A}}\pi^{(i)}_{spibb}(a'|x)$
				
				\Return $\pi^{(i)}_{spibb}$
			}
		}
		\label{alg:pileqbproj}
	\end{pseudocode}
	
	\subsection{Comprehensive illustration of the difference between $\Pi_b$-SPIBB and $\Pi_{\leq b}$-SPIBB policy improvement steps}
	\label{sup:pibvspi<b}
	
	Table~\ref{tab:comprehensiveexample} illustrates the difference between $\Pi_b$-SPIBB and $\Pi_{\leq b}$-SPIBB in the policy improvement step of the policy iteration process. It shows how the baseline probability mass is locally redistributed among the different actions for the two policy-based SPIBB algorithms. We observe that for $\Pi_b$-SPIBB, the bootstrapped state-action pairs probabilities remain untouched whatever their $Q$-value estimates are. On the contrary, $\Pi_{\leq b}$-SPIBB removes all mass from the bootstrapped state-action pairs that are performing worse than the current $Q$-value estimates.
	\begin{table*}[ht!]
		\caption{Policy improvement step at iteration $(i)$ for the two policy-based SPIBB algorithms.} 
		\centering
		\setlength\tabcolsep{3pt}
		\def\arraystretch{1.6}
		\small
		\begin{tabular}{|l|l|l|l|l|l|}
			\hline
			\normalsize$Q$-value estimate & \normalsize Baseline policy & \normalsize Boostrapping & \normalsize$\Pi_b$-SPIBB & \normalsize$\Pi_{\leq b}$-SPIBB \\ 
			\hline
			$Q_{\widehat{M}}^{(i)}(x,a_1) = 1$ & $\pi_b(a_1|x) = 0.1$ & $(x,a_1)\in\mathfrak{B}$ & $\pi^{(i+1)}(a_1|x) = 0.1$ &  $\pi^{(i+1)}(a_1|x) = 0$ \\
			$Q_{\widehat{M}}^{(i)}(x,a_2) = 2$ & $\pi_b(a_2|x) = 0.4$ & $(x,a_2)\notin\mathfrak{B}$ & $\pi^{(i+1)}(a_2|x) = 0$ &  $\pi^{(i+1)}(a_2|x) = 0$ \\
			$Q_{\widehat{M}}^{(i)}(x,a_3) = 3$ & $\pi_b(a_3|x) = 0.3$ & $(x,a_3)\notin\mathfrak{B}$ & $\pi^{(i+1)}(a_3|x) = 0.7$ &  $\pi^{(i+1)}(a_3|x) = 0.8$ \\
			$Q_{\widehat{M}}^{(i)}(x,a_4) = 4$ & $\pi_b(a_4|x) = 0.2$ & $(x,a_4)\in\mathfrak{B}$ & $\pi^{(i+1)}(a_4|x) = 0.2$ &  $\pi^{(i+1)}(a_4|x) = 0.2$ \\
			\hline
		\end{tabular}
		\normalsize
		\label{tab:comprehensiveexample}
	\end{table*}

	\newpage
	\section{Finite MDP Benchmark Design}
	\label{sup:expe_finite}
	\subsection{Experiments details}
	\subsubsection{Pseudo code for the Gridworld benchmark}
	\label{sup:gridworld-protocol}
	\SetKwFor{RepTimes}{repeat}{times}{end}
	\begin{pseudocode}[ht!]
		\caption{Gridworld benchmark}
		\KwIn{List of dataset size}
		\KwIn{List of algorithms in the benchmark}
		\KwIn{List of hyper-parameter values for each algorithm}
		\BlankLine
		\RepTimes{$10^5$}{
		        
            \For{each dataset size}{
                
                Generate a dataset. (see Section \ref{sup:datasetgen})
                
                \For{each algorithm}{
                
                    \For{each algorithm hyper-parameter value}{
                    
                        Train a policy. (see Sections \ref{sec:pibootstrap} and \ref{sup:benchmarkalgos})
                        
                        Evaluate the policy. (see Section \ref{sup:evaluationgen})
                        
                        Record the performance of the trained policy.
		            }    
		        }
		    }
		}
		
		\label{alg:maze_benchmark}
	\end{pseudocode}
	
	\subsubsection{Pseudo code for the Random MDPs benchmark}
	
	\SetKwFor{RepTimes}{repeat}{times}{end}
	\begin{pseudocode}[ht!]
		\caption{Random MDPs benchmark}
		\KwIn{List of hyper-parameter values for the baseline}
		\KwIn{List of dataset size}
		\KwIn{List of algorithms in the benchmark}
		\KwIn{List of hyper-parameter values for each algorithm}
		\BlankLine
		\RepTimes{$10^5$}{
            Generate an MDP. (see Section \ref{sup:MDPgen})
            
		    \For{each hyper parameter value for the baseline}{
		        
		        Generate a baseline. (see Section \ref{sup:baselinegen})
		        
		        \For{each dataset size}{
		            
		            Generate a dataset. (see Section \ref{sup:datasetgen})
		            
		            \For{each algorithm}{
		            
		                \For{each algorithm hyper-parameter value}{
		                
		                    Train a policy. (see Sections \ref{sec:pibootstrap} and \ref{sup:benchmarkalgos})
		                    
		                    Evaluate the policy. (see Section \ref{sup:evaluationgen})
		                    
		                    Record the performance of the trained policy.
		                
		                }
		            }
		        }
		    }
		}
		
		\label{alg:garnet_benchmark}
	\end{pseudocode}
	
	\subsubsection{MDP generation}
	\label{sup:MDPgen}
	We use three parameters for our MDP generation: the number of states, the number of actions in each state, and the connectivity of the transition function stating how many states are reachable after performing a given action in a given state. We tried out various values for those parameters and found little sensitivity in those preliminary experimental results and decided to fix their respective values to 25/4/4 in the reported experiments. The discount factor $\gamma$ is set to 0.95.
	
	The initial state is arbitrarily set to be $x_0$, then we search with dynamic programming the performance of the optimal policy for all potential terminal state $x_f \in \mathcal{X}/{x_0}$. We select the terminal state for which the optimal policy yields the smaller value function and set it as terminal: $R(x,a,x_f)=1$ and $P(x|x_f,a)=0$ for all $x\in\mathcal{X}$ and all $a\in\mathcal{A}$. The reward function is set to 0 everywhere else. We found that the optimal value-function is on average $0.6$ and with a surprising low variance, which amounts to an average horizon of 10. Later on, we write this environmental MDP $M^* = \langle \mathcal{X}, \mathcal{A}, P^*, R^*, \gamma \rangle$, its optimal action-value function $Q^*$, its optimal performance $\rho^* = \rho(\pi^*,M^*)$, and its random policy performance  $\widetilde{\rho} = \rho(\widetilde{\pi},M^*)$, where $\widetilde{\pi}$ denotes the uniform random policy: $\widetilde{\pi}(a|x) = \frac{1}{|\mathcal{A}|}$ for all $x\in\mathcal{X}$ and all $a\in\mathcal{A}$.
	
	\subsubsection{Baseline generation}
	\label{sup:baselinegen}
	We use a hyper-parameter for the baseline generation:
	\begin{align}
	    \rho_b = \rho(\pi_b, M^*) = \eta \rho^* + (1-\eta) \widetilde{\rho}.
	\end{align}
	
	Therefore, $\eta \in \left\{0.1,0.2,0.3,0.4,0.5,0.6,0.7,0.8,0.9\right\}$ determines the performance of the baseline, normalized with respect to the performances of the random and the optimal performance. There are an infinite number of policies that yield this performance. We designed several heuristics to generate the actual baseline and again notice a moderate sensitivity in our preliminary results. All the reported results use the following heuristics which consists in two steps: softening and randomization.
	
	\paragraph{Softening:} We apply a softmax on $Q^*$ with temperature such that $\rho(\pi_s,M^*) = \frac{\rho_b + \rho^*}{2}$, where $\pi_s$ denotes the policy obtained after the softening operation.
	
	\paragraph{Randomization:} Until reaching the desired performance for the baseline we repeatedly apply the following process: we randomly select a state $x$, and move a 0.1 probability mass from $a^*=\argmax_{a\in\mathcal{A}} Q^*(x,a)$ to another random action. When this loop stops, the output is the baseline $\pi_b$.
	
	\subsubsection{Dataset generation}
	\label{sup:datasetgen}
	The dataset generation depends a single parameter $|\mathcal{D}|\in\left\{10,20,50,100,200,500,1000,2000\right\}$ ($\cup \left\{5000,10000\right\}$ for the Gridworld experiments): its size expressed in the number of trajectories. A trajectory generation simply consists in sampling the environment and the baseline policy until reaching the final state. The output is the dataset $\mathcal{D}$.
	
	\subsubsection{Trained policy evaluation}
	\label{sup:evaluationgen}
	In the Random MDPs experiments, we use different MDPs and baselines for each run. We need a standardized method for evaluating the trained policy $\pi$. We use the performance normalized with respect to the baseline and optimal policies:
	\begin{align}
	    \rho = \cfrac{\rho(\pi,M^*) - \rho_b}{\rho^*-\rho_b} \leq 1. \label{eq:normalized_perf}
	\end{align}
	
	Then, the results are analyzed with respect to $\rho$ as everywhere else in the paper: according to the mean and CVaR performances.
	
	\subsubsection{Mean and CVaR performance}
	The mean performance is simply the average of performance over all the runs.
	
	The X\%-CVaR performance is the average performance of the X\% worst runs. To compute this, we sort the performance of all the runs, and keep the lowest X\% fraction and then take the average. The 100\%-CVaR performance is obviously equivalent to the mean performance.

	\begin{figure*}[t]
		\centering
		\subfloat[HCPI with $\delta_{hcpi}=0.1$]{
			\includegraphics[trim = 5pt 5pt 5pt 5pt, clip, width=0.5\textwidth]{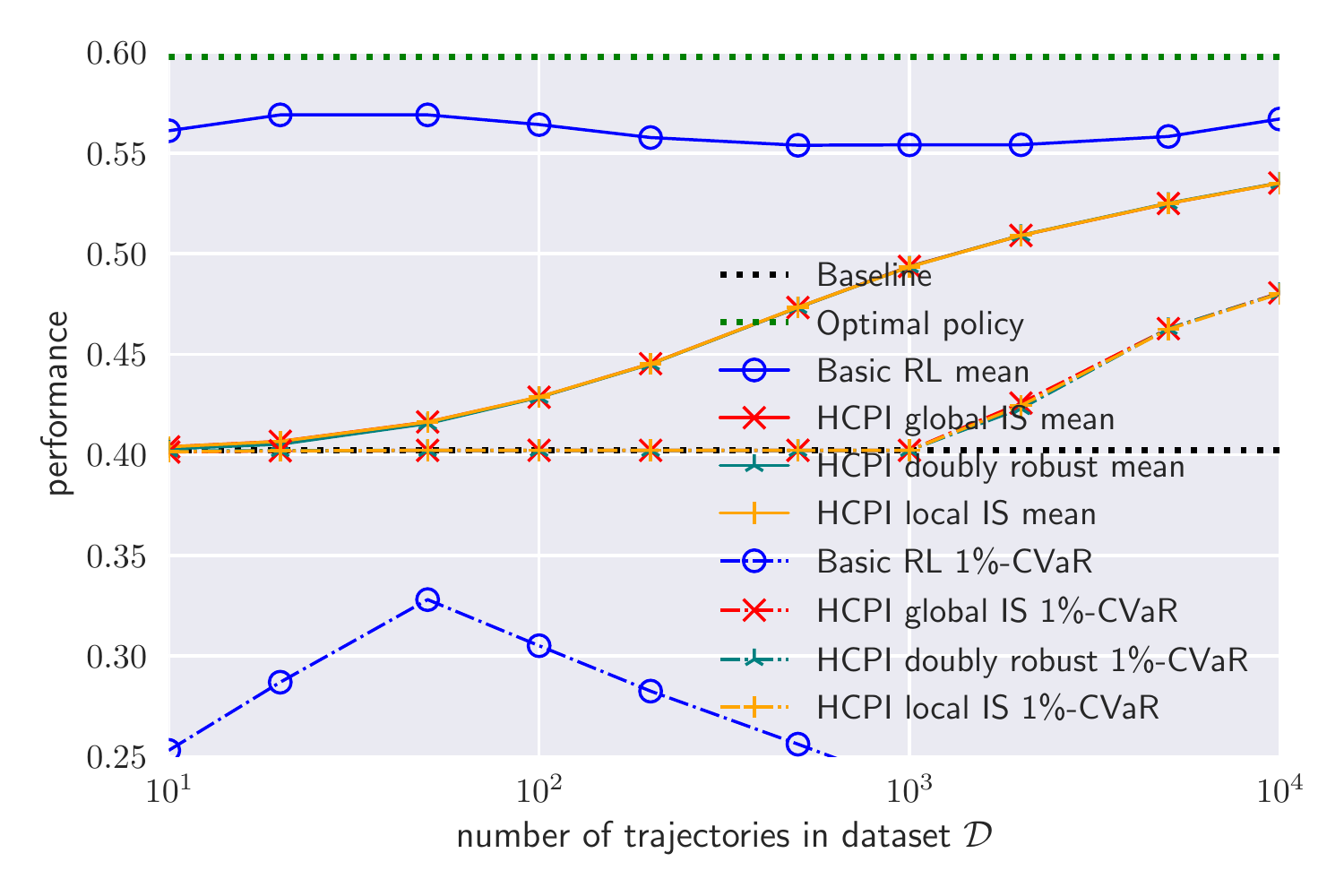}
			\label{fig:HCPI_delta=0.1}
		}
		\subfloat[HCPI with $\delta_{hcpi}=0.9$]{
			\includegraphics[trim = 5pt 5pt 5pt 5pt, clip, width=0.5\textwidth]{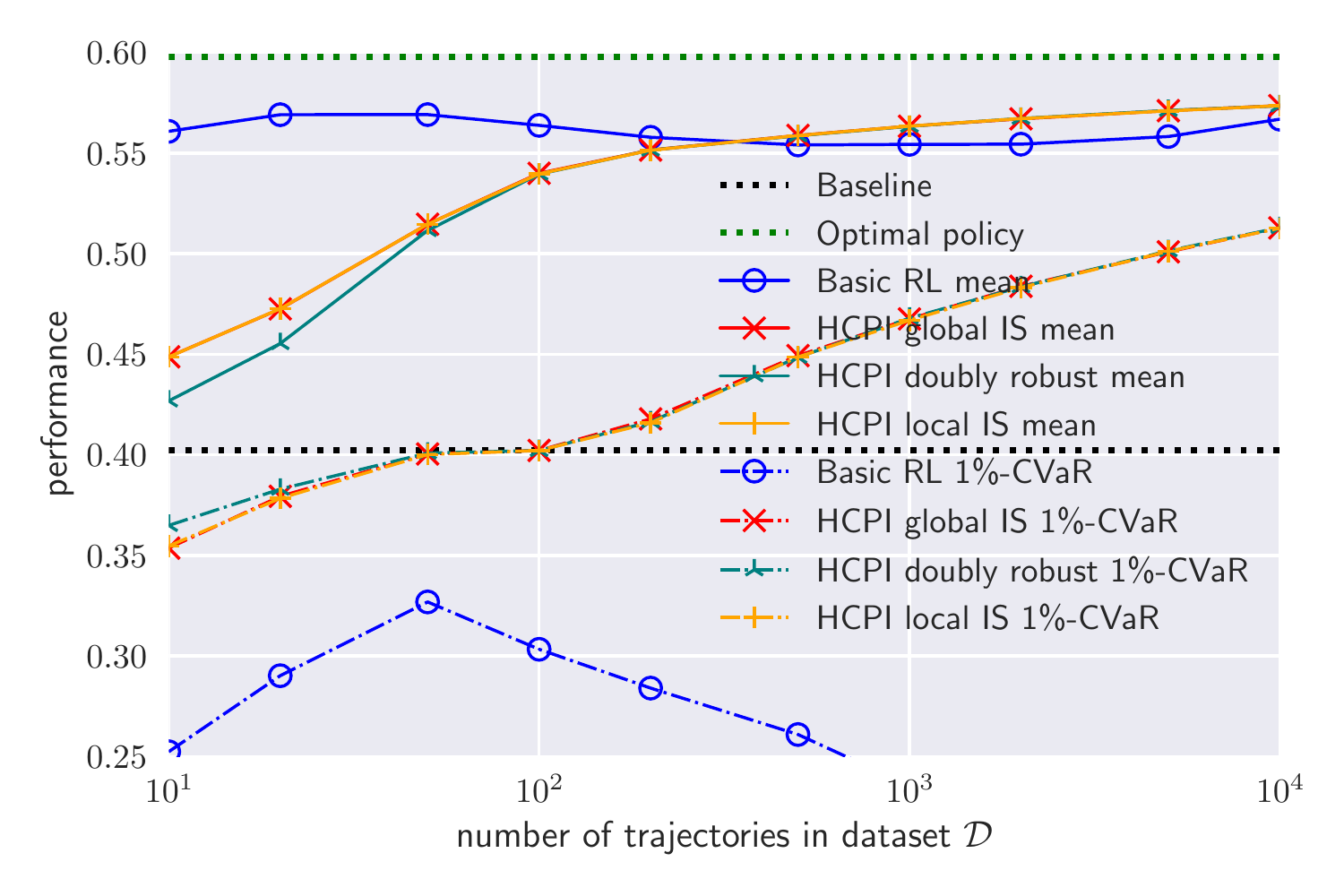}
			\label{fig:HCPI_delta=0.9}
		} \\
		\centering
		\subfloat[Mean performance HCPI doubly-robust heatmap]{
			\includegraphics[trim = 10pt 40pt 45pt 60pt, clip, width=0.5\textwidth]{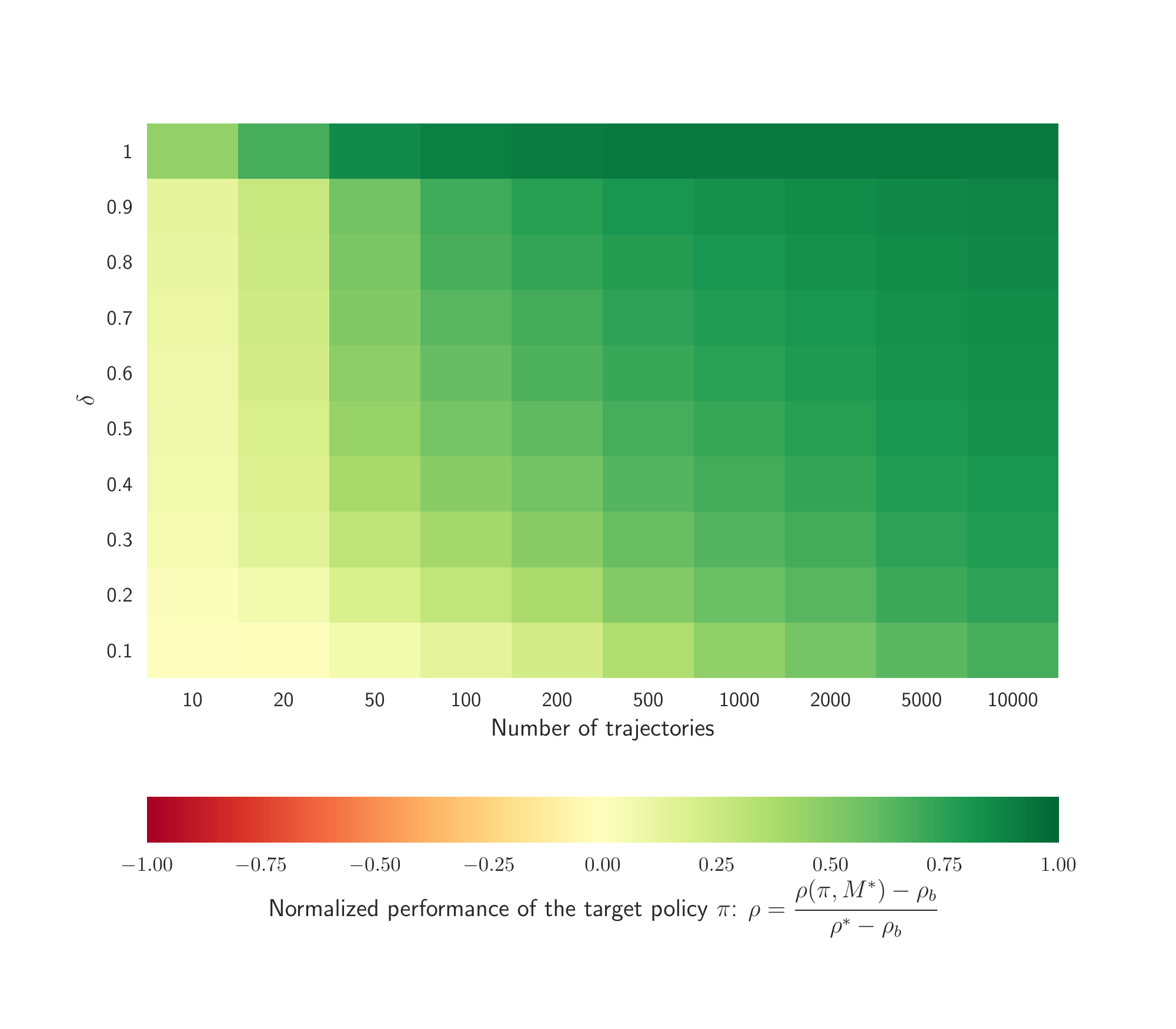}
			\label{fig:HCPI_heatmap_mean}
		}
		\subfloat[1\%-CVaR performance HCPI doubly robust heatmap]{
			\includegraphics[trim = 10pt 40pt 45pt 60pt, clip, width=0.5\textwidth]{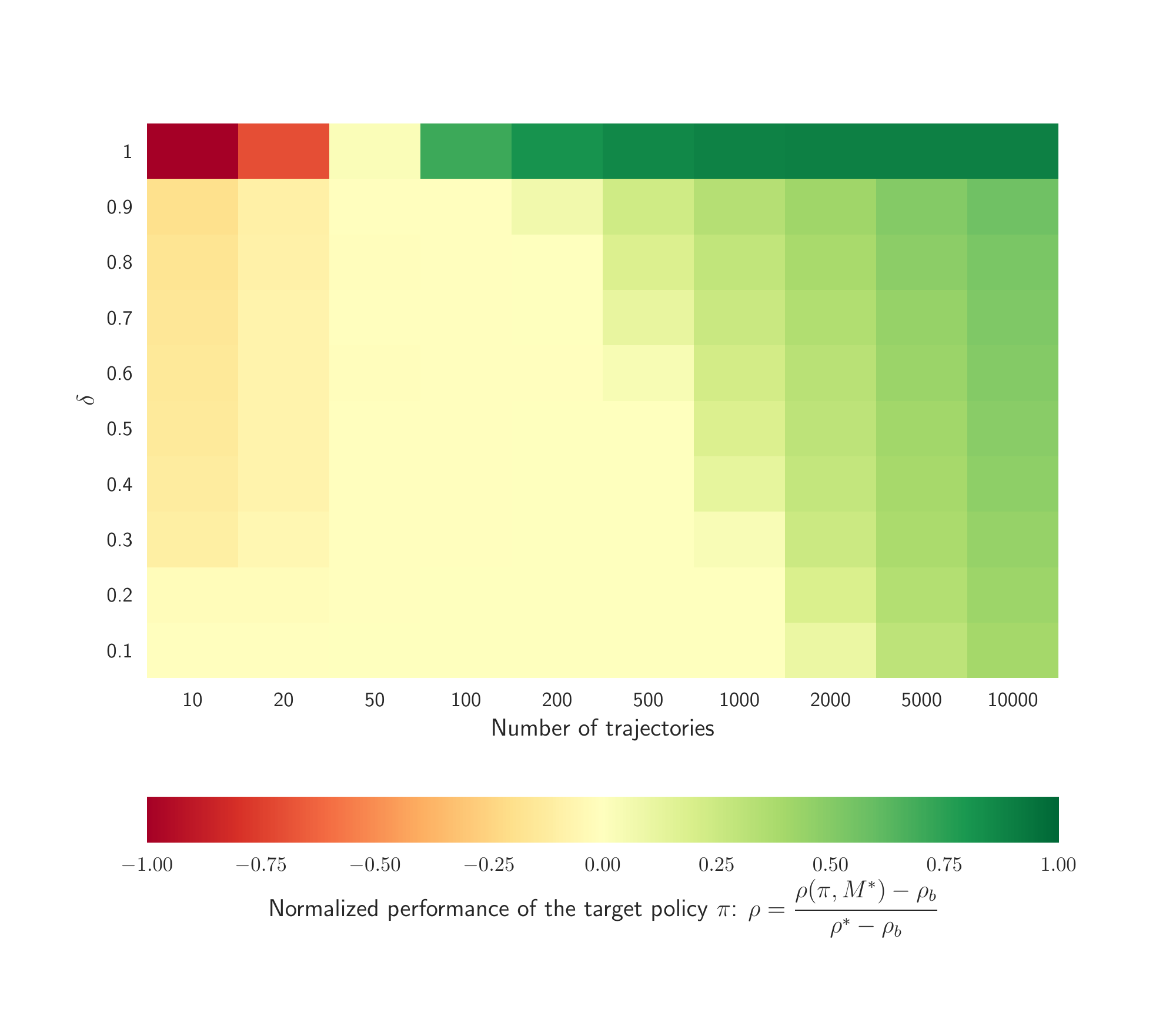}
			\label{fig:HCPI_heatmap_percentile}
		}
		\caption{HCPI hyper-parameter search results on the Gridworld domain.}
		\label{fig:HCPI_maze}
	\end{figure*} 
	
	\subsubsection{Figures}
	We present three types of figure in the paper (main document and appendix).
	
	\paragraph{Performance vs. dataset size:} These figures (for instance Figure \subref*{fig:HCPI_delta=0.1}) show the (mean and/or CVaR) performance of the algorithms as a function of the dataset size.
	
	\paragraph{Hyper-parameter search heatmaps:} These figures (for instance Figure \subref*{fig:HCPI_heatmap_mean}) show the (mean or CVaR) normalized performance of the algorithms as a function of both the dataset size and the hyper-parameter value of the evaluated algorithm. The normalized performance is computed with Equation \ref{eq:normalized_perf} and represented with colour. Red means that the performance is worse than that of the baseline, yellow means that it is equal and green means that it improves the baseline.
	
	\paragraph{Random MDPs heatmaps:} These figures (for instance Figure \subref*{fig:RaMDP_heatmap_percentile_randomMDP_0.003}) are very similar to the other heatmaps except that the normalized performance is shown as a function of both the dataset size and the hyper-parameter $\eta$ used for the baseline generation (instead of the hyper-parameter of the evaluated algorithm).
	
	\subsection{Other benchmark algorithms: competitors}
	\label{sup:benchmarkalgos}
	Since the \emph{baseline} meaning is overridden in this paper, we refer to the non-SPIBB benchmark algorithms with the term \emph{competitors}.
	
	\subsubsection{Basic RL}
	Basic RL is implemented by computing the MLE MDP and solving it with dynamic programming. In order to cover the state-action pairs absent from the dataset, two $Q$ initializations were investigated in our experiments: optimistic ($V_{max}$), and pessimistic ($-V_{max}$). The former yields awful performances in our batch RL setting. This is not surprising because optimism makes it imprudently explore every unknown state-action pairs. All the presented results were therefore obtained with the pessimistic initialization as in \citet{Jiang2015}.

	\subsubsection{HCPI}
	Safe policy improvement in a model-free setting is closely related to High Confidence Off-policy evaluation ~\cite{thomas2015high2}. Instead of relying on the model uncertainty, this class of methods relies on a high-confidence lower bound on the Importance Sampling (IS) estimate of the trained policy performance. Given a dataset $\mathcal{D}$, a part of it, $\mathcal{D}_{train}$, is used to derive a set of candidates policies.  A policy $\pi_t$ is first derived using an off policy reinforcement learning algorithm ($Q$-learning for instance) and is regularized using the baseline to obtain a set of candidate policies $\Pi_{candidates} = \{((1-\alpha)\pi_t+\alpha\pi_b),\alpha\in\{0,0.1,0.2,0.3,...1\}\}$. The remaining data $\mathcal{D}_{test}$ are used to evaluate the candidate policies. The policy with the highest lower bound on the estimated performance is returned. \citet{thomas2015high} introduced three ways of obtaining the lower bound on the estimate. 
	\begin{itemize}
	    \item The first one is an extension of Maurer and Pontil’s empirical Bernstein inequality. Let $X_1,...X_n$ be $n$ independent real-valued random variables, such that for each $i\in\{1,...,n\}$, we have $\mathbb{P}[0\leq X_i]=1$, $\mathbb{E}[X_i]\leq \nu$ and some threshold value $c_i\geq0$. Let $\delta\geq0$ and $Y_i=min\{X_i,x_i\}$. Then with probability at least $1-\delta$, we have:
    \begin{align*}
        \mu\geq \left(\sum_{i=1}^{n}\frac{1}{c_i}\right)^{-1}\sum_{i=1}^{n}\frac{Y_i}{c_i}-\left(\sum_{i=1}^{n}\frac{1}{c_i}\right)^{-1}\frac{7nln(\frac{2}{\delta})}{3(n-1)}-\left(\sum_{i=1}^{n}\frac{1}{c_i}\right)^{-1}\sqrt{\frac{ln(\frac{2}{\delta})}{n-1}\sum_{i,j=1}^{n}\left(\frac{Y_i}{c_i}-\frac{Y_j}{c_j}\right)^2}
    \end{align*}
    In the SPI setting, $X_i$ is the unbiased estimate of the return related to each trajectory. The drawback of this method is the hyper-parameter $c_i$ which needs to be tuned. 
    \item The second method is based on the assumption that the mean return is normally distributed. Relying on this assumption, a less conservative lower bound can be obtained using Student’s t-test (with the same notations):
    \begin{align*}
        \mathbb{E}[X_i]\geq \frac{1}{n}\sum_{i=1}^n X_i-\frac{\sigma}{\sqrt{n}}t_{1-\delta,n-1}
    \end{align*}
     with $\sigma=\sqrt{\frac{1}{n-1}\sum_{i=1}^n\left(X_i-\hat{X_i})^2\right)}$ the sample standard deviation of $X_1,..,X_n$ with Bessel's correction.
     \item The last one is based on Efron’s
Bootstrap methods \cite{efron1987better}. It relies on bootstrapping to estimate the true distribution of the mean return instead of considering it as normally distributed. 
	\end{itemize}
	In practice, the first method is too conservative and the third one is not computationally efficient. Therefore we limit our study to the second one, which relies on Student's t-test.\\
	We implemented three versions of HCPI: with global importance sampling, with local importance sampling, and with the doubly robust method. As Figures \subref*{fig:HCPI_delta=0.1} and \subref*{fig:HCPI_delta=0.9} reveal, they all behave more or less the same on the Gridworld domain. We also searched for the best hyper-parameter $\delta_{hcpi}\in \left\{0.1,0.2,0.3,0.4,0.5,0.6,0.7,0.8,0.9,1\right\}$ value. Figures \subref*{fig:HCPI_heatmap_mean} and \subref*{fig:HCPI_heatmap_percentile} respectively display the mean and 1\%-CVaR performances. One can observe that $\delta_{hcpi}=1$ yields the best result in mean, but turns out to be strongly unsafe for small datasets. $\delta_{hcpi} = 0.9$ appears to offer the best compromise and this is the value we retain for the experiments reported in the main document. Note that those $\delta_{hcpi}$ values mean that the confidence is very small: 0.1 for $\delta_{hcpi} = 0.9$, and even null for $\delta_{hcpi} = 1$.
	
	\begin{figure*}[t]
		\centering
		\subfloat[Robust MDP with $\delta_{rob}=0.1$ (no safety test)]{
			\includegraphics[trim = 5pt 5pt 5pt 5pt, clip, width=0.5\textwidth]{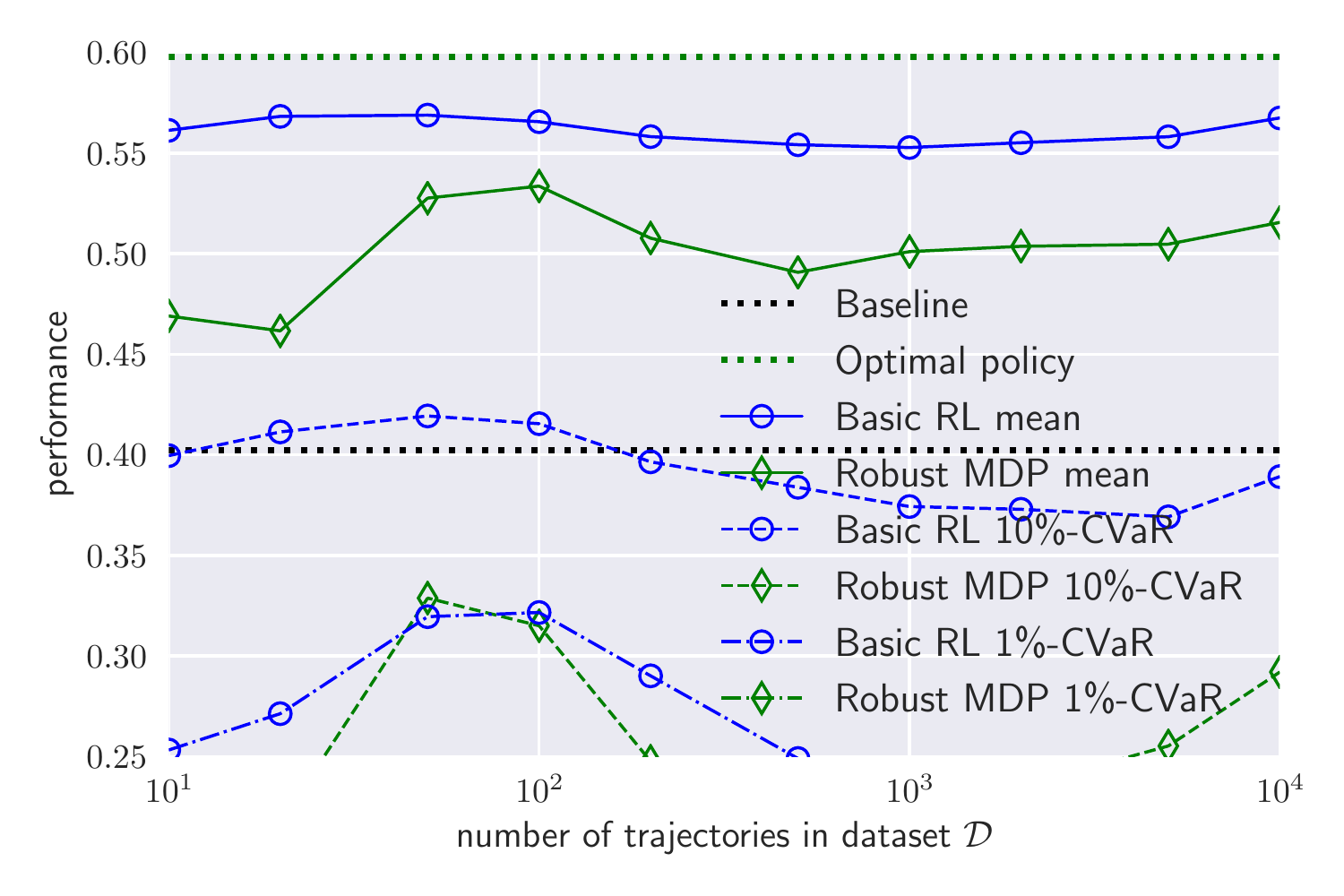}
			\label{fig:RobustMDP_delta=0.1}
		}
		\subfloat[Mean performance Robust MDP heatmap (no safety test)]{
			\includegraphics[trim = 10pt 40pt 45pt 60pt, clip, width=0.5\textwidth]{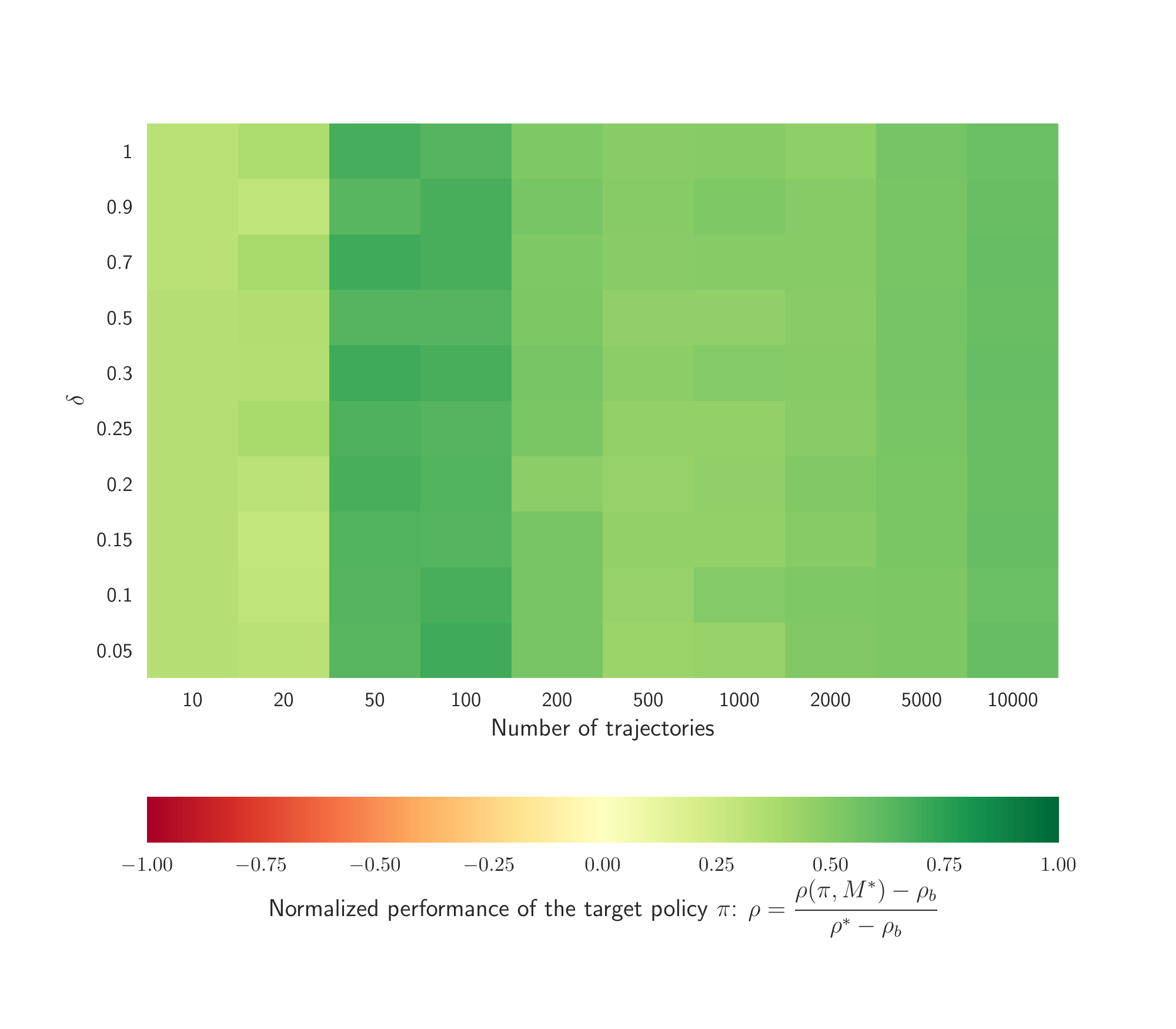}
			\label{fig:RobustMDP_heatmap_mean}
		}
		\caption{Robust MDP hyper-parameter search results on the Gridworld domain.}
		\label{fig:RobustMDP_maze}
	\end{figure*}
	
	\begin{figure*}
		\centering
		\subfloat[RaMDP with $\kappa_{adj}=0.002$ (Gridworld)]{
			\includegraphics[trim = 5pt 5pt 5pt 5pt, clip, width=0.5\textwidth]{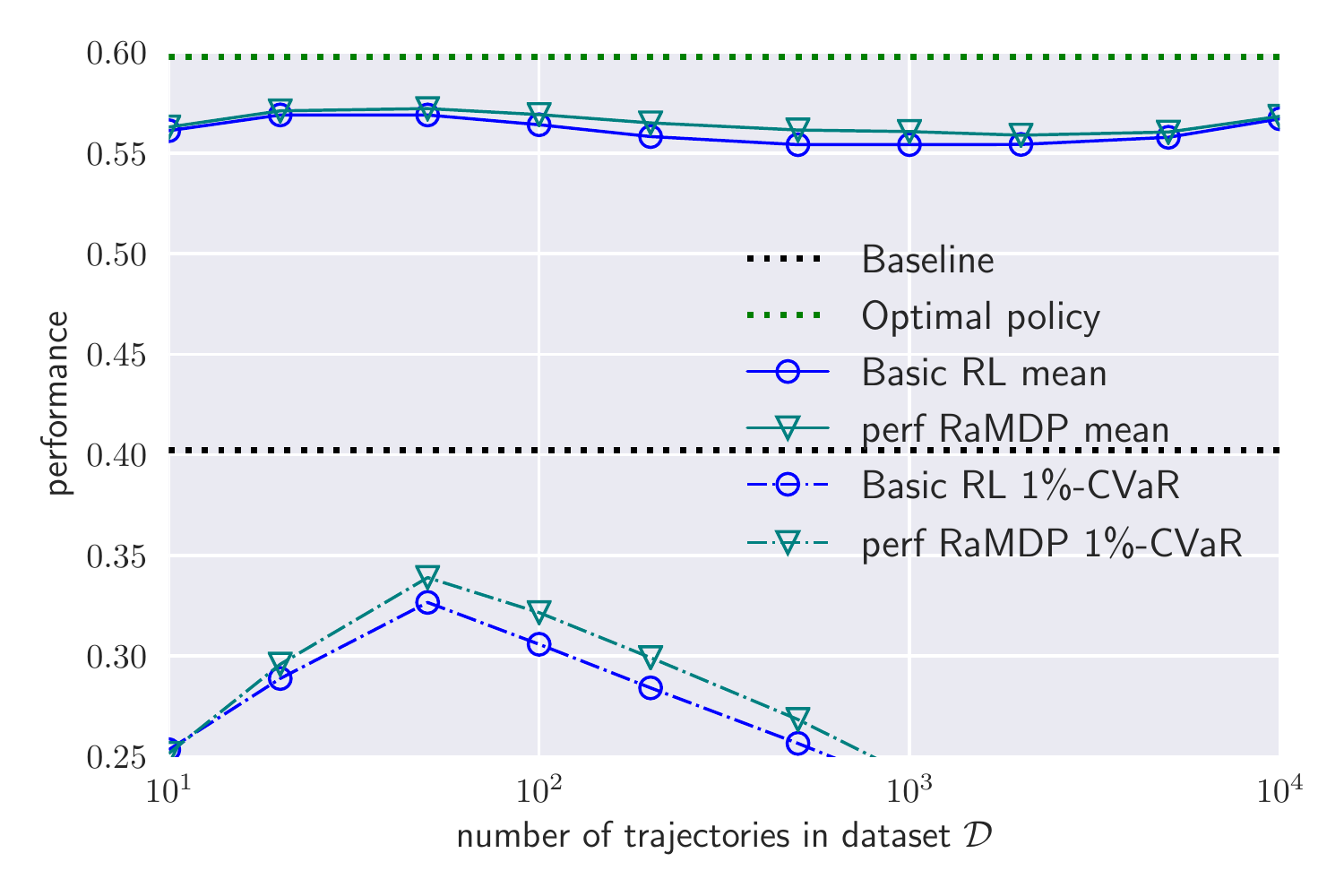}
			\label{fig:RaMDP_kappa=0.002}
		}
		\subfloat[RaMDP with $\kappa_{adj}=0.003$ (Gridworld)]{
			\includegraphics[trim = 5pt 5pt 5pt 5pt, clip, width=0.5\textwidth]{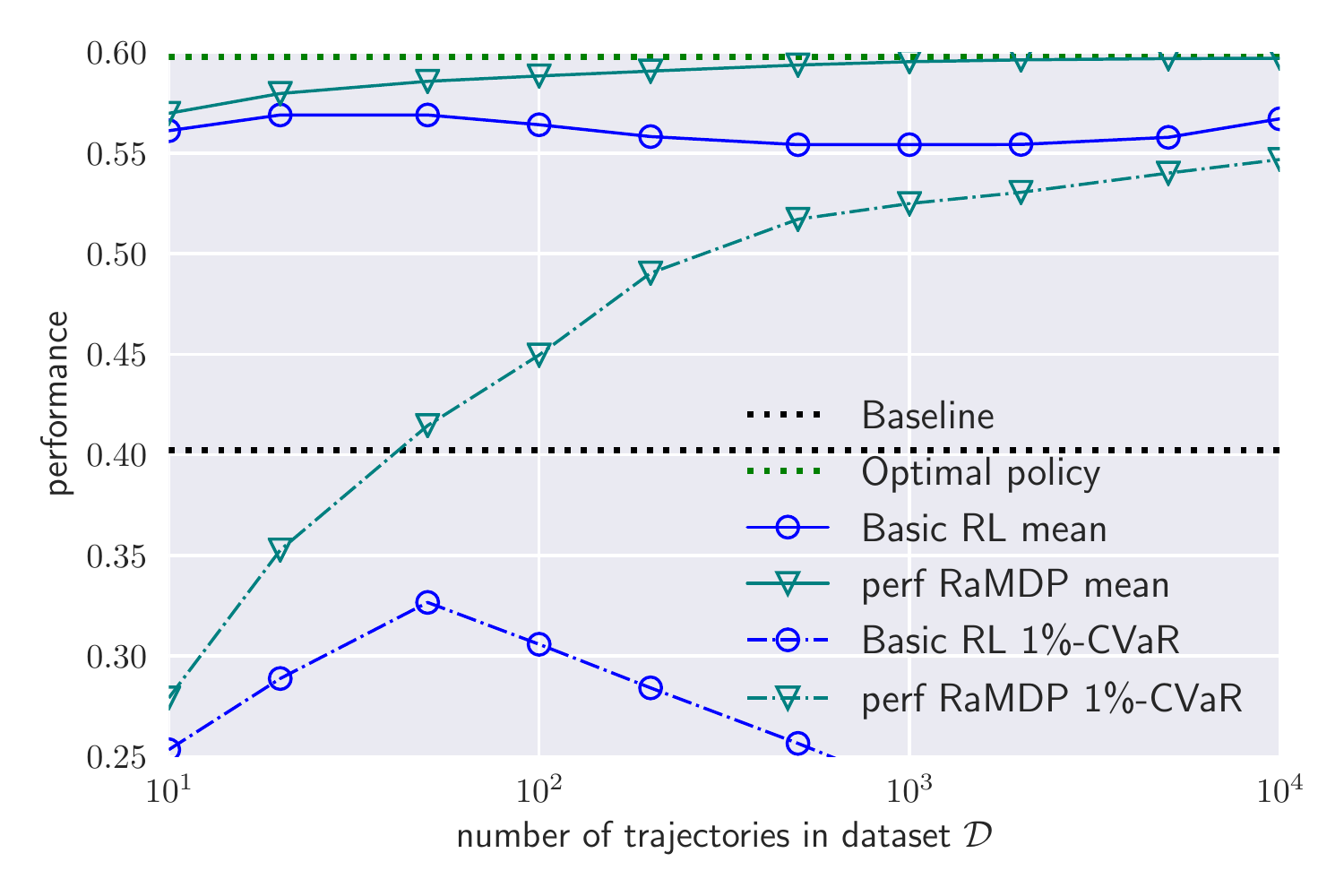}
			\label{fig:RaMDP_kappa=0.003}
		} \\
		\centering
		\subfloat[Mean performance RaMDP heatmap (Gridworld)]{
			\includegraphics[trim = 10pt 40pt 45pt 60pt, clip, width=0.5\textwidth]{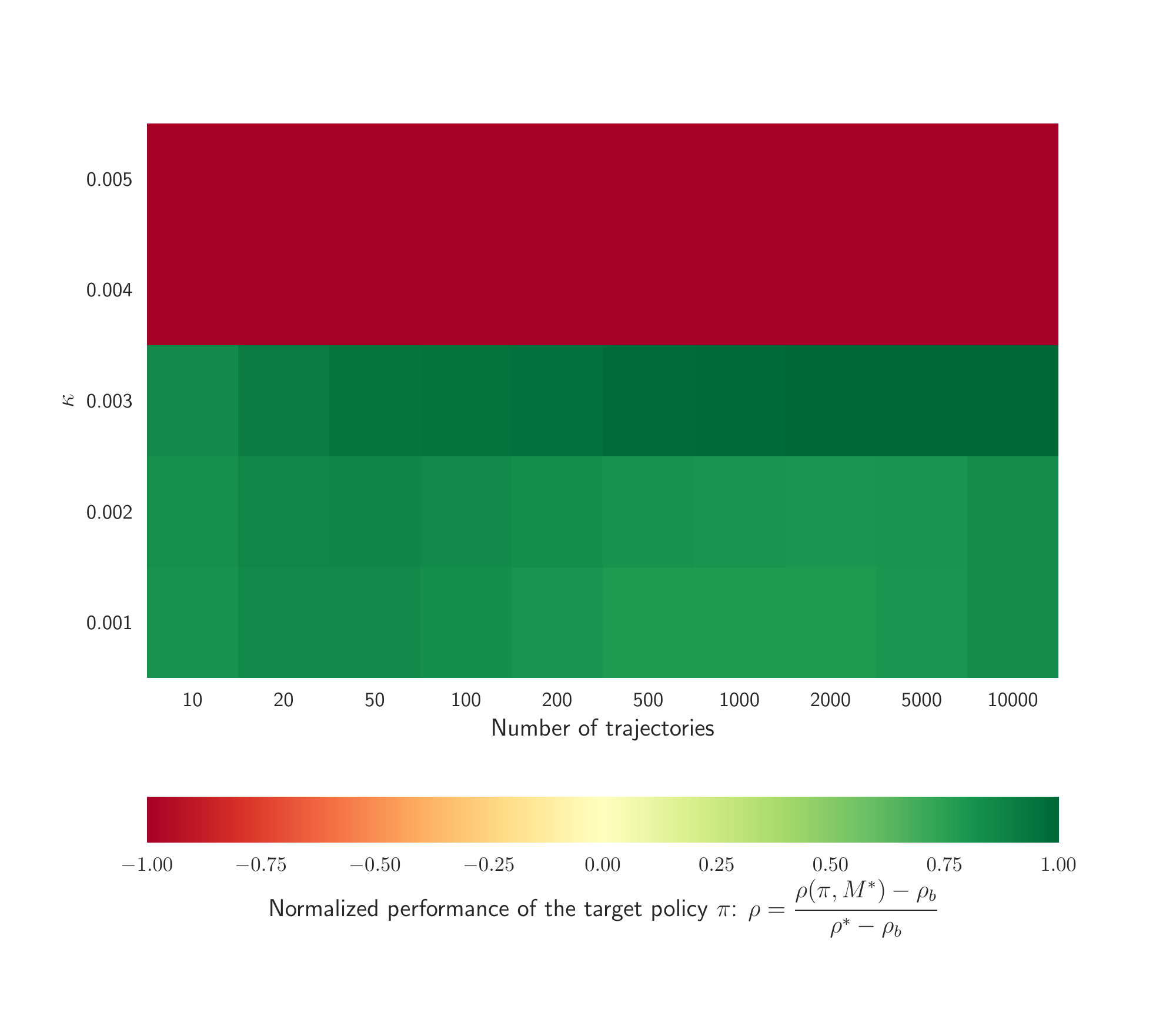}
			\label{fig:RaMDP_heatmap_mean}
		}
		\subfloat[1\%-CVaR performance RaMDP heatmap (Gridworld)]{
			\includegraphics[trim = 10pt 40pt 45pt 60pt, clip, width=0.5\textwidth]{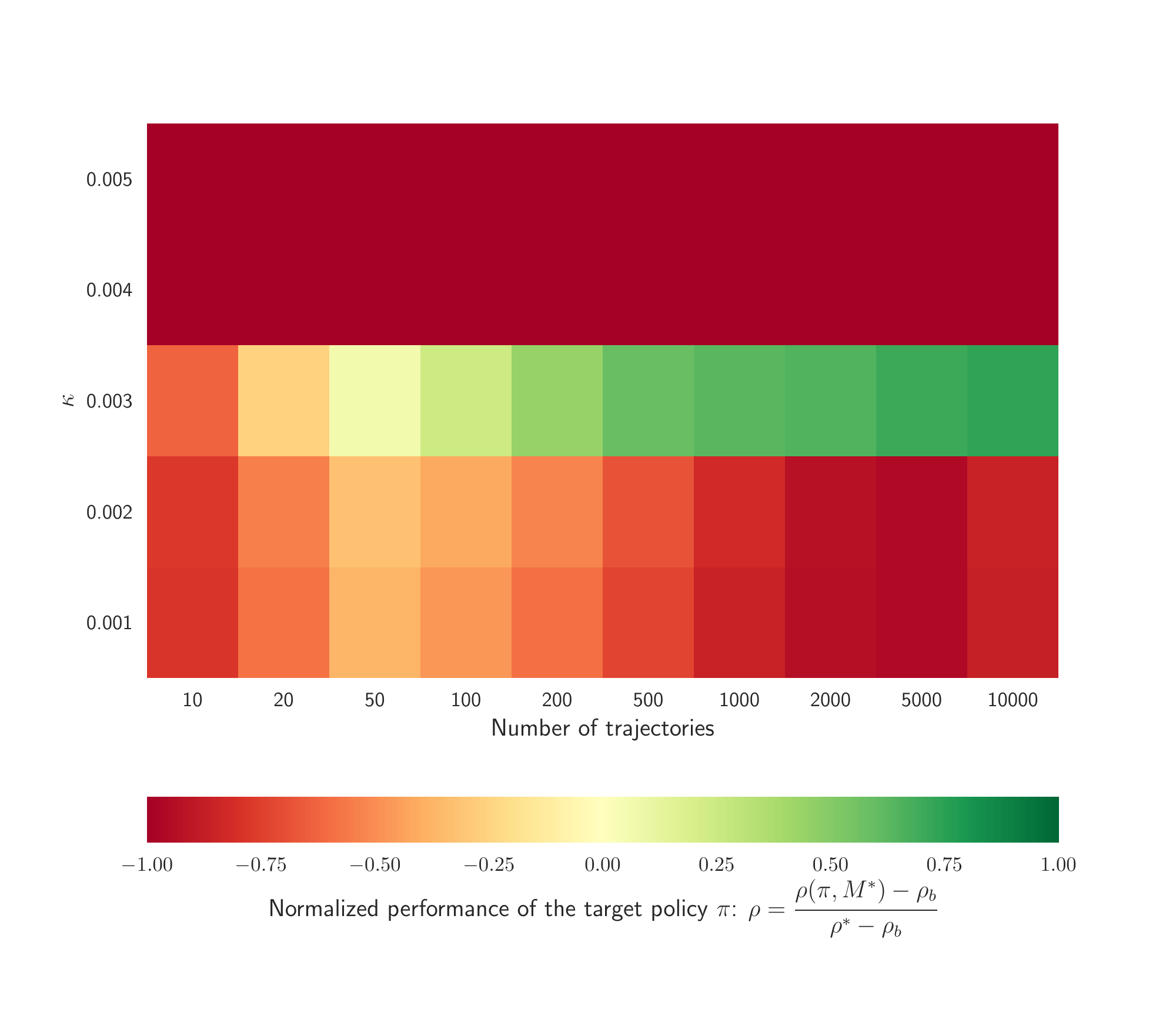}
			\label{fig:RaMDP_heatmap_percentile}
		}\\
		\centering
		\subfloat[1\%-CVaR RaMDP heatmap with $\kappa_{adj}=0.002$ (Random MDPs)]{
			\includegraphics[trim = 10pt 40pt 45pt 60pt, clip, width=0.5\textwidth]{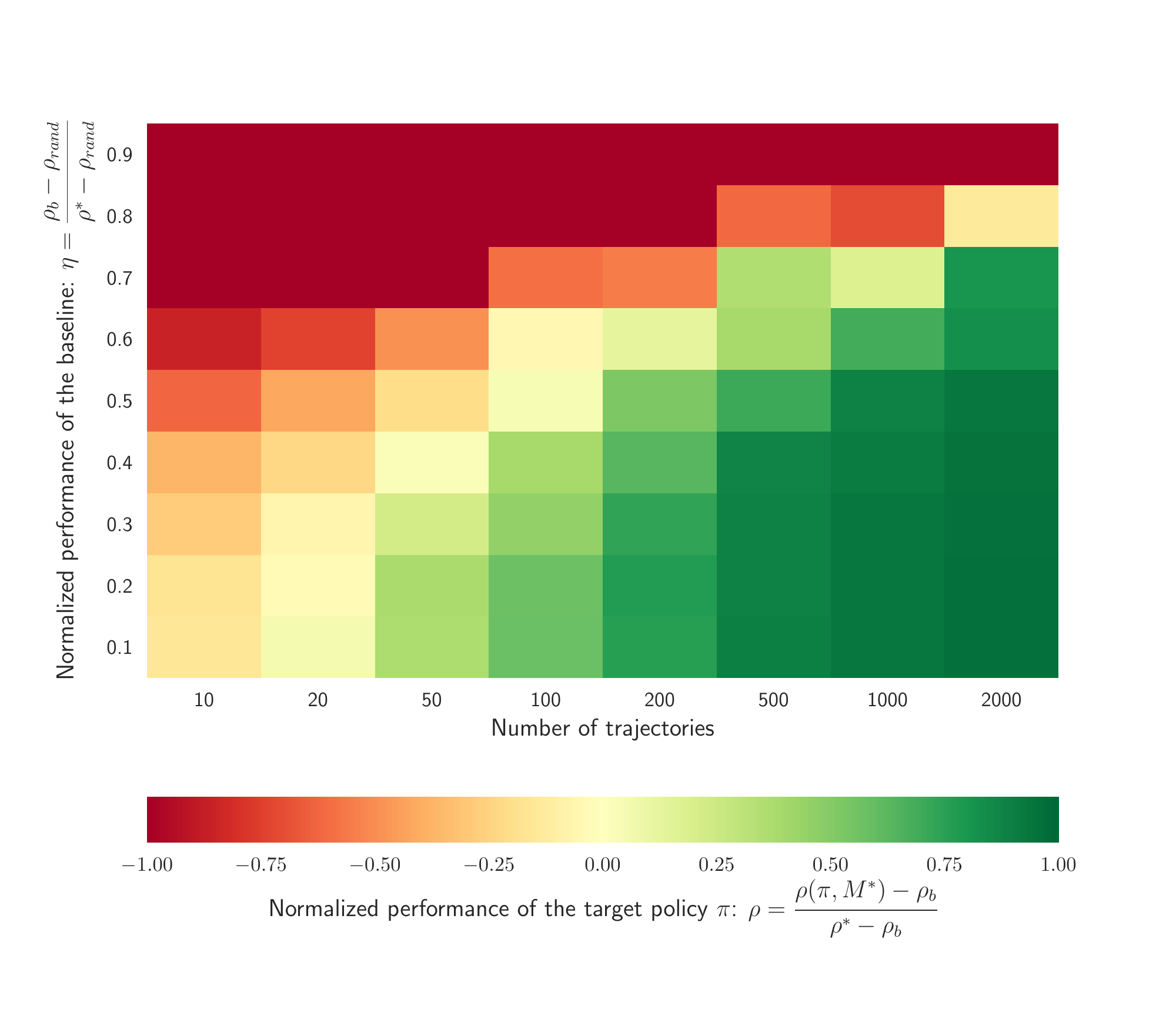}
			\label{fig:RaMDP_heatmap_percentile_randomMDP_0.002}
		}
		\subfloat[1\%-CVaR RaMDP heatmap with $\kappa_{adj}=0.003$ (Random MDPs)]{
			\includegraphics[trim = 10pt 40pt 45pt 60pt, clip, width=0.5\textwidth]{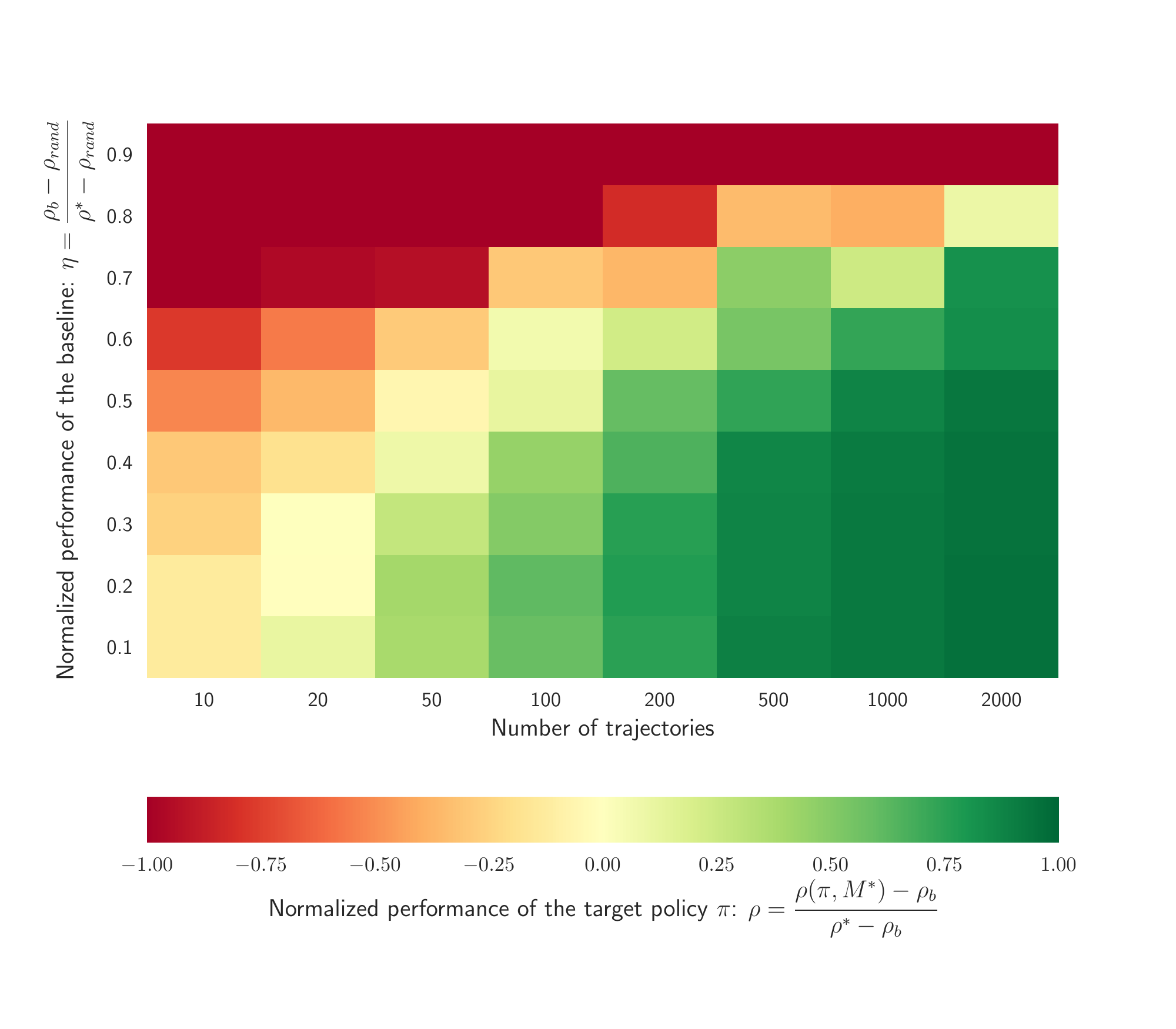}
			\label{fig:RaMDP_heatmap_percentile_randomMDP_0.003}
		}
		\caption{RaMDP hyper-parameter search results on the Gridworld and Random MDPs domains.}
		\label{fig:RaMDP}
	\end{figure*} 
	
	\subsubsection{Robust MDP}
	Robust MDP also relies on a confidence hyperparameter $\delta_{rob}$. We observe that the behaviour of the algorithm is not much dependent on $\delta_{rob} \in \left\{0.05,0.1,0.15,0.2,0.25,0.3,0.5,0.7,0.9,1\right\}$, and that it always fall back on the baseline when the dataset is under 50,000 trajectories. We observe also on Figure \subref*{fig:RobustMDP_delta=0.1} that, for the smaller datasets we do our benchmark on, independently from the safety set, the policy trained with the Robust MDP algorithm, which is the best policy in the worst-case MDP, is worse that the policy trained with Basic RL on mean and also on CVaR. 1\%-CVaR even falls down out of the figure. We interpret this as the fact that, in the Gridworld domain, there is a zone where all the states have been experienced a reasonable amount of time, and where the algorithm infers that the outcome is well known: 0 reward. Near the goal, on the contrary, there are some states where there is a risk to go because of the stochastic transitions that are largely unknown. This behaviour seem to reproduce also frequently on the Random MDPs domain. Figure \subref*{fig:RobustMDP_heatmap_mean} displays the mean performance for a large set of $\delta_{rob}$ values without the safety test. The figures of Robust MDP with the safety test are omitted because it always fails and therefore the algorithm always outputs the baseline. On the main document figures, we report the Robust MDP performance without safety test for $\delta_{rob}=0.1$.

	\subsubsection{Reward-adjusted MDP}
	The theory developed in \citet{ghavamzadeh2016safe} states that the reward should be adjusted as follows:
	\begin{align}
	    \widetilde{R}(x,a) \leftarrow R^*(x,a) - \cfrac{\gamma R_{max}}{1-\gamma} e(x,a),
	\end{align}
	where $R^*(x,a)$ is the true reward function, that they assume to be known, and $e(x,a)$ is the error function on the dynamics, with bounded with concentration bounds as in our Proposition \ref{prop:eps_pib}:
	\begin{align}
	    e(x,a) \leq \sqrt{\cfrac{2}{N_\mathcal{D}(x,a)}\log\cfrac{2|\mathcal{X}||\mathcal{A}|2^{|\mathcal{X}|}}{\delta_{adj}}}
	\end{align}
	
	Also, we do not assume that $R^*(x,a)$ is known in our experiments and there is consequently a $\gamma$ factor disappearing. We obtain:
	\begin{align}
	    \widetilde{R}(x,a) &\leftarrow \widehat{R}(x,a) - \cfrac{ R_{max}}{1-\gamma} \sqrt{\cfrac{2}{N_\mathcal{D}(x,a)}\log\cfrac{2|\mathcal{X}||\mathcal{A}|2^{|\mathcal{X}|}}{\delta_{adj}}} \\
	    &\leftarrow \widehat{R}(x,a) - \cfrac{100}{\sqrt{N_\mathcal{D}(x,a)}},
	\end{align}
	with our domain parameters and the choice of $\delta_{adj}=0.1$. Instead, we consider the following hyper-parametrization:
	\begin{align}
	    \widetilde{R}(x,a) \leftarrow \widehat{R}(x,a) - \cfrac{\kappa_{adj}}{\sqrt{N_\mathcal{D}(x,a)}}.
	\end{align}
	
	 We perform a hyper-parameter seach for:
	 $$\kappa_{adj} \in \left\{0.0001,0.0002,0.0005,0.001,0.002,0.003,0.004,0.005,0.01,0.02,0.05,0.1,0.2,0.5,1,2,5,10,20,50,100\right\}.$$
	 
	 Figures \subref*{fig:RaMDP_heatmap_mean} and \subref*{fig:RaMDP_heatmap_percentile} respectively show the mean and 1\%-CVaR performance of RaMDP for $\kappa_{adj} \in \left\{0.001,0.002,0.003,0.004,0.005\right\}$. They reveal that for $\kappa_{adj}\geq 0.004$, RaMDP is overly frightened to go near the goal in the same way as with Robust MDP; and that for $\kappa_{adj}\leq 0.002$, RaMDP just ignores the penalty and yields results very close to the Basic RL's (see Figure \subref*{fig:RaMDP_kappa=0.002}). In the middle, there is a tight spot ($\kappa_{adj}=0.003$) where it works quite well on the Gridworld domain as may be seen on Figure \subref*{fig:RaMDP_kappa=0.003}, even though it is not safe for very small datasets. It has to be noted also that, in theory, RaMDP uses a safety test, which fails everytime similarly to that of Robust MDP. In addition to the sensitivity to the $\kappa_{adj}$ parameter, on the Random MDPs benchmark, the unsafety of RaMDP is much more obvious (see Figures \subref*{fig:RaMDP_heatmap_percentile_randomMDP_0.002} and \subref*{fig:RaMDP_heatmap_percentile_randomMDP_0.003}), which tends us to think that the Gridworld domain is favorable to RaMDP. On the main document figures, we report the RaMDP performance without safety test for $\kappa_{adj}=0.003$.

	\newpage
	\section{Extensive Empirical Results on Finite MDPs}	
	\label{sup:figs}
	\subsection{Gridworld additional results}
	\label{sup:maze_full_results}
	\begin{figure*}[ht!]
		\vspace{-10pt}
		\centering
		\subfloat[10\%-CVaR: benchmark with $N_\wedge=5$.]{
			\includegraphics[trim = 5pt 5pt 5pt 5pt, clip, width=0.33\textwidth]{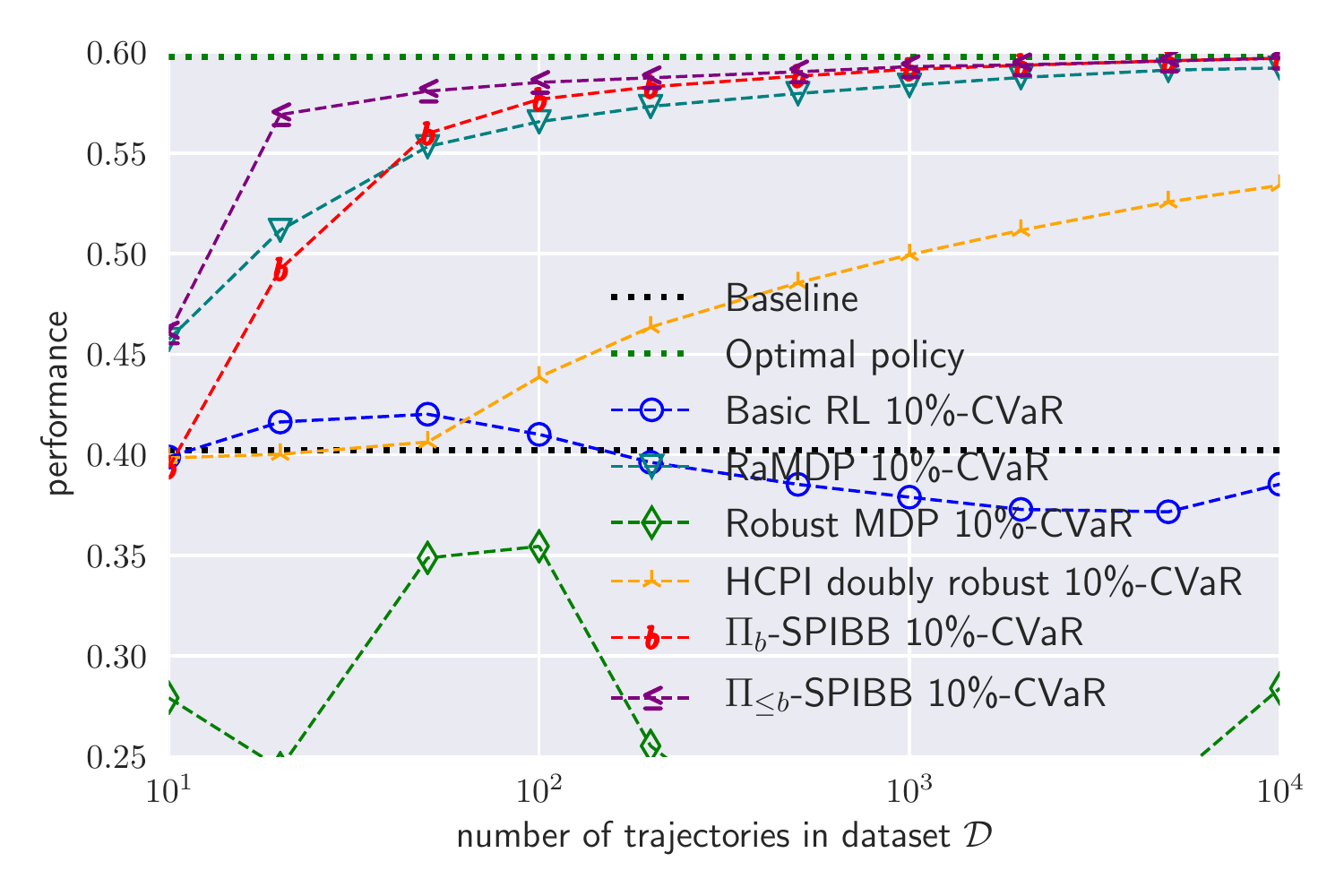}
			\label{fig:heat_mean_pib_5}
		}
		\subfloat[0.1\%-CVaR: benchmark with $N_\wedge=5$.]{
			\includegraphics[trim = 5pt 5pt 5pt 5pt, clip, width=0.33\textwidth]{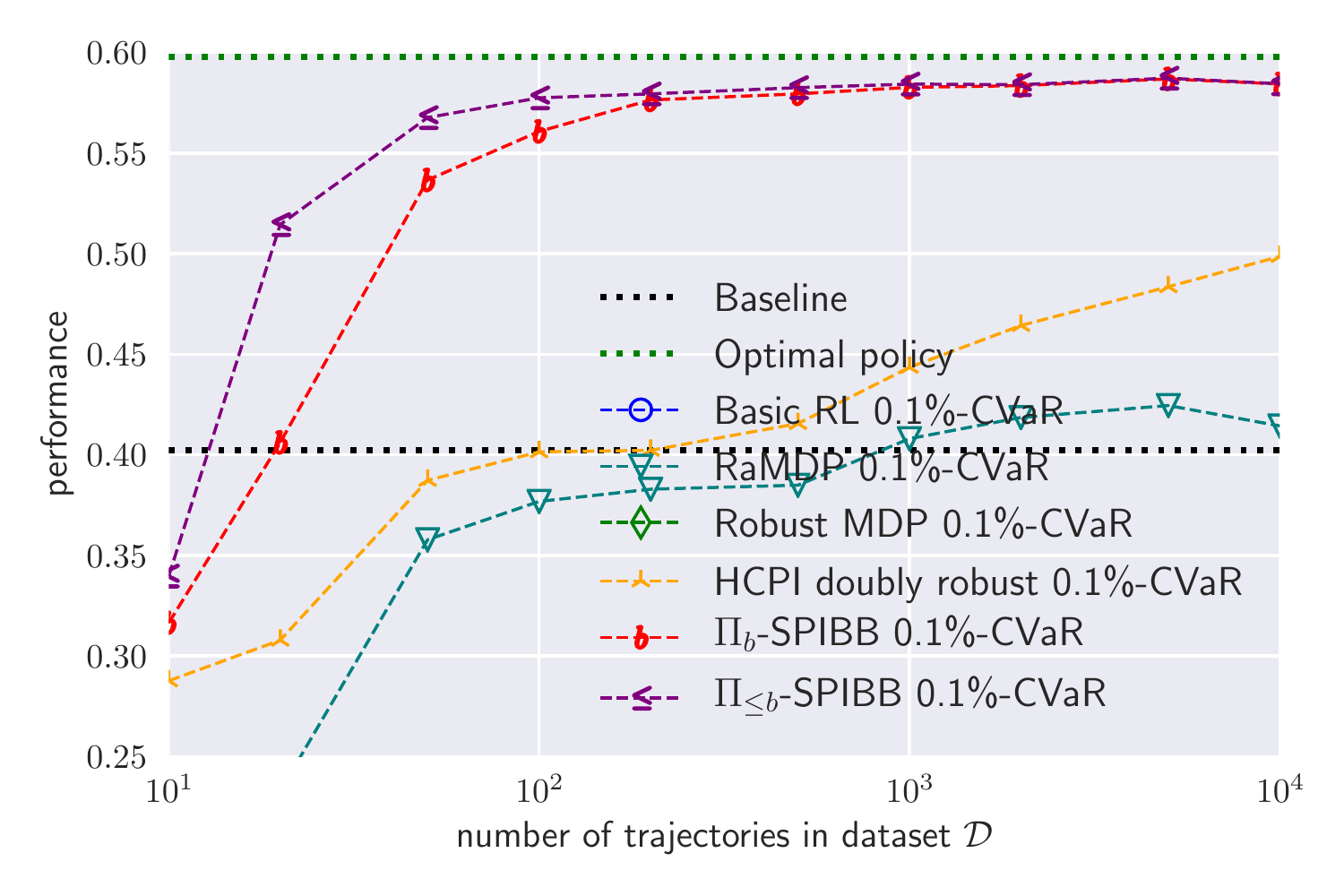}
			\label{fig:heat_mean_pileqb_5}
		}
		\subfloat[Mean \& 1\%-CVaR: SPIBB w. $N_\wedge=5$.]{
			\includegraphics[trim = 5pt 5pt 5pt 5pt, clip, width=0.33\textwidth]{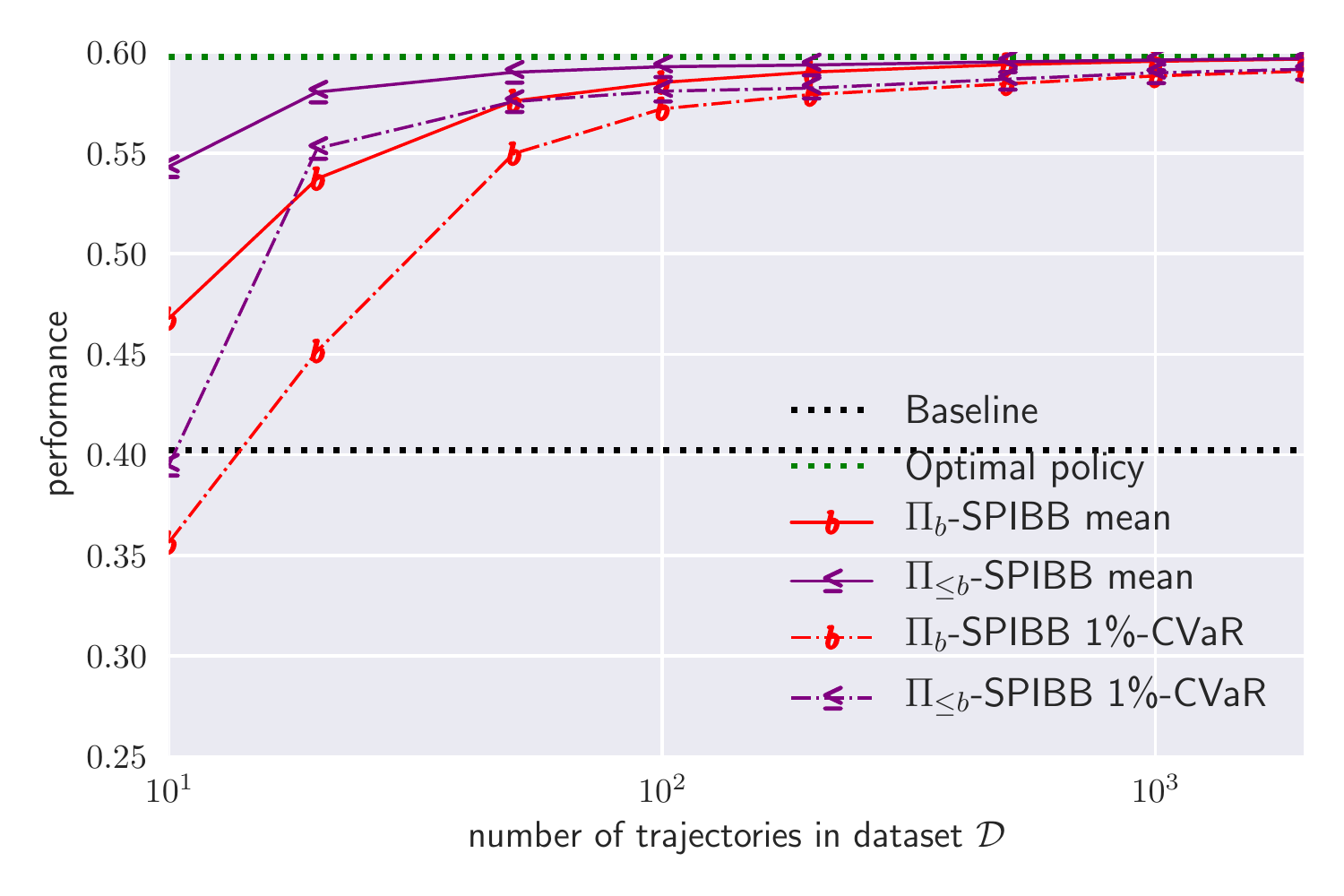}
			\label{fig:heat_mean_basic}
		} \\
		\centering
		\subfloat[Mean \& 1\%-CVaR: SPIBB w. $N_\wedge=10$.]{
			\includegraphics[trim = 5pt 5pt 5pt 5pt, clip, width=0.33\textwidth]{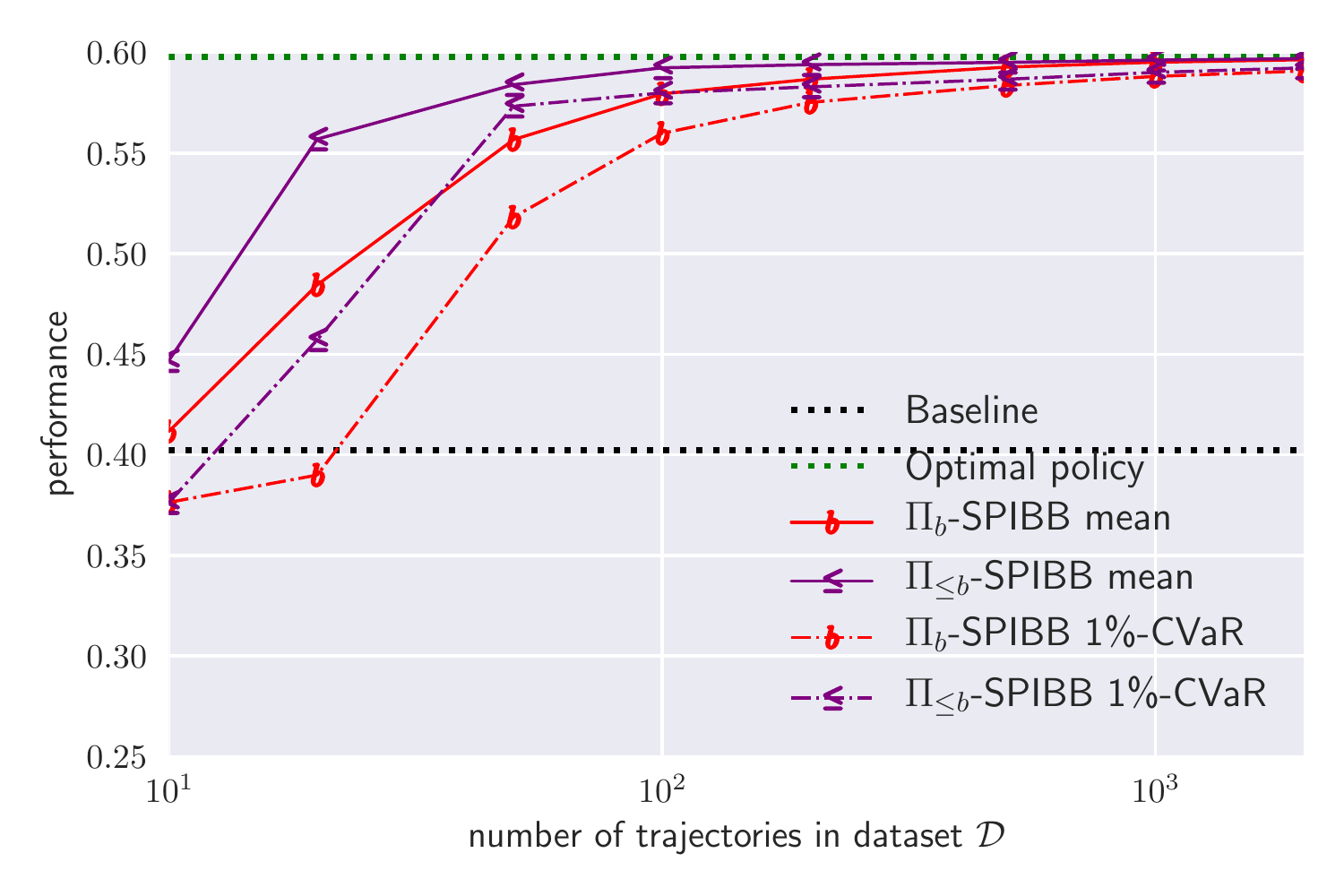}
			\label{fig:heat_mean_pib_10}
		}
		\subfloat[Mean \& 1\%-CVaR: SPIBB w. $N_\wedge=50$.]{
			\includegraphics[trim = 5pt 5pt 5pt 5pt, clip, width=0.33\textwidth]{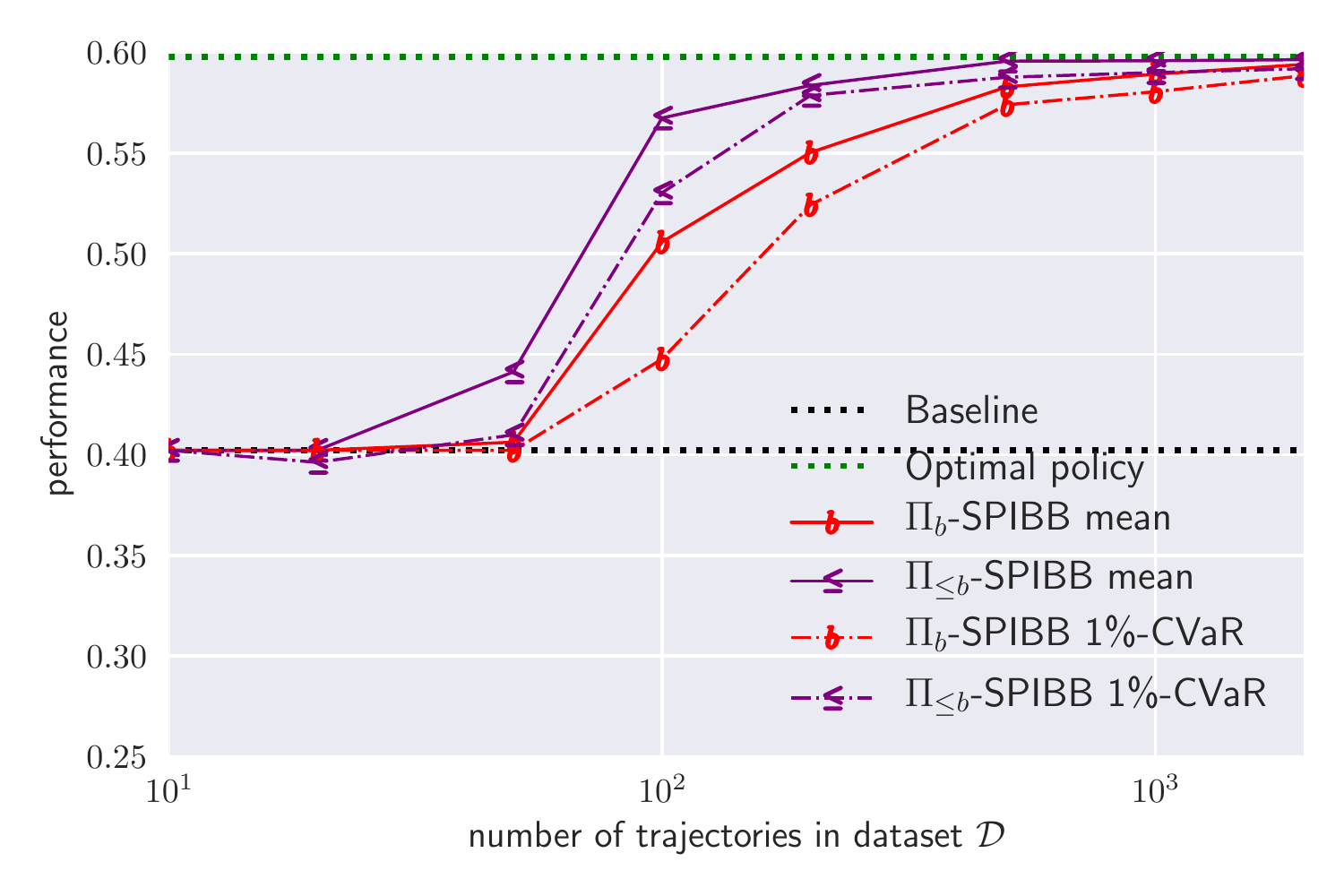}
			\label{fig:heat_mean_pib_20}
		}
		\subfloat[Mean \& 1\%-CVaR: SPIBB w. $N_\wedge=100$.]{
			\includegraphics[trim = 5pt 5pt 5pt 5pt, clip, width=0.33\textwidth]{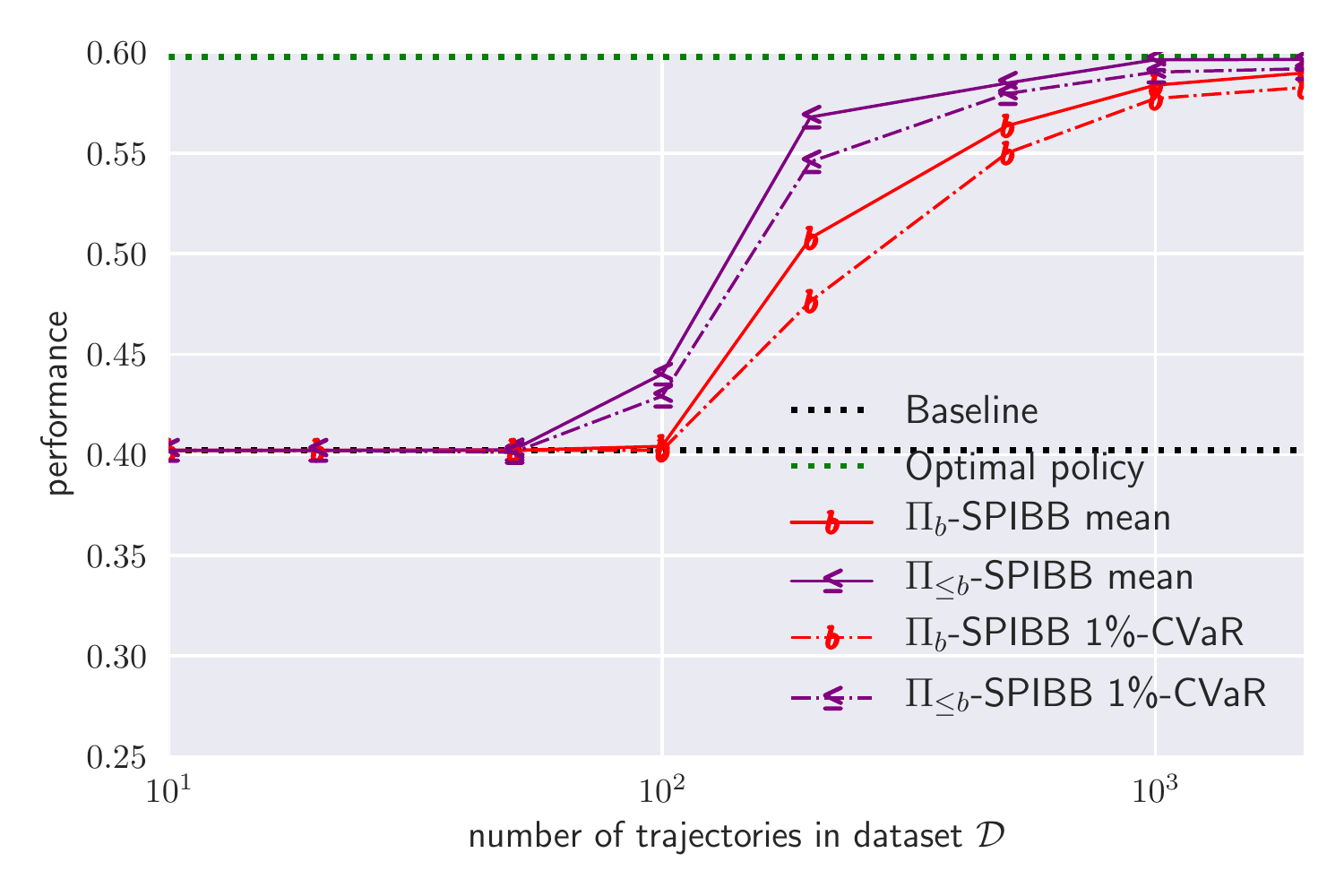}
			\label{fig:heat_mean_pileqb_20}
		} 
		\caption{Gridworld experiment: Figures (a-b) respectively show the benchmark for the 10\%-CVaR and 0.1\%-CVaR performances. Figures (c-f) display additional curves for other $N_\wedge$ values: respectively 5, 10, 50, 100. }
		\label{fig:maze_sup}
		\vspace{-10pt}
	\end{figure*}
	
	\subsection{Gridworld full results with random behavioural policy}
	\label{sup:random_baseline_full_results}
	\begin{figure*}[ht!]
		\vspace{-10pt}
		\centering
		\subfloat[Mean: benchmark with $N_\wedge=20$.]{
			\includegraphics[trim = 5pt 5pt 5pt 5pt, clip, width=0.33\textwidth]{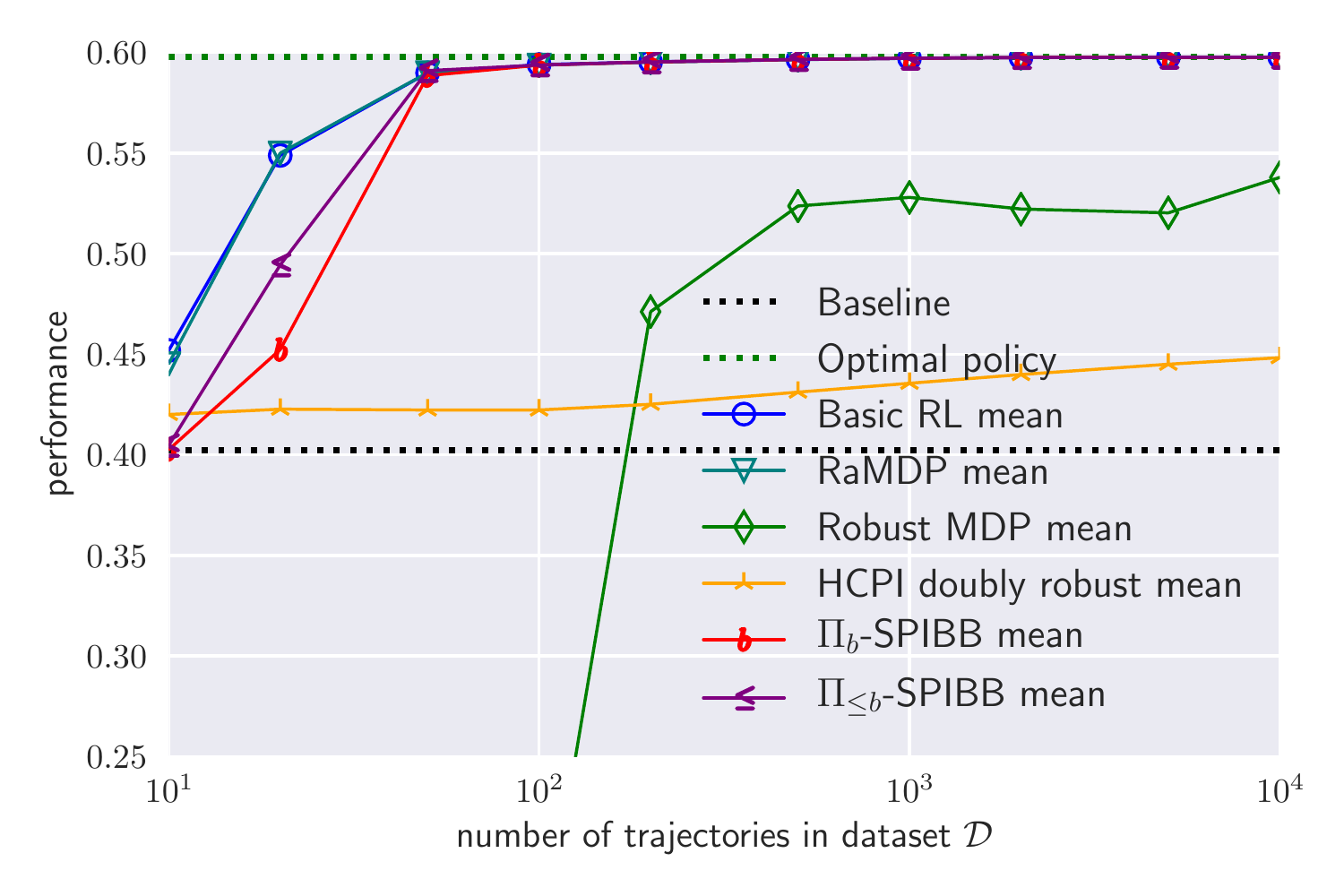}
			\label{fig:maze_rand_baseline_mean_benchmark}
		}
		\subfloat[10\%-CVaR: benchmark with $N_\wedge=20$.]{
			\includegraphics[trim = 5pt 5pt 5pt 5pt, clip, width=0.33\textwidth]{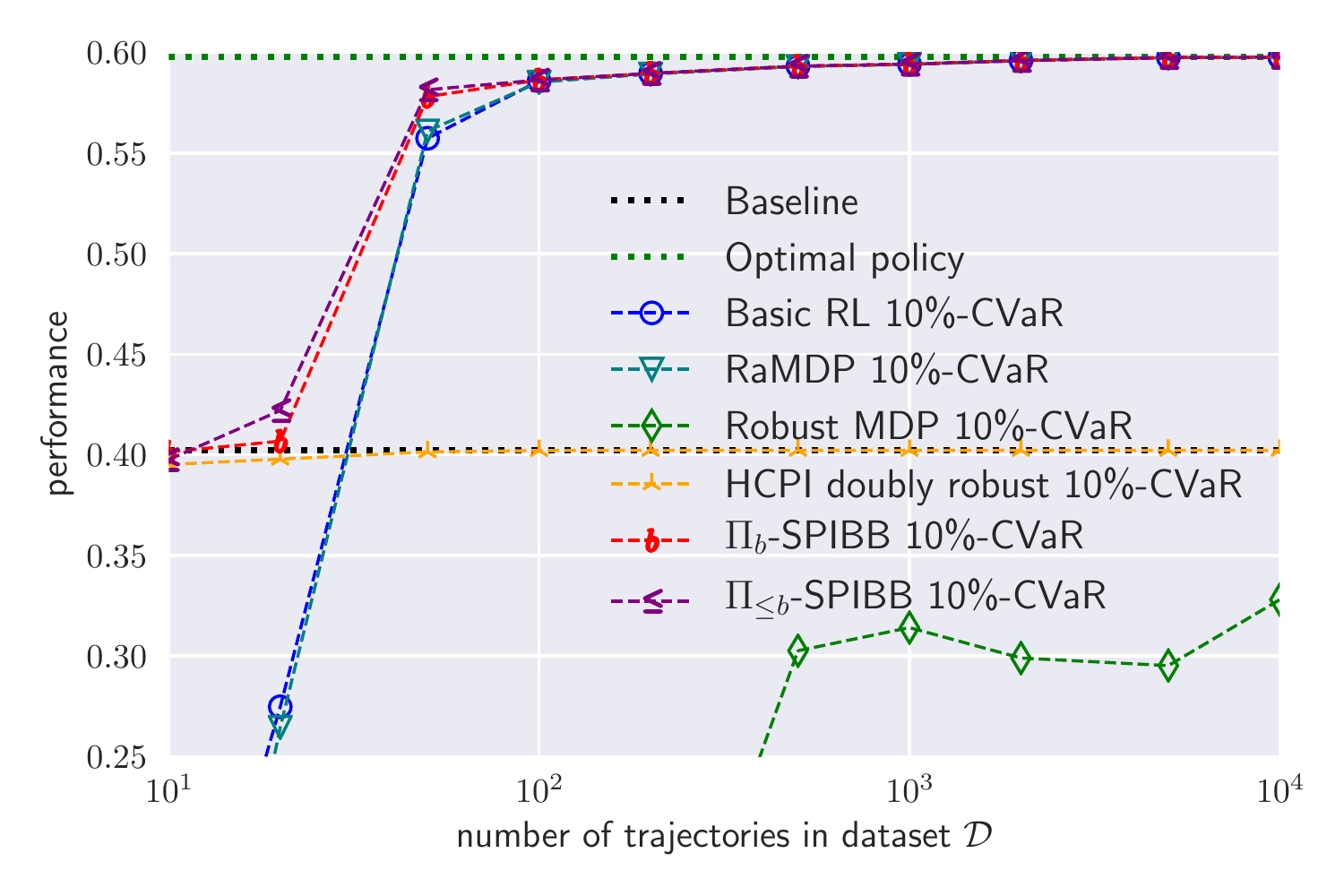}
			\label{fig:maze_rand_baseline_10CVaR_benchmark}
		}
		\subfloat[0.1\%-CVaR: benchmark with $N_\wedge=20$.]{
			\includegraphics[trim = 5pt 5pt 5pt 5pt, clip, width=0.33\textwidth]{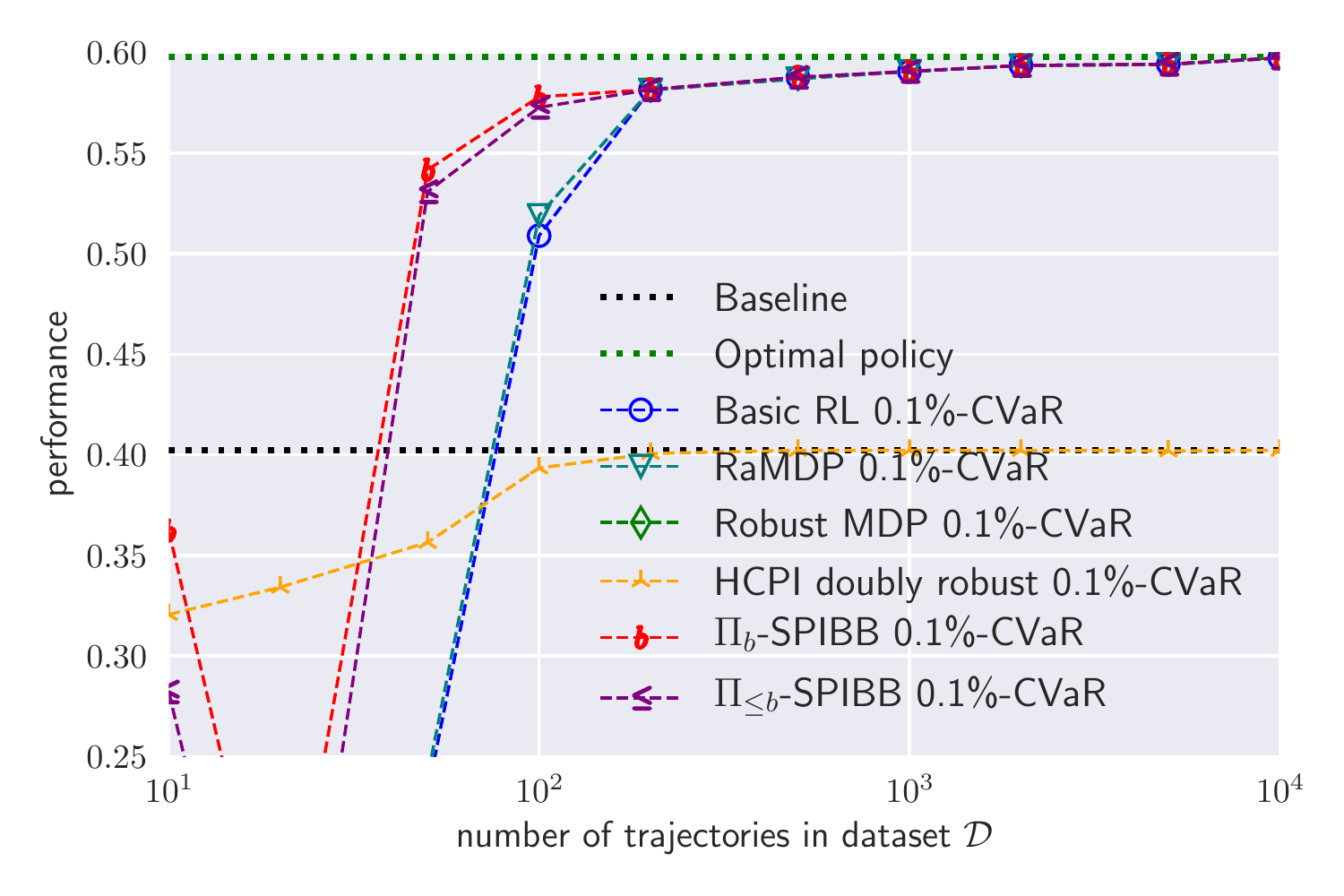}
			\label{fig:maze_rand_baseline_0.1CVaR_benchmark}
		}\\
		\centering
		\subfloat[Mean \& 1\%-CVaR: SPIBB w. $N_\wedge=5$.]{
			\includegraphics[trim = 5pt 5pt 5pt 5pt, clip, width=0.33\textwidth]{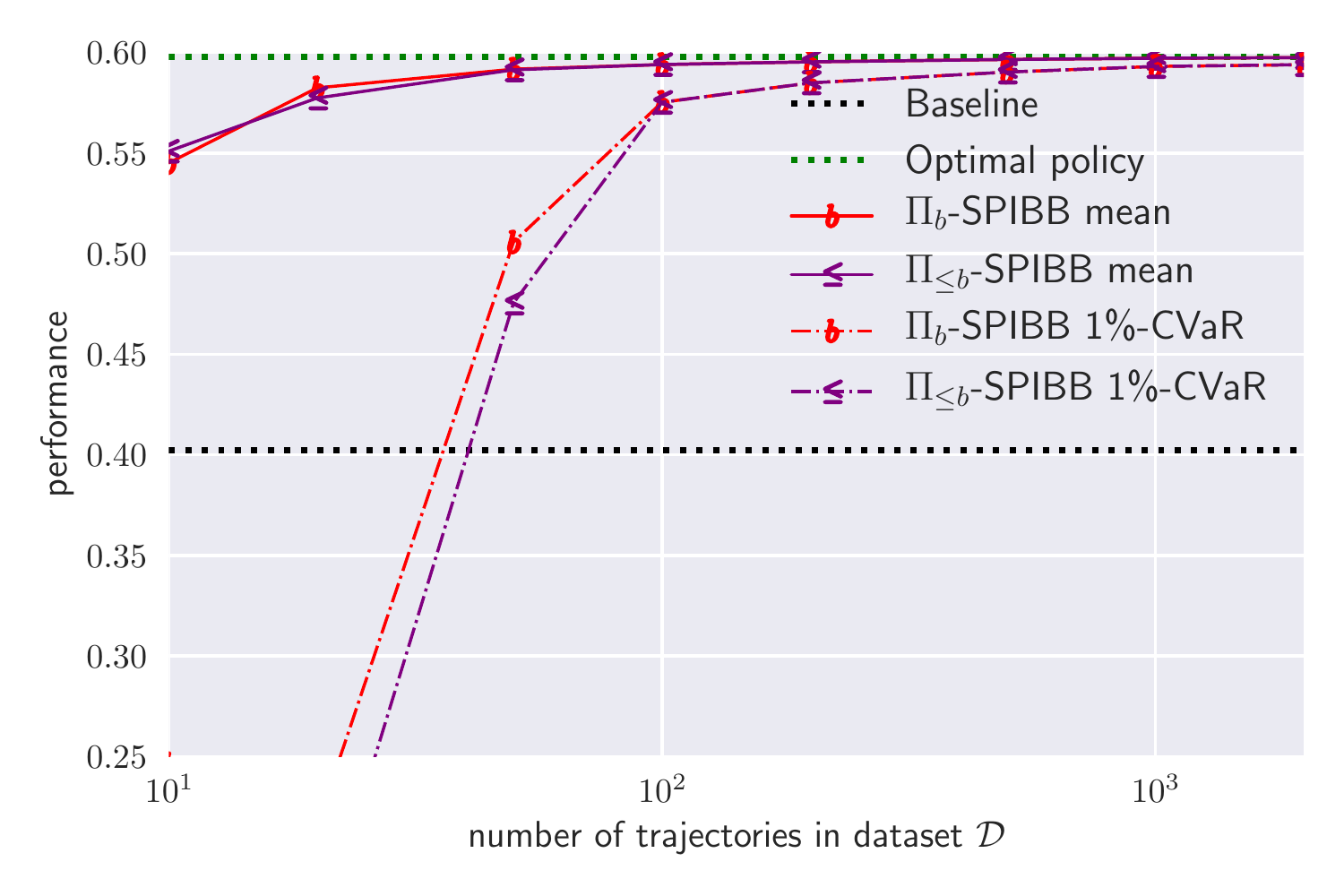}
			\label{fig:maze_rand_baseline_N=5}
		}
		\subfloat[Mean \& 1\%-CVaR: SPIBB w. $N_\wedge=10$.]{
			\includegraphics[trim = 5pt 5pt 5pt 5pt, clip, width=0.33\textwidth]{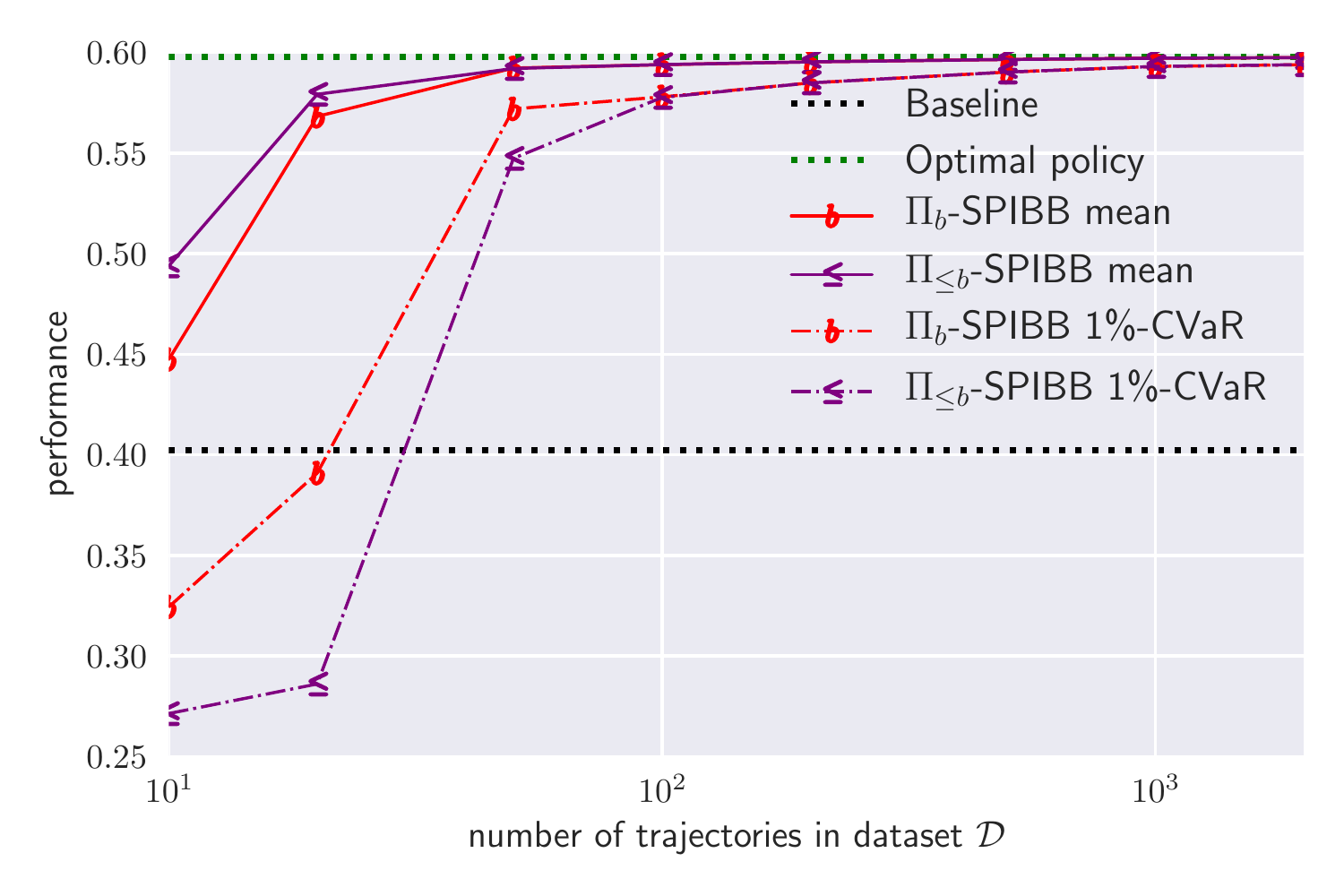}
			\label{fig:maze_rand_baseline_N=10}
		}
		\subfloat[Mean \& 1\%-CVaR: SPIBB w. $N_\wedge=50$.]{
			\includegraphics[trim = 5pt 5pt 5pt 5pt, clip, width=0.33\textwidth]{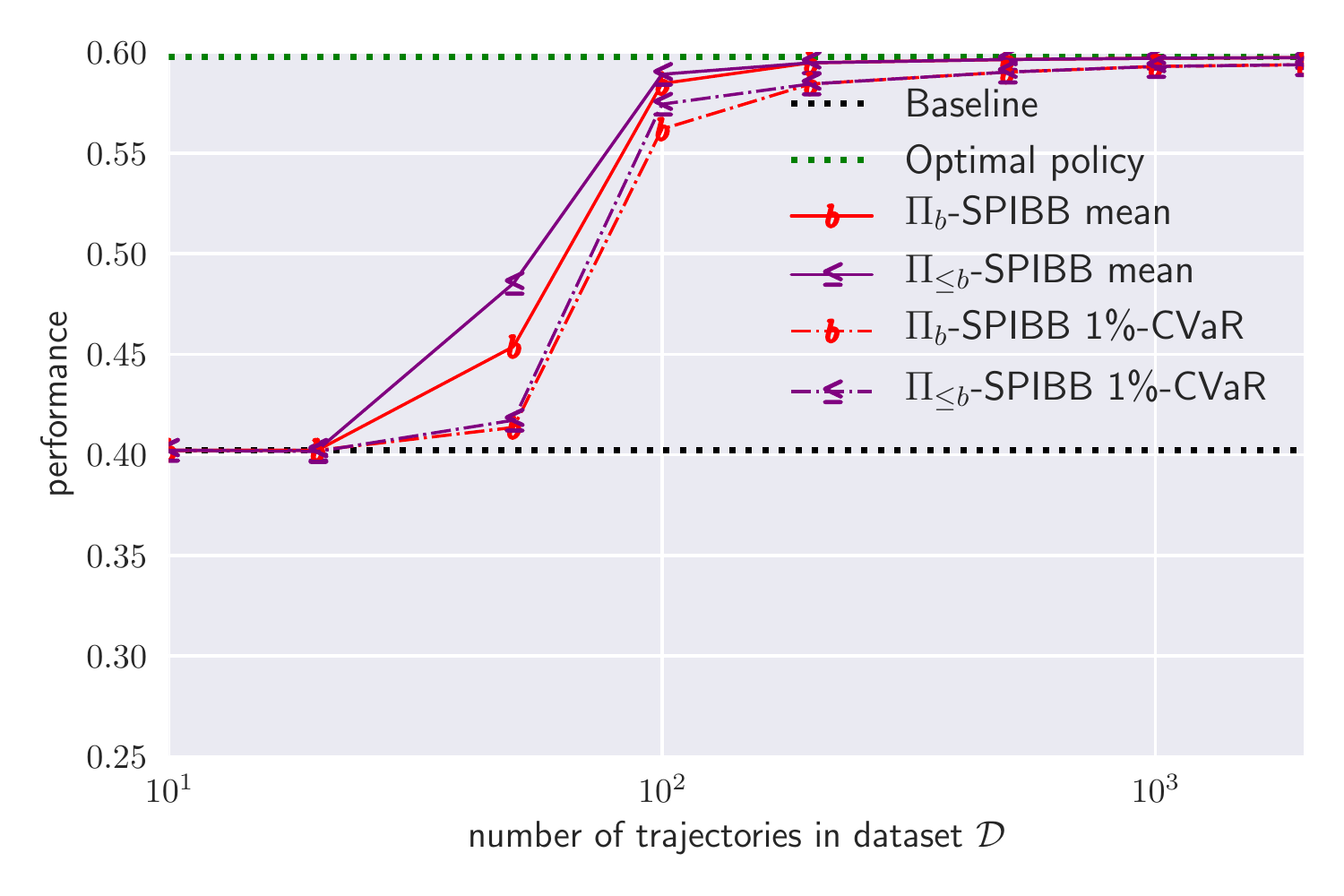}
			\label{fig:maze_rand_baseline_N=50}
		} 
		\caption{Gridworld experiment with random behavioural policy: Figures (a-c) respectively show the benchmark for the mean, 10\%-CVaR and 0.1\%-CVaR performances. Figures (d-f) display additional curves for other $N_\wedge$ values: respectively 5, 10, 50. }
		\label{fig:maze_rand_baseline_sup}
		\vspace{-10pt}
	\end{figure*}
	
	\subsection{Full Random MDPs experiment results}
	\label{sup:random_MDPs_full_results}
	\begin{figure*}[ht!]
		\centering
		\subfloat[1\%-CVaR: Basic RL.]{
			\includegraphics[trim = 10pt 40pt 50pt 60pt, clip, width=0.33\textwidth]{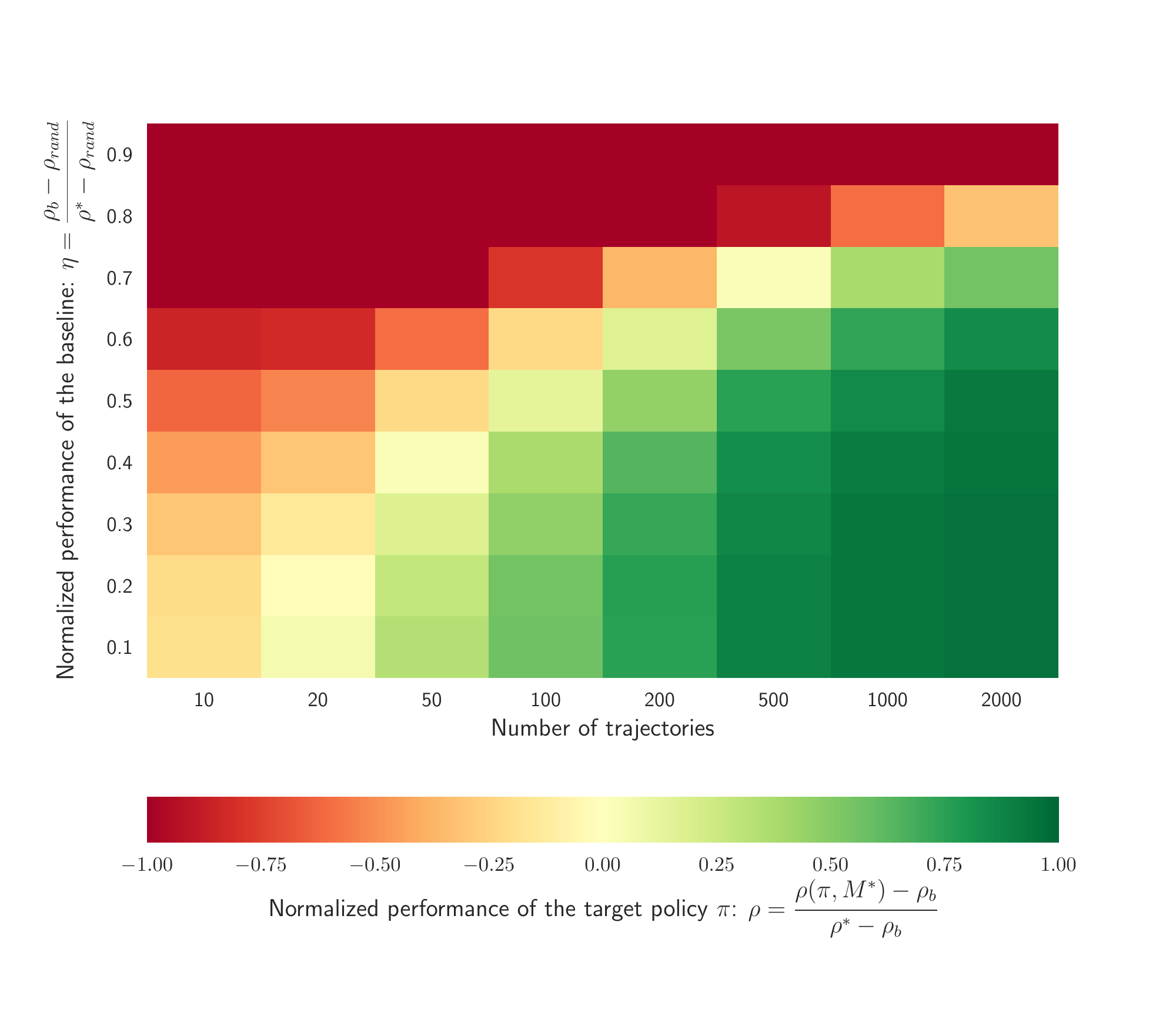}
			\label{fig:random_MDPs_heat_basic}
		}
		\subfloat[1\%-CVaR: HCPI doubly robust.]{
			\includegraphics[trim = 10pt 40pt 50pt 60pt, clip, width=0.33\textwidth]{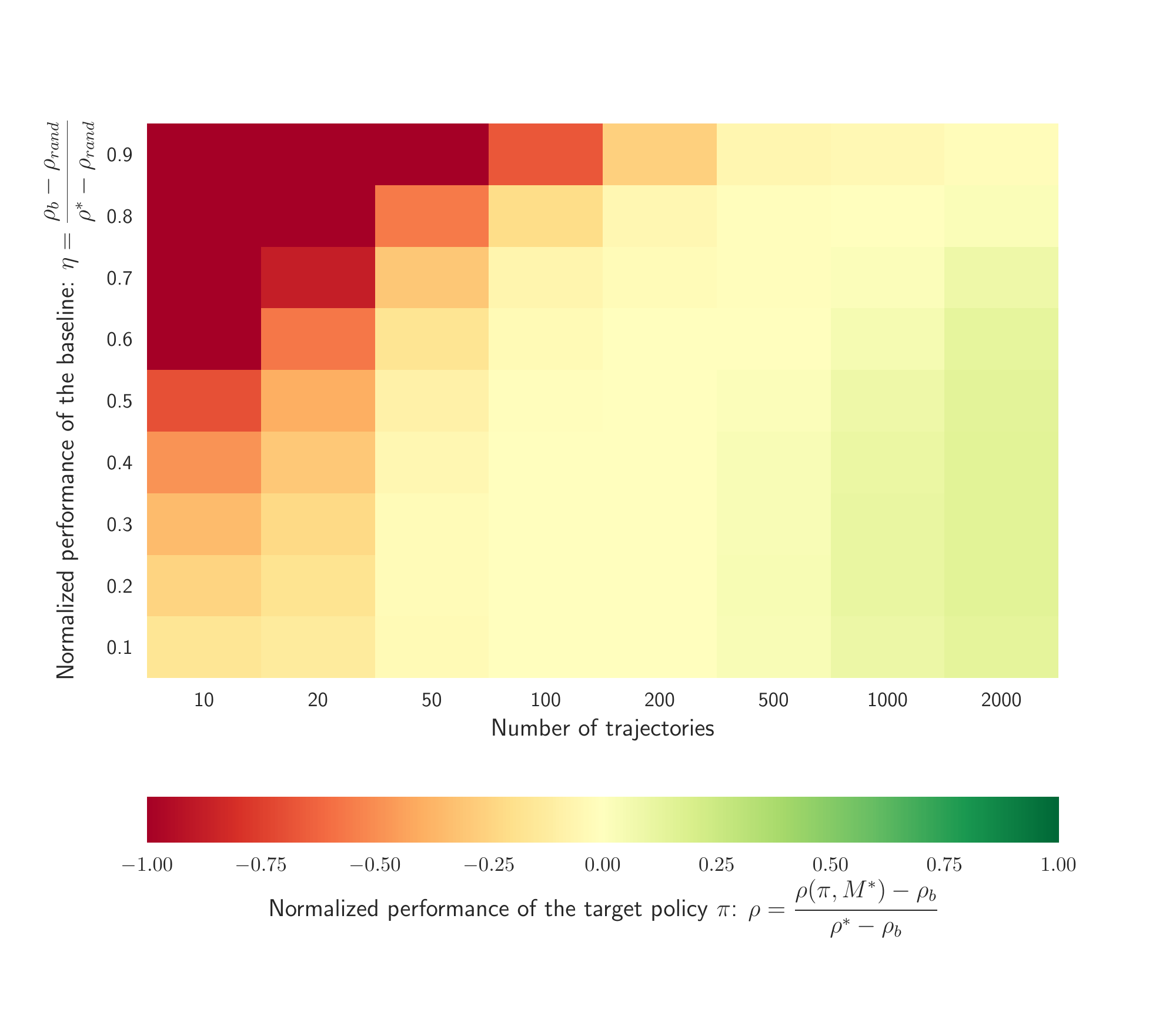}
			\label{fig:random_MDPs_heat_HCPI}
		}
		\subfloat[1\%-CVaR: Robust MDP.]{
			\includegraphics[trim = 10pt 40pt 50pt 60pt, clip, width=0.33\textwidth]{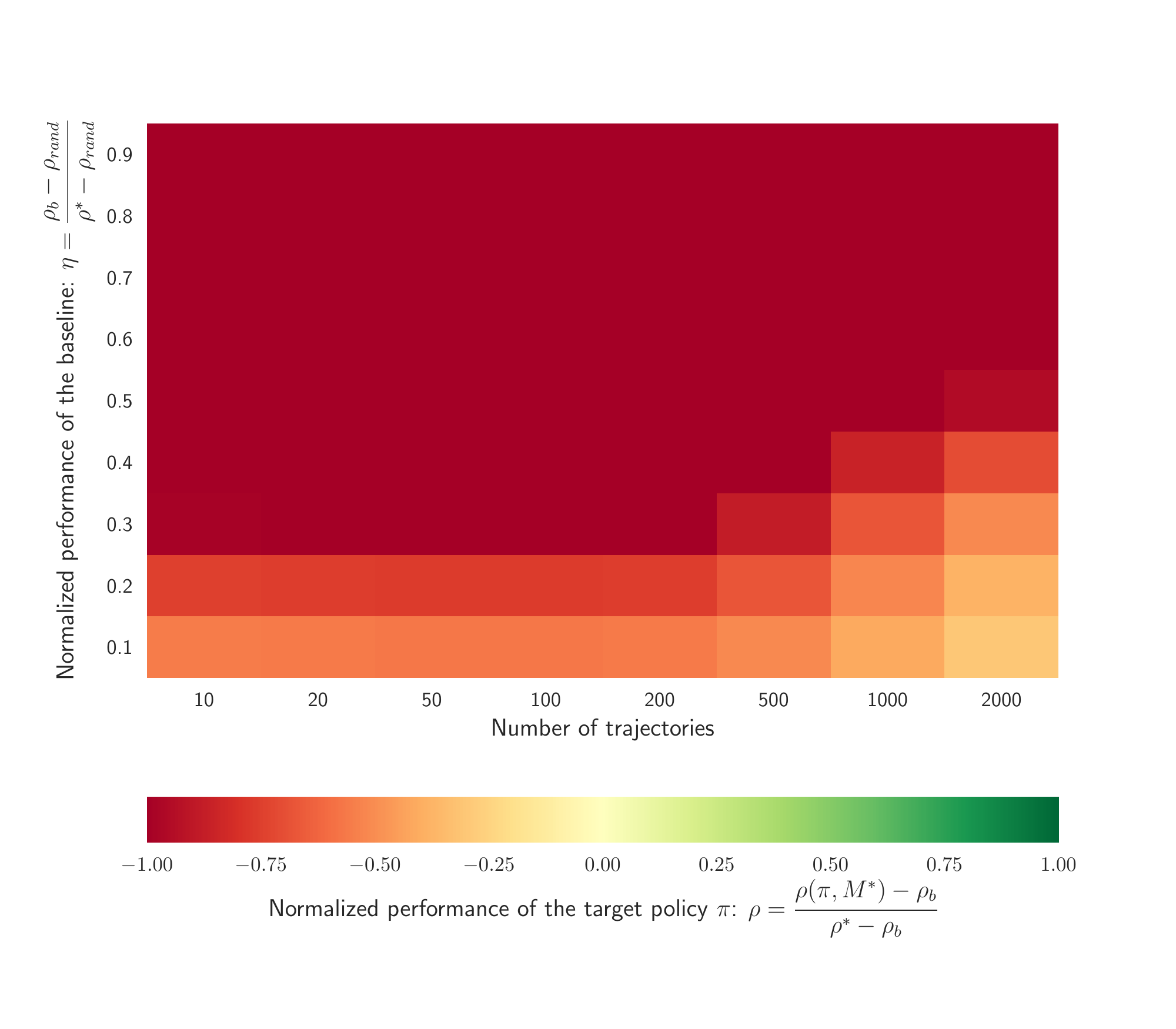}
			\label{fig:random_MDPs_heat_RobustMDP}
		} \\
		\centering
		\subfloat[1\%-CVaR: $\Pi_{\leq b}$-SPIBB, $N_\wedge=5$.]{
			\includegraphics[trim = 10pt 130pt 50pt 60pt, clip, width=0.33\textwidth]{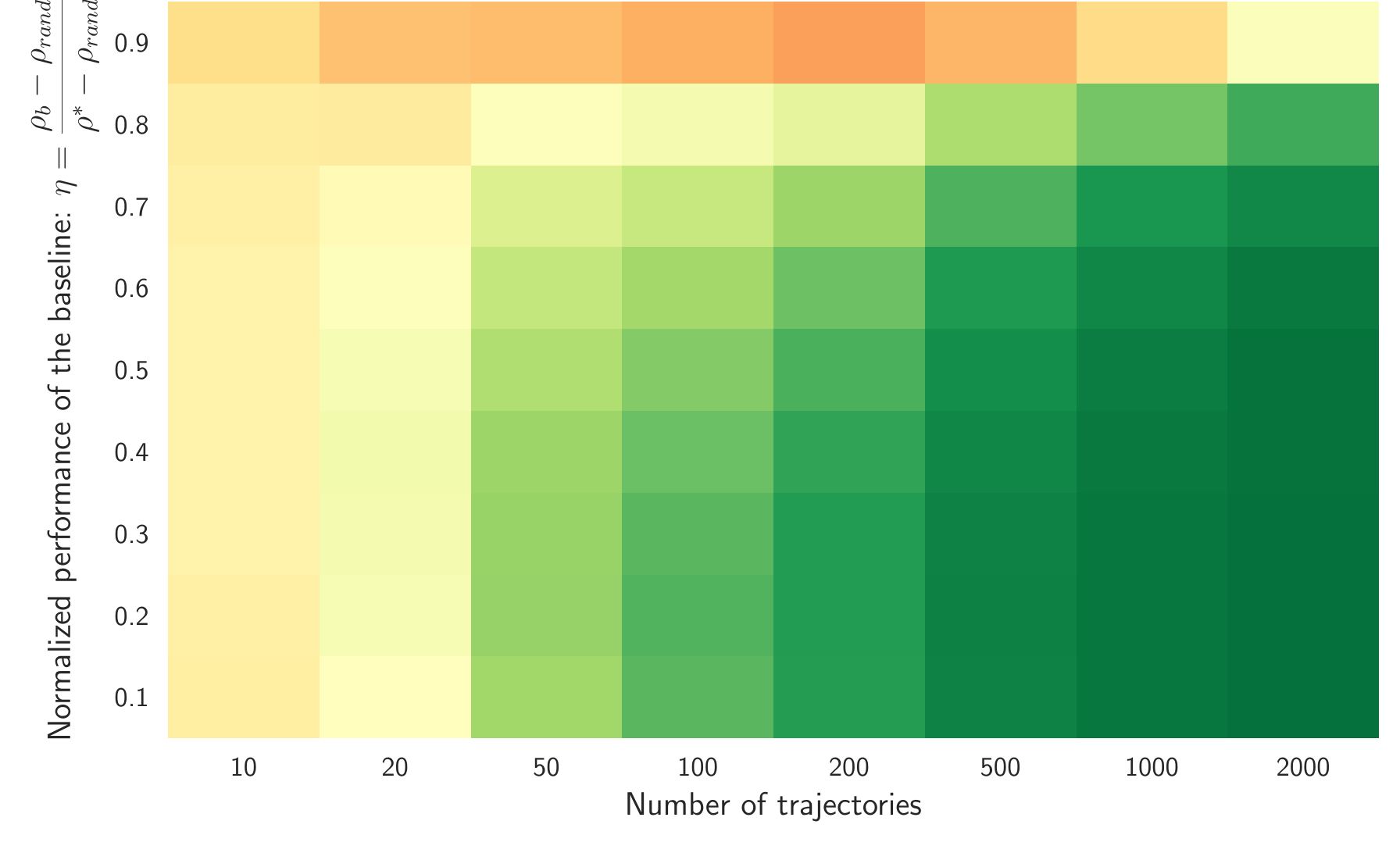}
			\label{fig:random_MDPs_heat_pi<b_N=5}
		}
		\subfloat[1\%-CVaR: $\Pi_{\leq b}$-SPIBB, $N_\wedge=20$.]{
			\includegraphics[trim = 10pt 130pt 50pt 60pt, clip, width=0.33\textwidth]{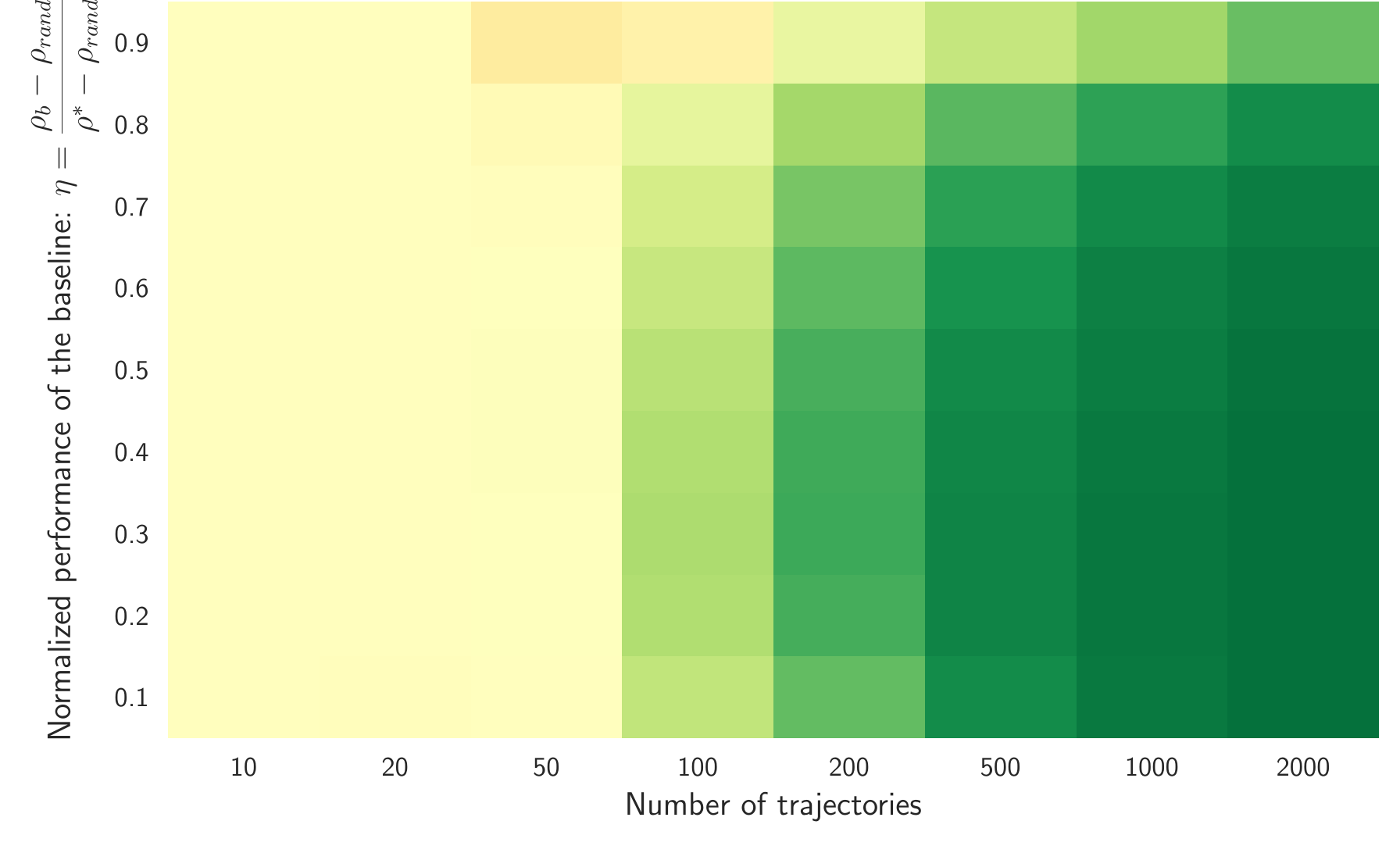}
			\label{fig:random_MDPs_heat_pi<b_N=20}
		}
		\subfloat[1\%-CVaR: $\Pi_{\leq b}$-SPIBB, $N_\wedge=50$.]{
			\includegraphics[trim = 10pt 130pt 50pt 60pt, clip, width=0.33\textwidth]{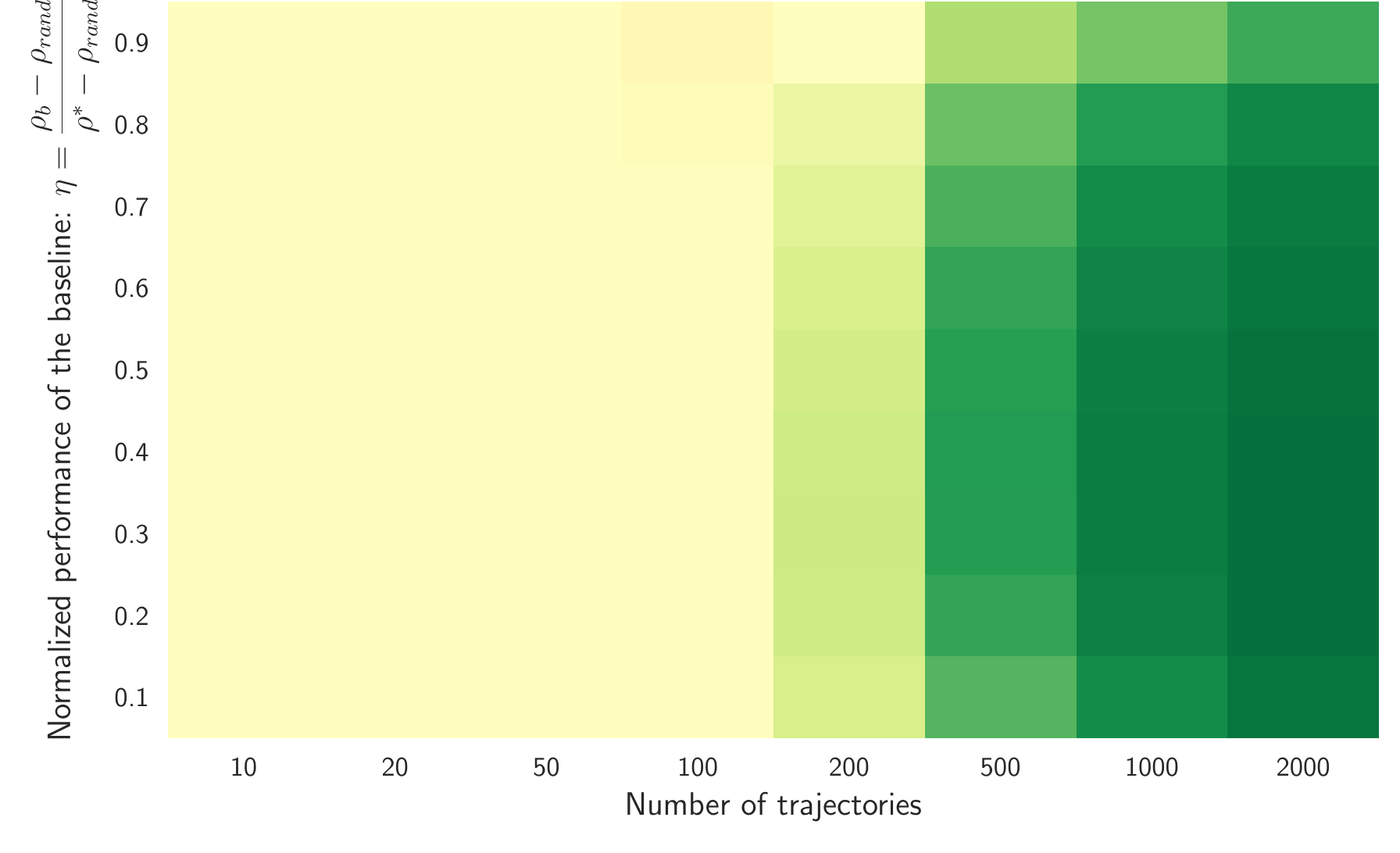}
			\label{fig:random_MDPs_heat_pi<b_N=50}
		}  \\
		\centering
		\subfloat[1\%-CVaR: $\Pi_{\leq b}$-SPIBB, $N_\wedge=100$.]{
			\includegraphics[trim = 10pt 130pt 50pt 60pt, clip, width=0.33\textwidth]{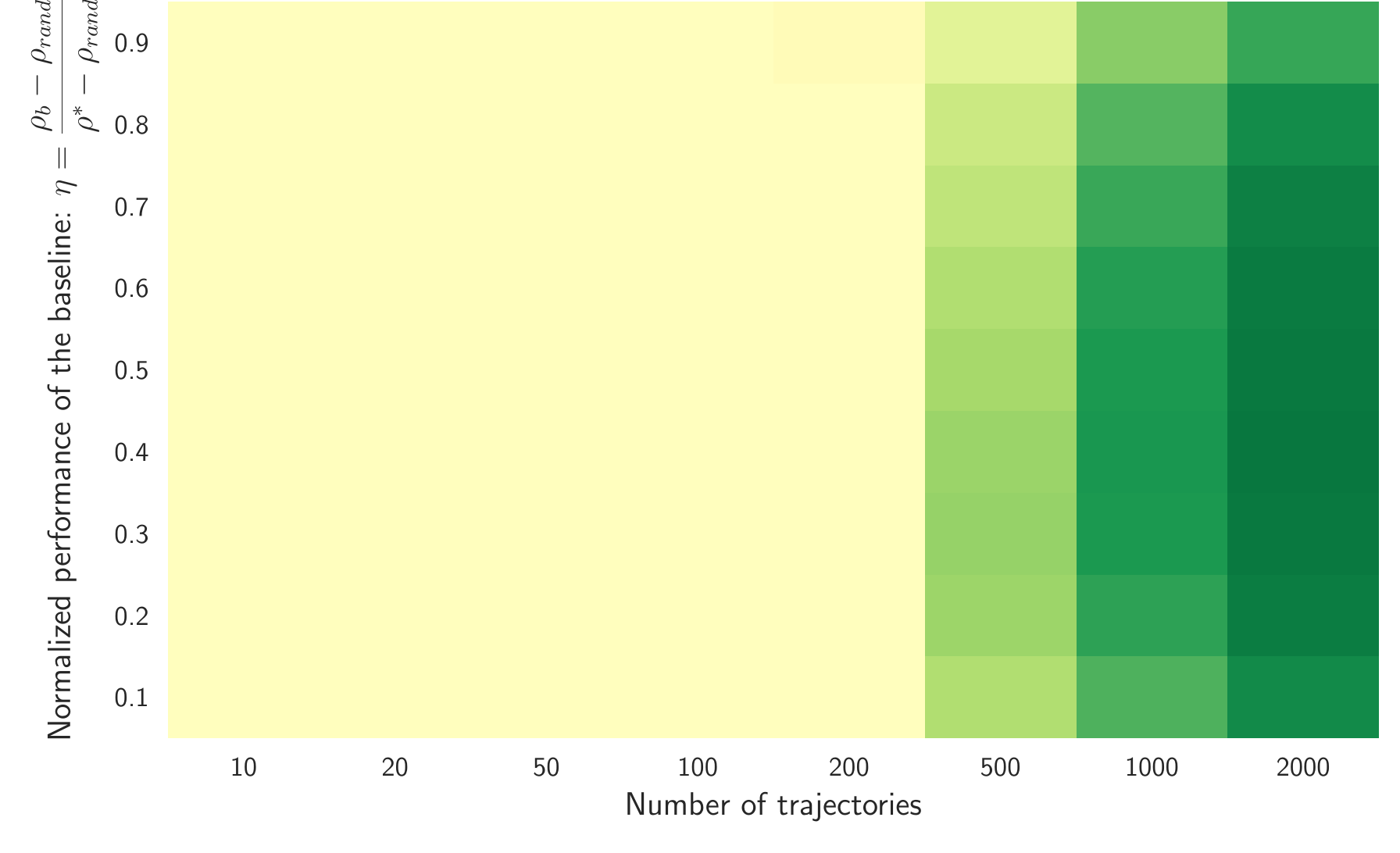}
			\label{fig:random_MDPs_heat_pi<b_N=100}
		} 
		\subfloat[1\%-CVaR: $\Pi_{b}$-SPIBB, $N_\wedge=5$.]{
			\includegraphics[trim = 10pt 130pt 50pt 60pt, clip, width=0.33\textwidth]{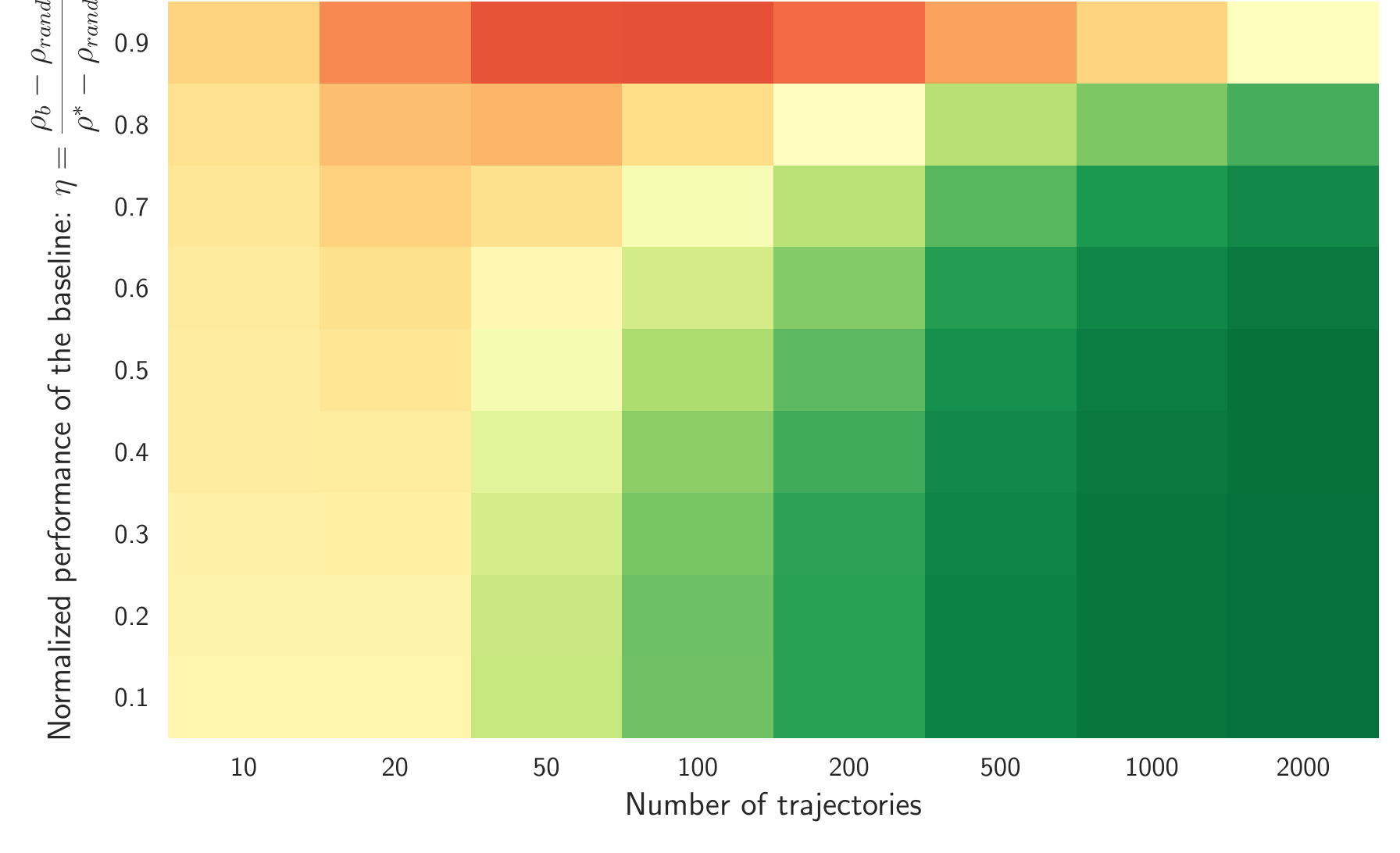}
			\label{fig:random_MDPs_heat_pib_N=5}
		}
		\subfloat[1\%-CVaR: $\Pi_{b}$-SPIBB, $N_\wedge=10$.]{
			\includegraphics[trim = 10pt 130pt 50pt 60pt, clip, width=0.33\textwidth]{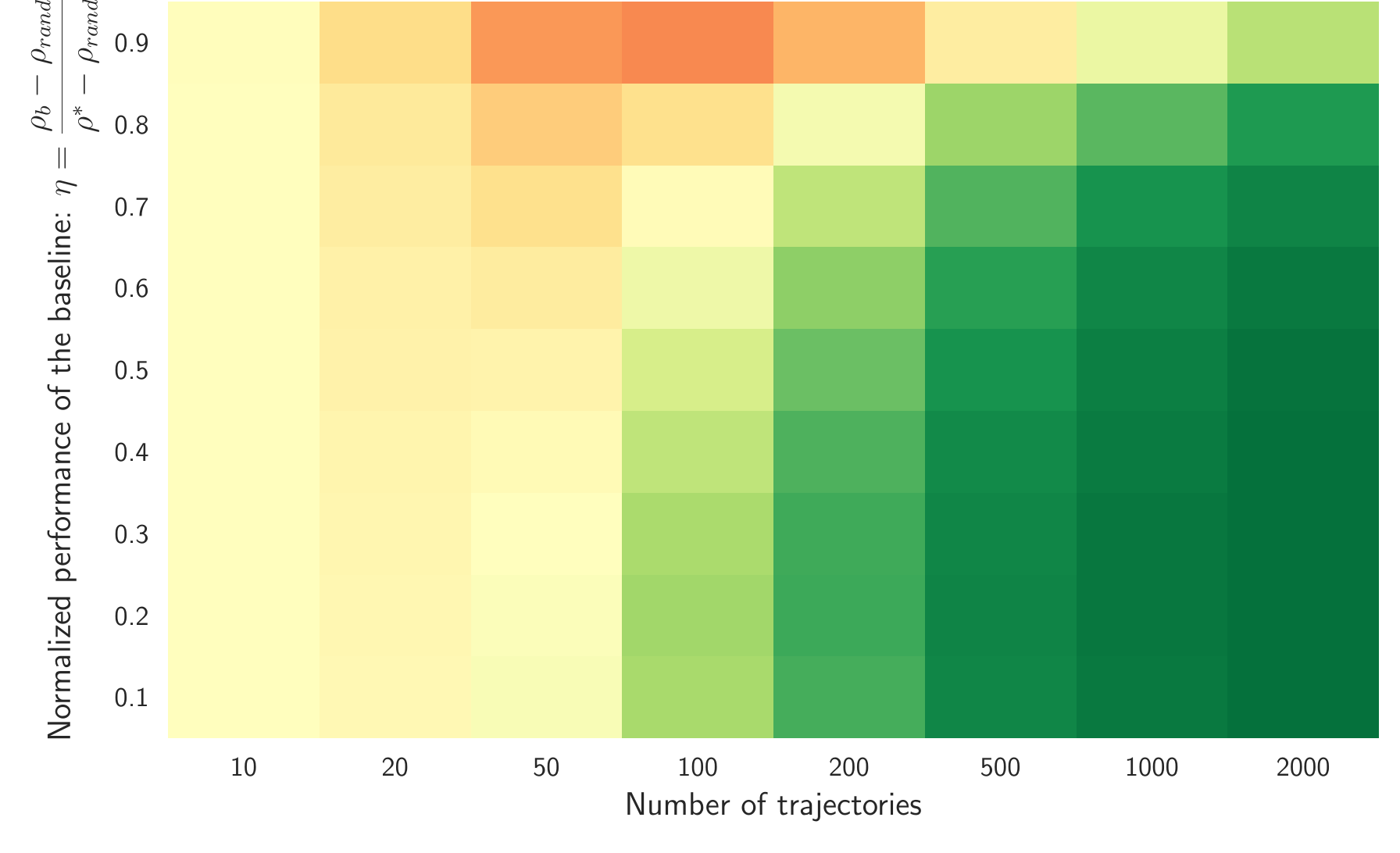}
			\label{fig:random_MDPs_heat_pib_N=10}
		} \\
		\centering
		\subfloat[1\%-CVaR: $\Pi_{b}$-SPIBB, $N_\wedge=20$.]{
			\includegraphics[trim = 10pt 130pt 50pt 60pt, clip, width=0.33\textwidth]{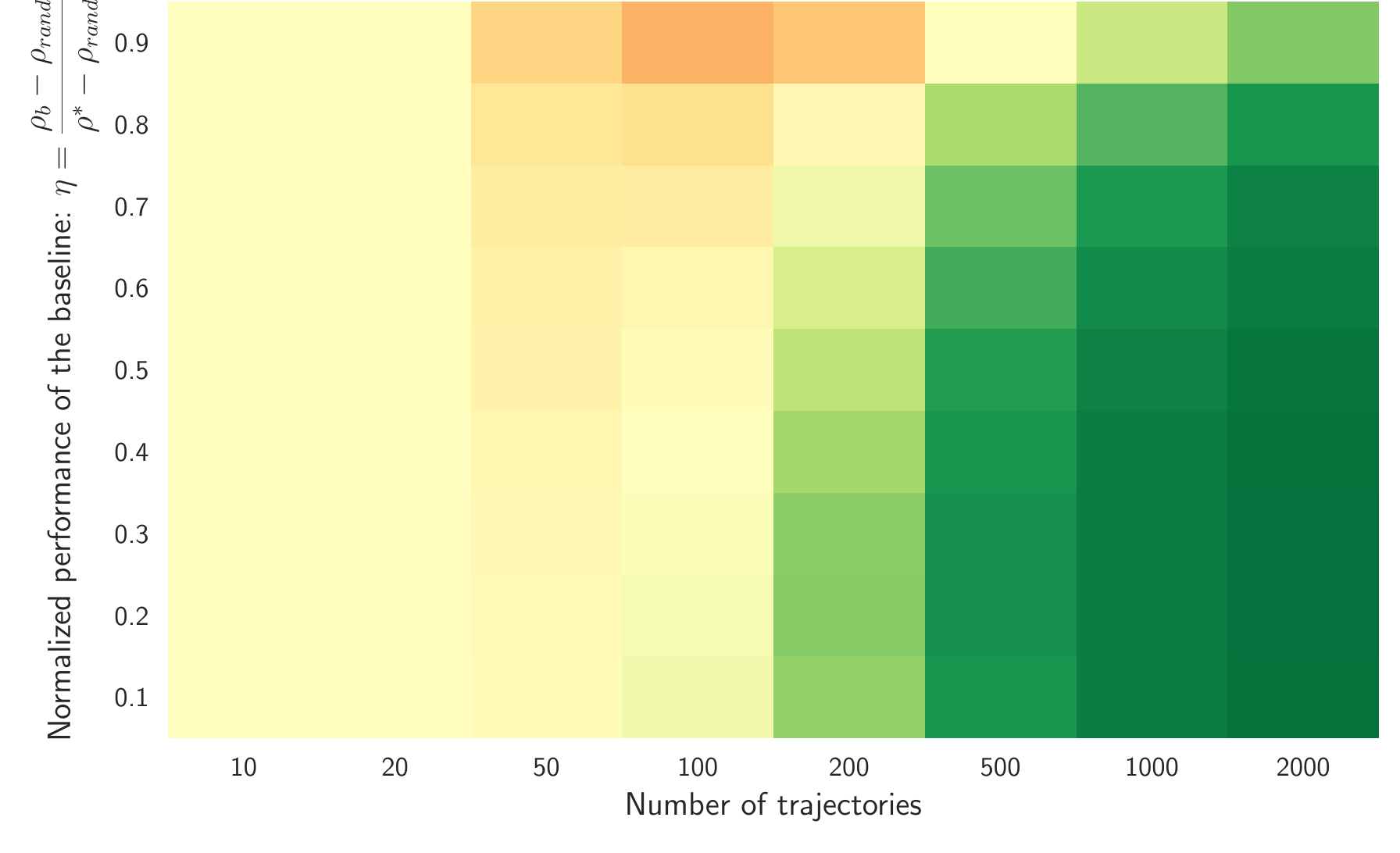}
			\label{fig:random_MDPs_heat_pib_N=20}
		} 
		\subfloat[1\%-CVaR: $\Pi_{b}$-SPIBB, $N_\wedge=50$.]{
			\includegraphics[trim = 10pt 130pt 50pt 60pt, clip, width=0.33\textwidth]{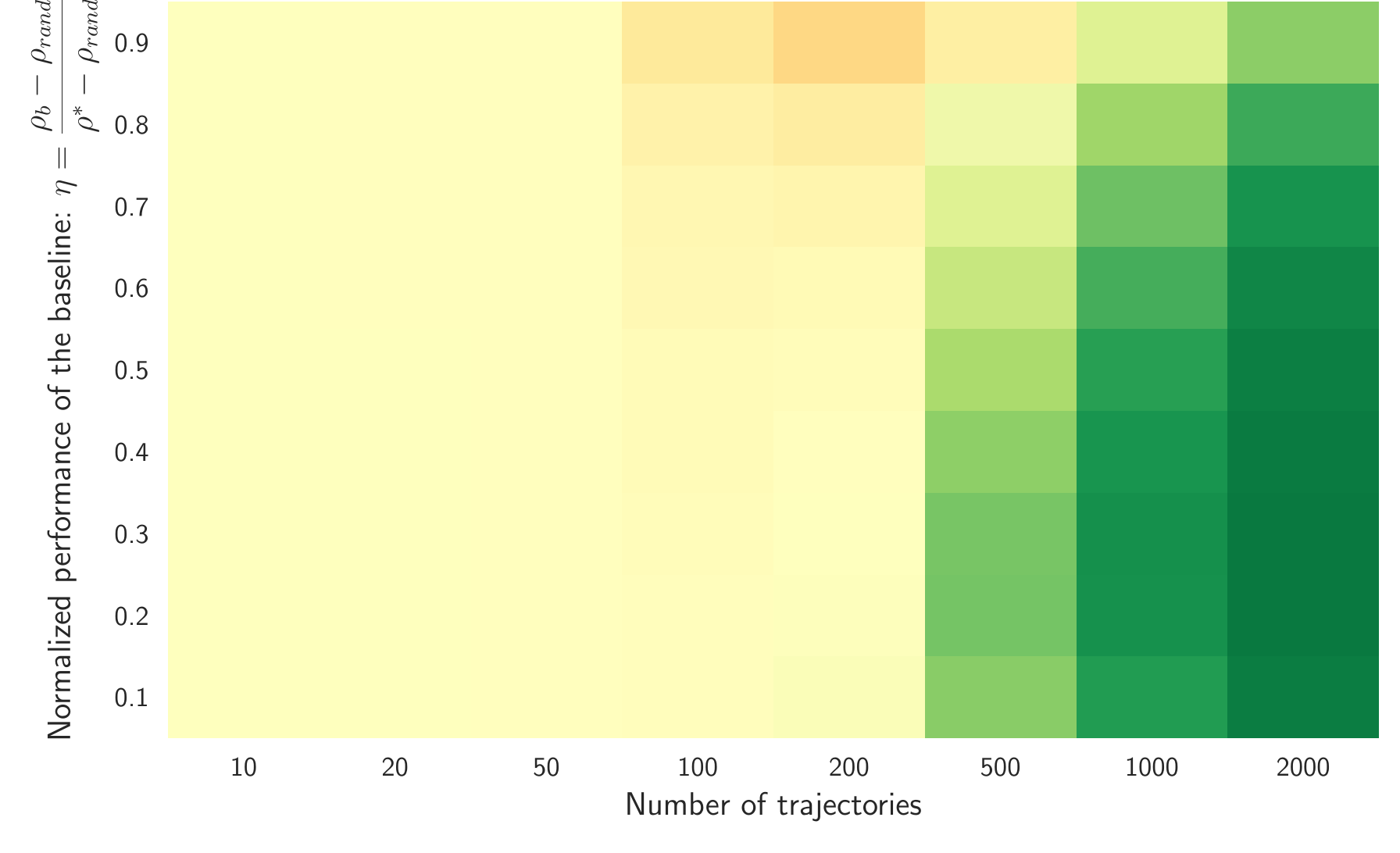}
			\label{fig:random_MDPs_heat_pib_N=50}
		}
		\subfloat[1\%-CVaR: $\Pi_{b}$-SPIBB, $N_\wedge=100$.]{
			\includegraphics[trim = 10pt 130pt 50pt 60pt, clip, width=0.33\textwidth]{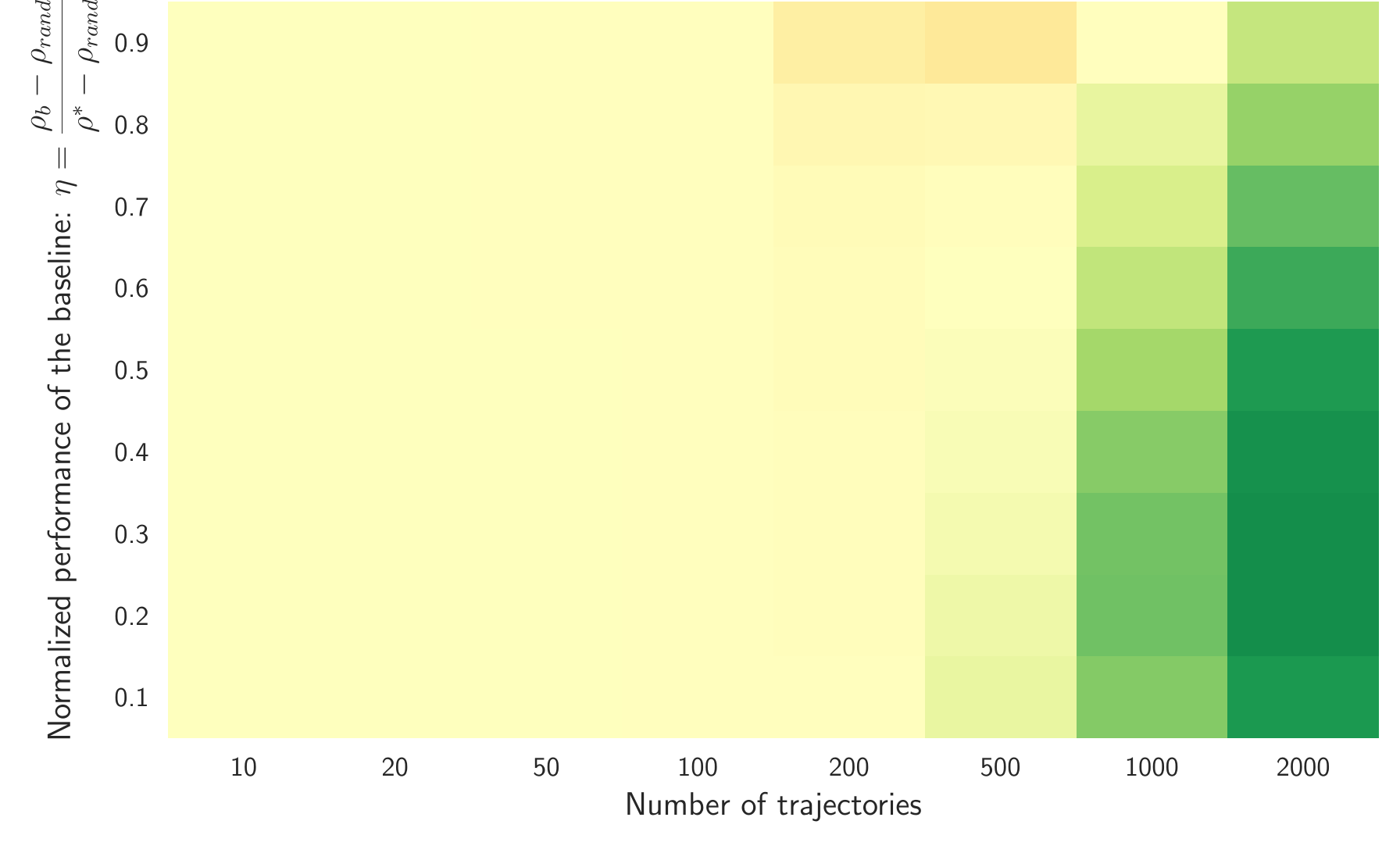}
			\label{fig:random_MDPs_heat_pib_N=100}
		}
		\caption{Random MDPs: 1\%-CVaR performance heatmaps. The abscissae is the dataset size, the ordinate is the baseline hyperparameter $\eta$, and the color is the normalized performance: red, yellow, and green respectively mean below, equal to, and above baseline performance. Heatmaps for the mean normalized performance and for additional $N_\wedge$ values: 7, 15, 30, 70, may be found in the supplementary material package. The supplementary material package also contains more heatmaps on the sensitivity to $N_\wedge$, with fixed $\eta$ values.}
		\label{fig:randomMDP_sup_heatmaps}
		\vspace{-10pt}
	\end{figure*}
	
	\begin{figure*}[ht!]
		\centering
		\subfloat[1\%-CVaR: with $\eta=0.1$ and $N_\wedge=10$.]{
			\includegraphics[trim = 5pt 5pt 5pt 5pt, clip, width=0.33\textwidth]{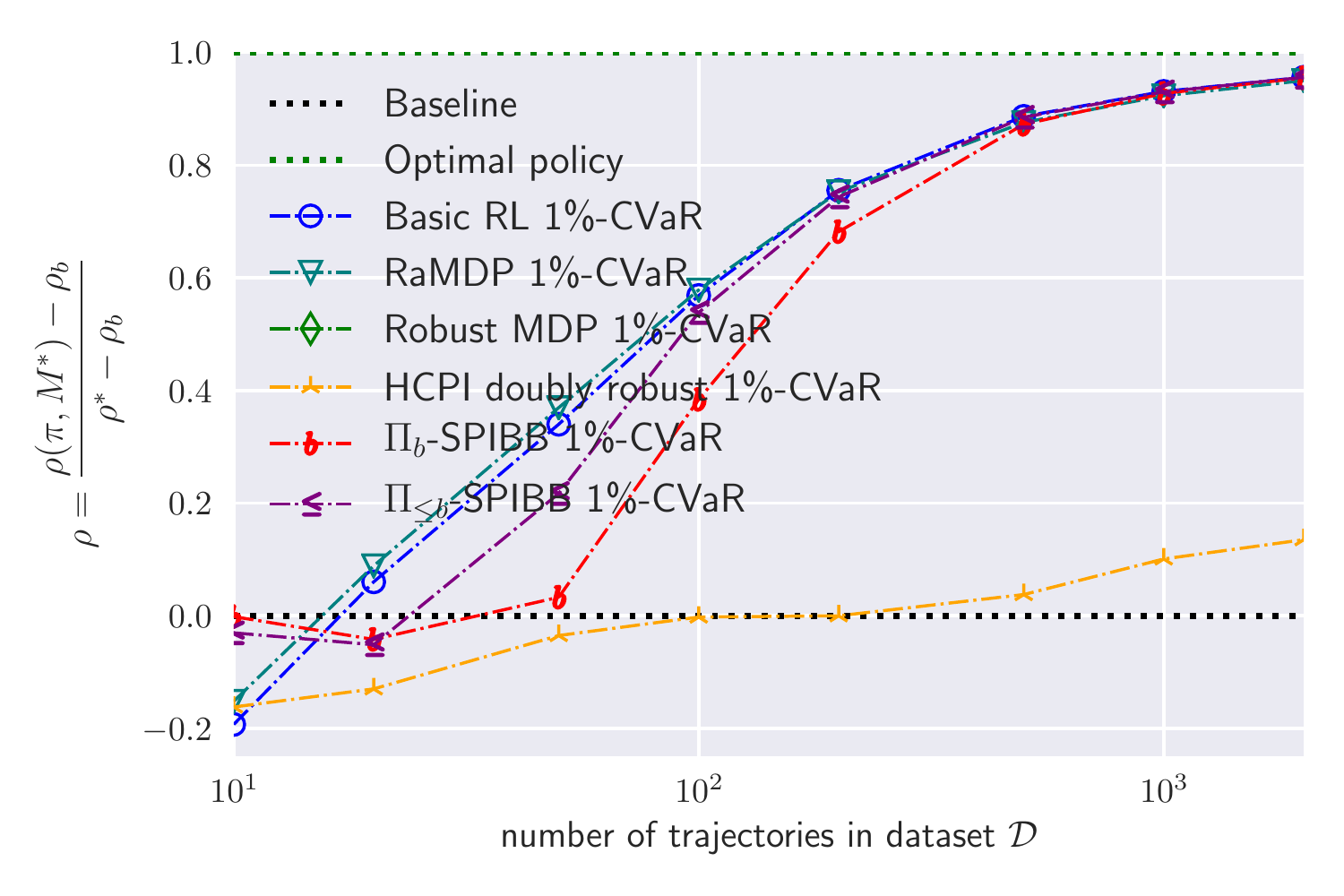}
			\label{fig:random_MDPs_1-CVaR_eta=0.1}
		}
		\subfloat[10\%-CVaR: with $\eta=0.1$ and $N_\wedge=10$.]{
			\includegraphics[trim = 5pt 5pt 5pt 5pt, clip, width=0.33\textwidth]{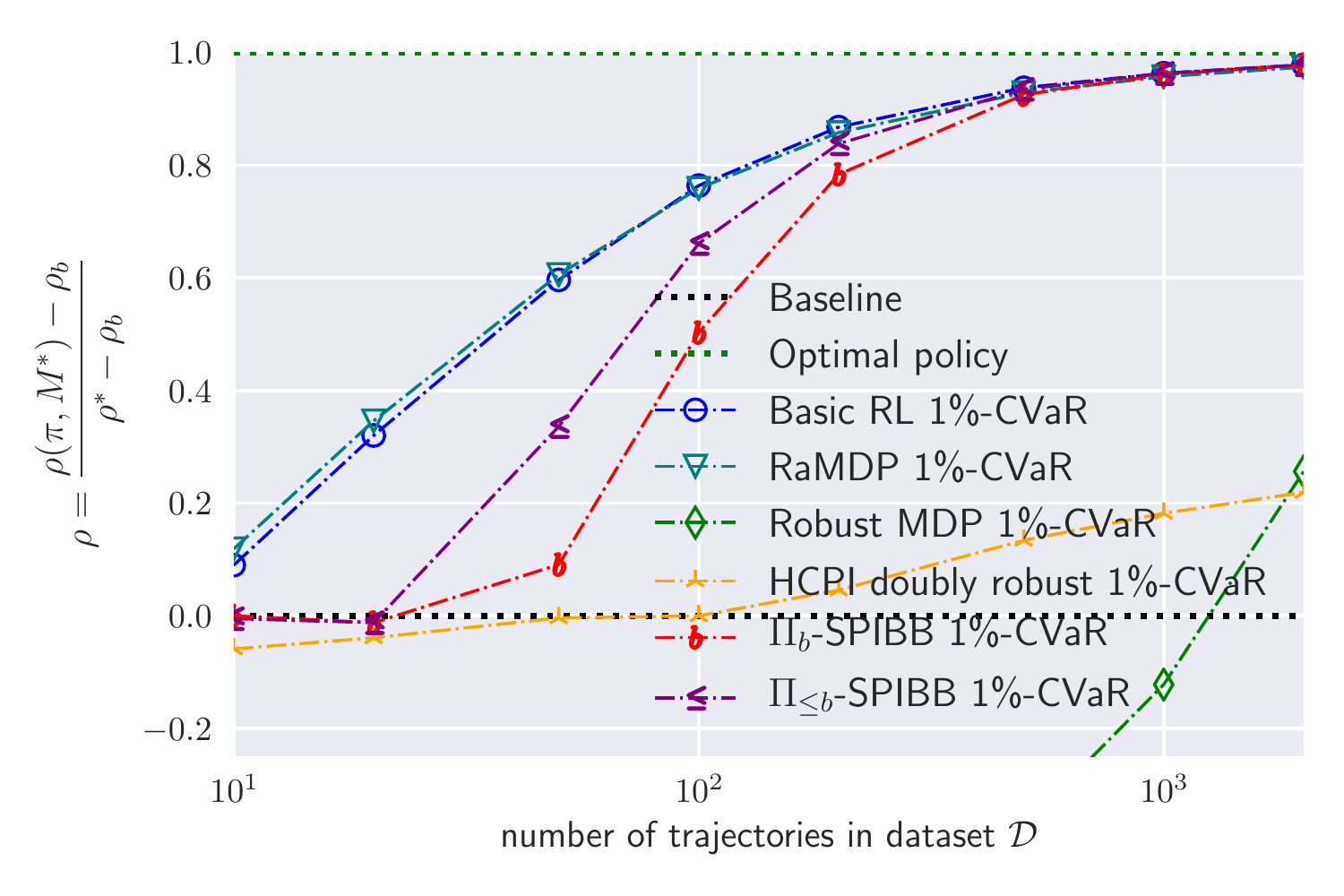}
			\label{fig:random_MDPs_10-CVaR_eta=0.1}
		}
		\subfloat[10\%-CVaR: with $\eta=0.9$ and $N_\wedge=10$.]{
			\includegraphics[trim = 5pt 5pt 5pt 5pt, clip, width=0.33\textwidth]{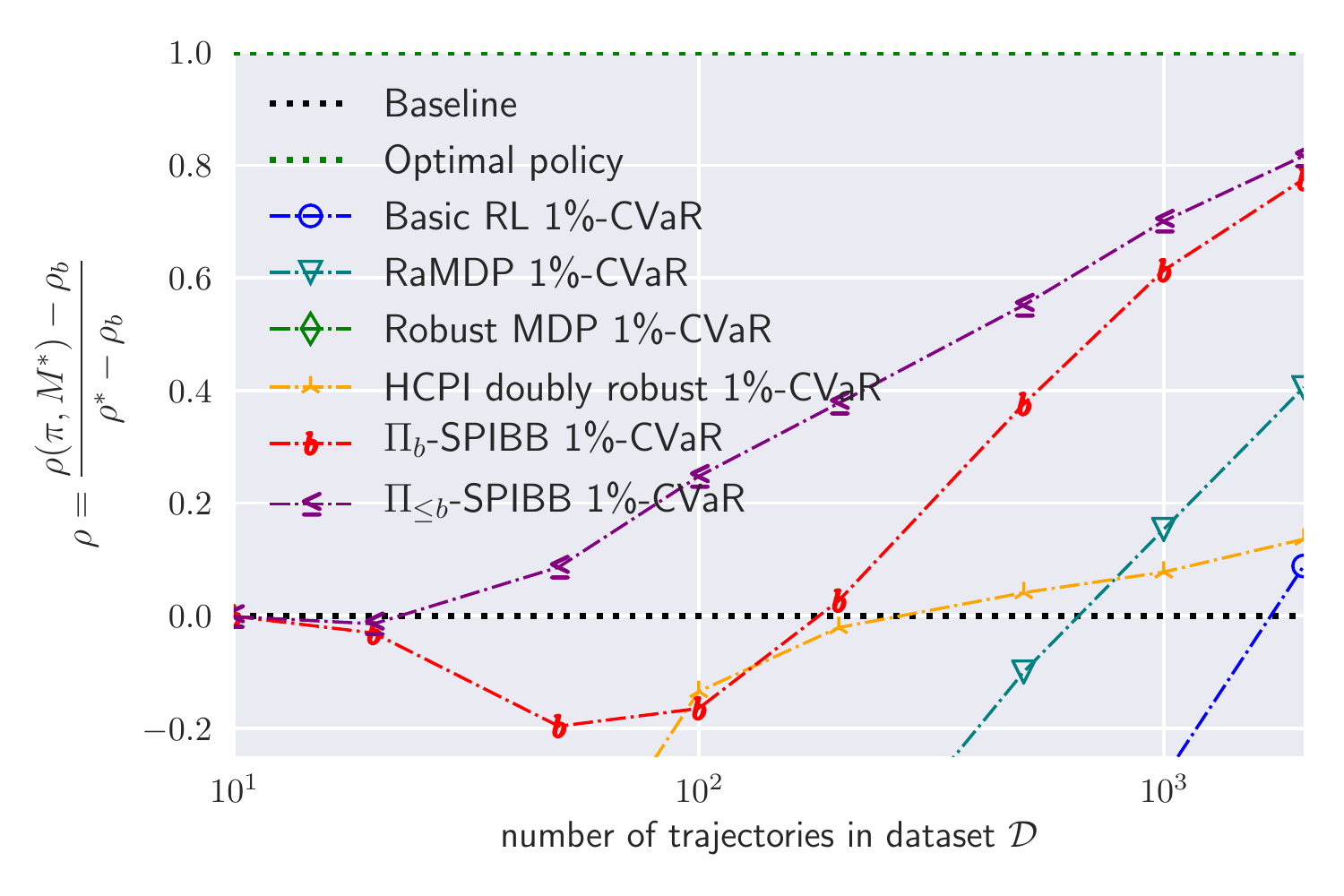}
			\label{fig:random_MDPs_10-CVaR_eta=0.9}
		} \\
		\centering
		\subfloat[1\%-CVaR: with $\eta=0.3$ and $N_\wedge=10$.]{
			\includegraphics[trim = 5pt 5pt 5pt 5pt, clip, width=0.33\textwidth]{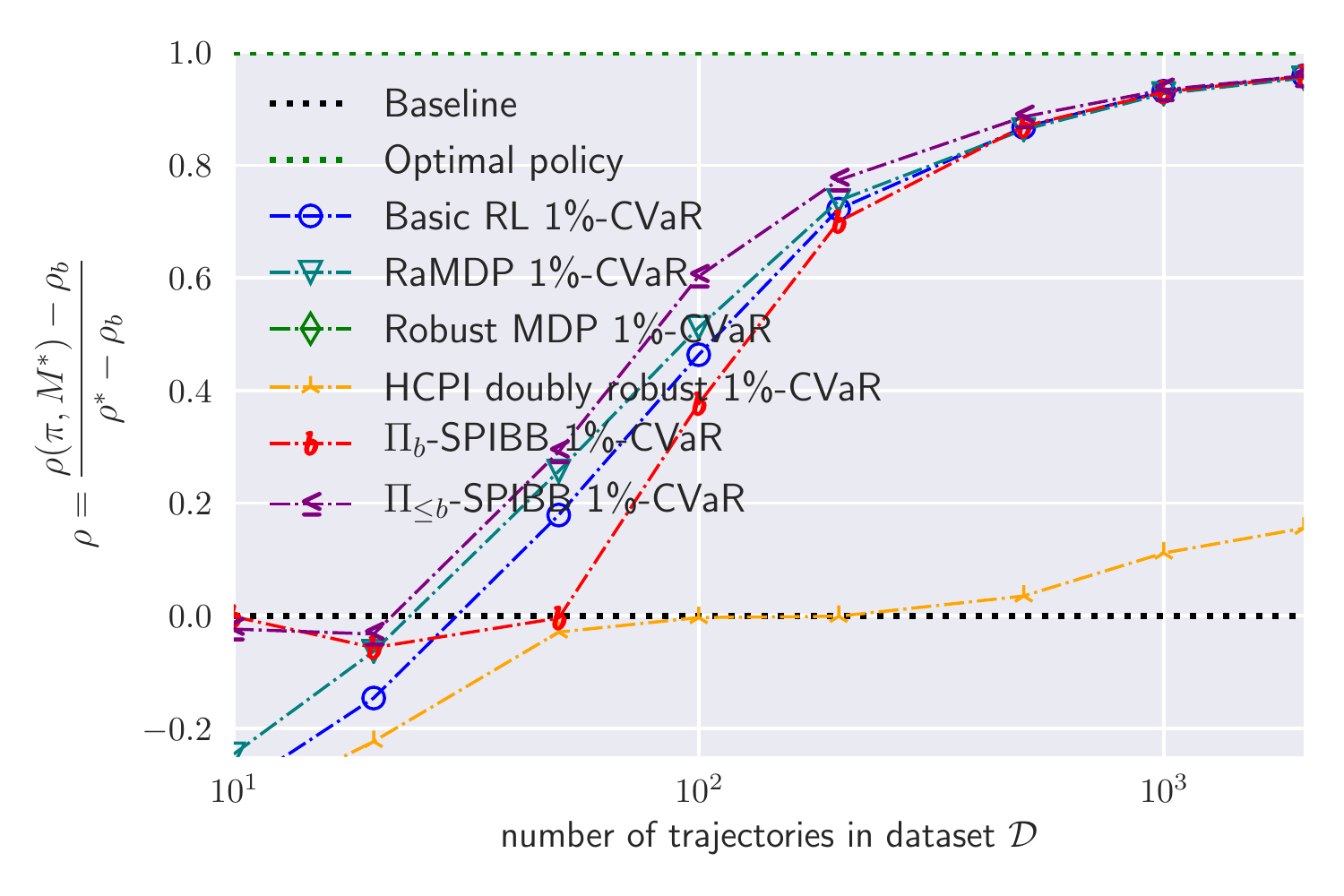}
			\label{fig:random_MDPs_1-CVaR_eta=0.3}
		}
		\subfloat[10\%-CVaR: with $\eta=0.3$ and $N_\wedge=10$.]{
			\includegraphics[trim = 5pt 5pt 5pt 5pt, clip, width=0.33\textwidth]{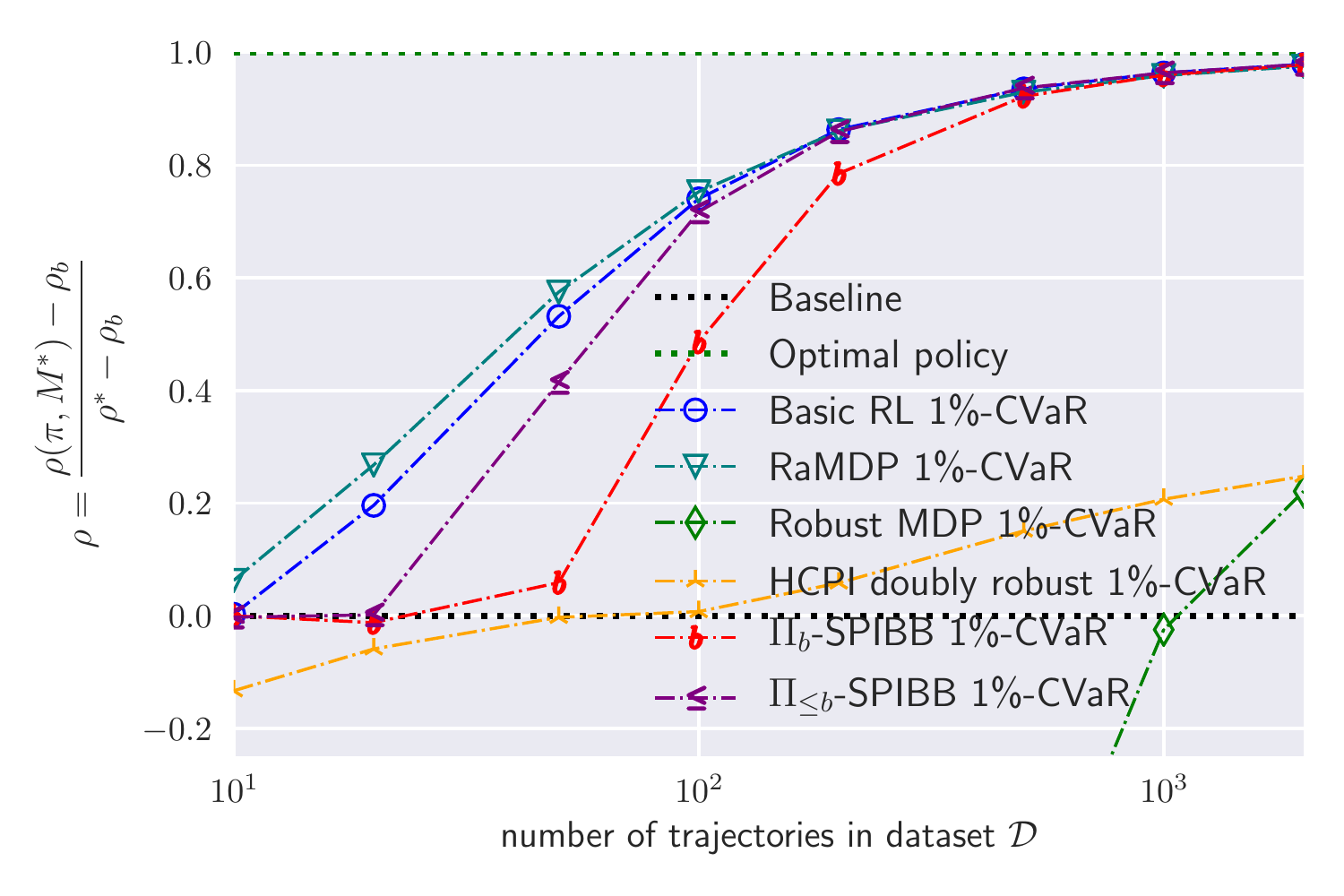}
			\label{fig:random_MDPs_10-CVaR_eta=0.3}
		}
		\subfloat[Mean: with $\eta=0.3$ and $N_\wedge=10$.]{
			\includegraphics[trim = 5pt 5pt 5pt 5pt, clip, width=0.33\textwidth]{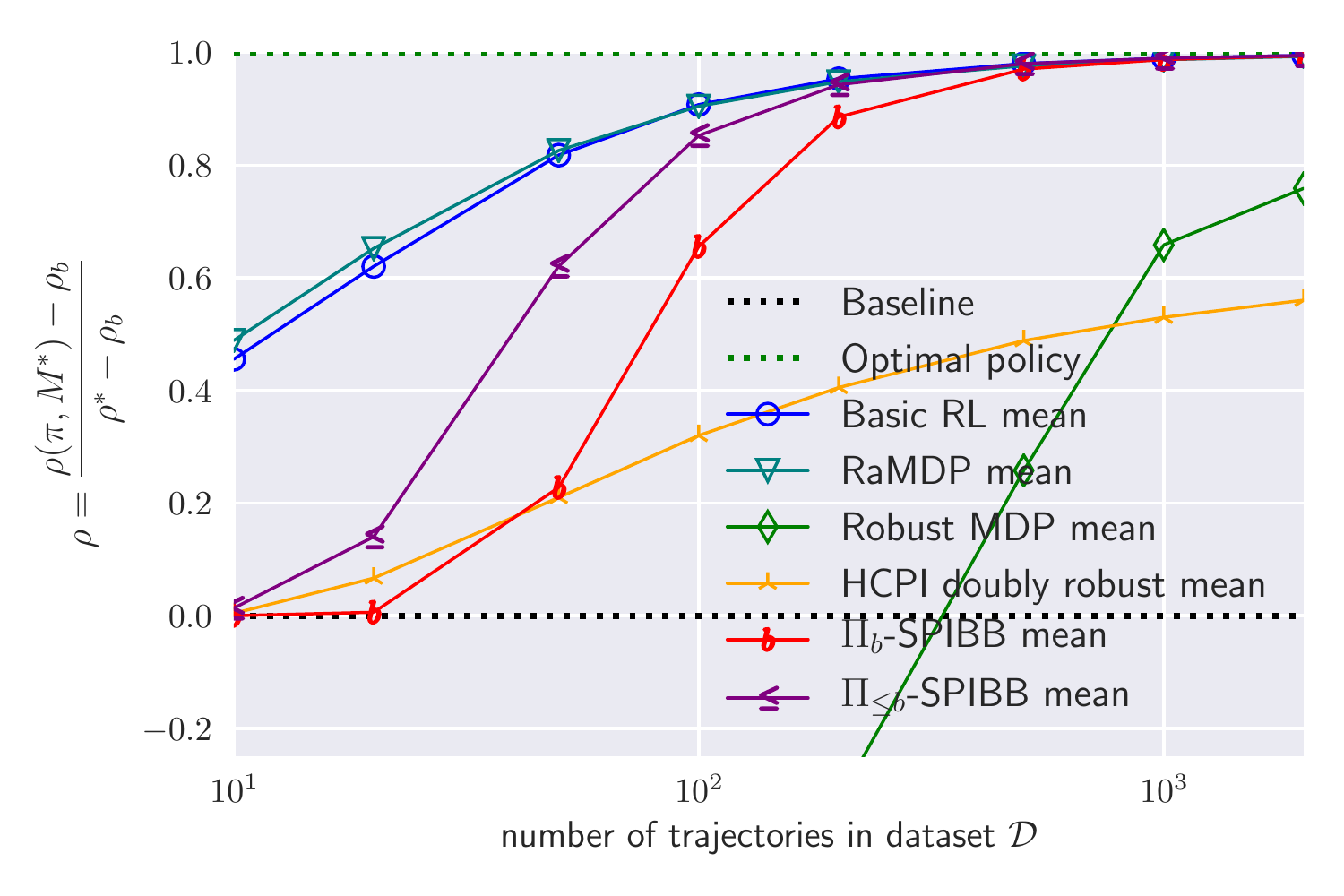}
			\label{fig:random_MDPs_mean_eta=0.3}
		}  \\
		\centering
		\subfloat[1\%-CVaR: with $\eta=0.5$ and $N_\wedge=10$.]{
			\includegraphics[trim = 5pt 5pt 5pt 5pt, clip, width=0.33\textwidth]{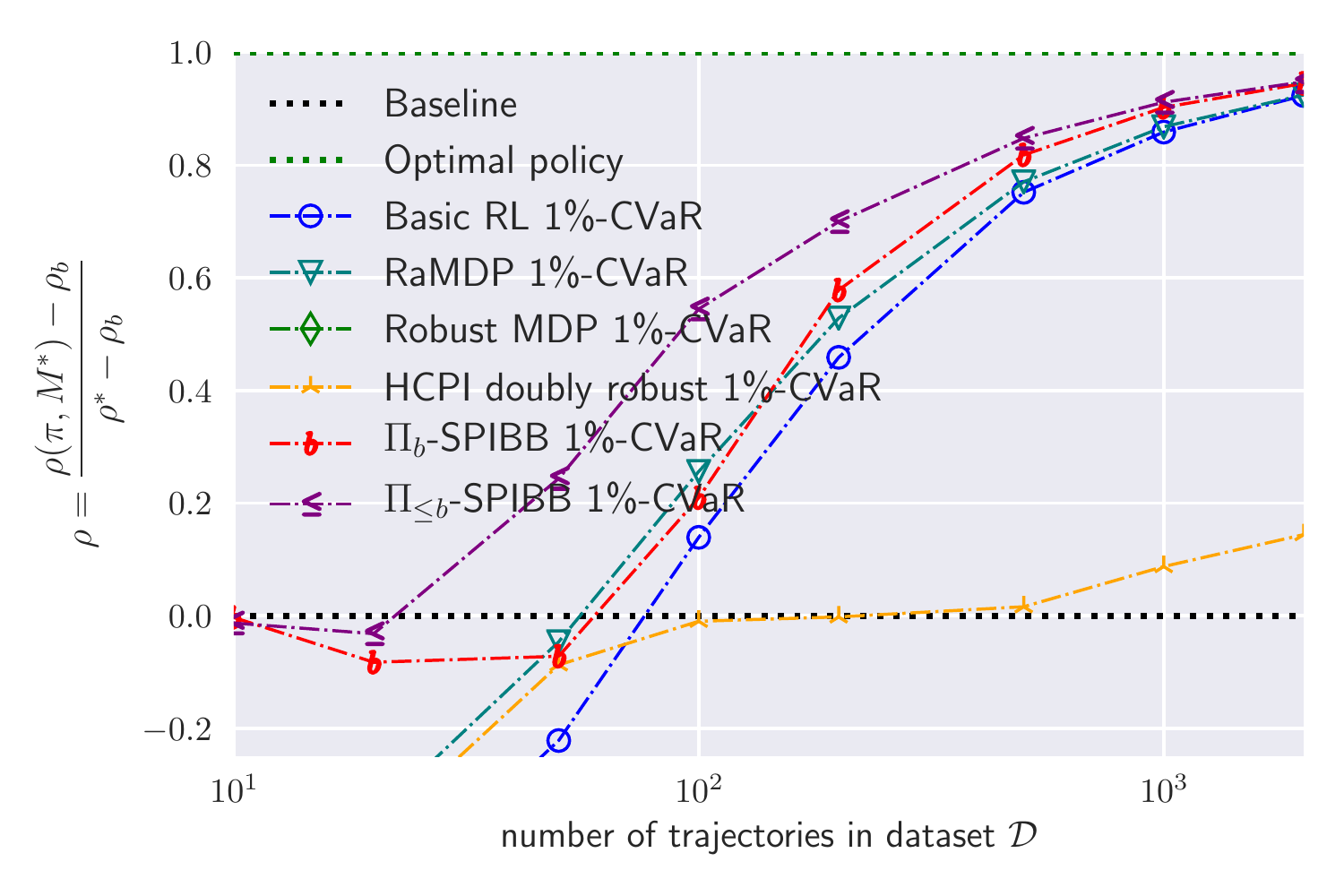}
			\label{fig:random_MDPs_1-CVaR_eta=0.5}
		} 
		\subfloat[10\%-CVaR: with $\eta=0.5$ and $N_\wedge=10$.]{
			\includegraphics[trim = 5pt 5pt 5pt 5pt, clip, width=0.33\textwidth]{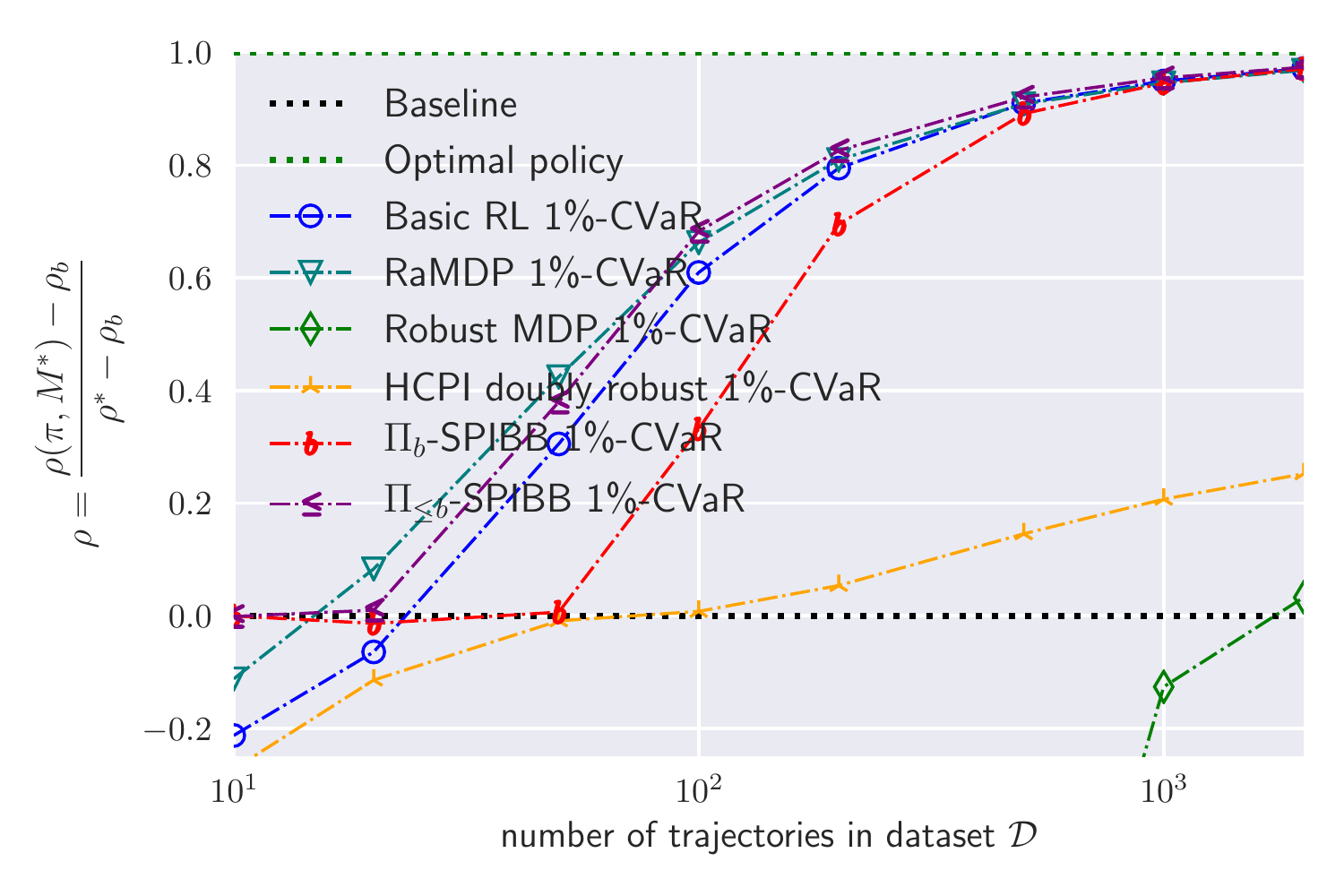}
			\label{fig:random_MDPs_10-CVaR_eta=0.5}
		}
		\subfloat[Mean: with $\eta=0.5$ and $N_\wedge=10$.]{
			\includegraphics[trim = 5pt 5pt 5pt 5pt, clip, width=0.33\textwidth]{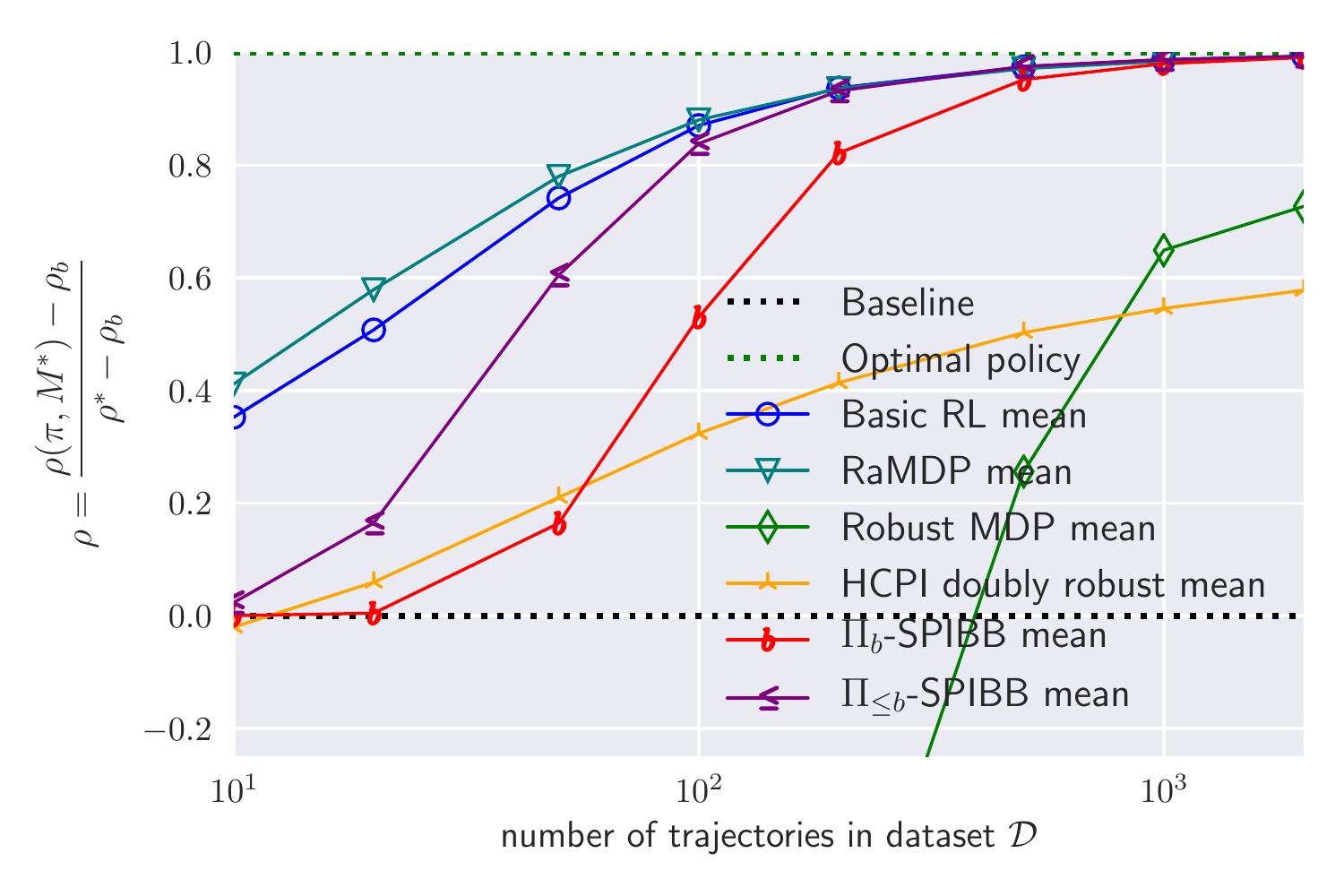}
			\label{fig:random_MDPs_mean_eta=0.5}
		} \\
		\centering
		\subfloat[1\%-CVaR: with $\eta=0.7$ and $N_\wedge=10$.]{
			\includegraphics[trim = 5pt 5pt 5pt 5pt, clip, width=0.33\textwidth]{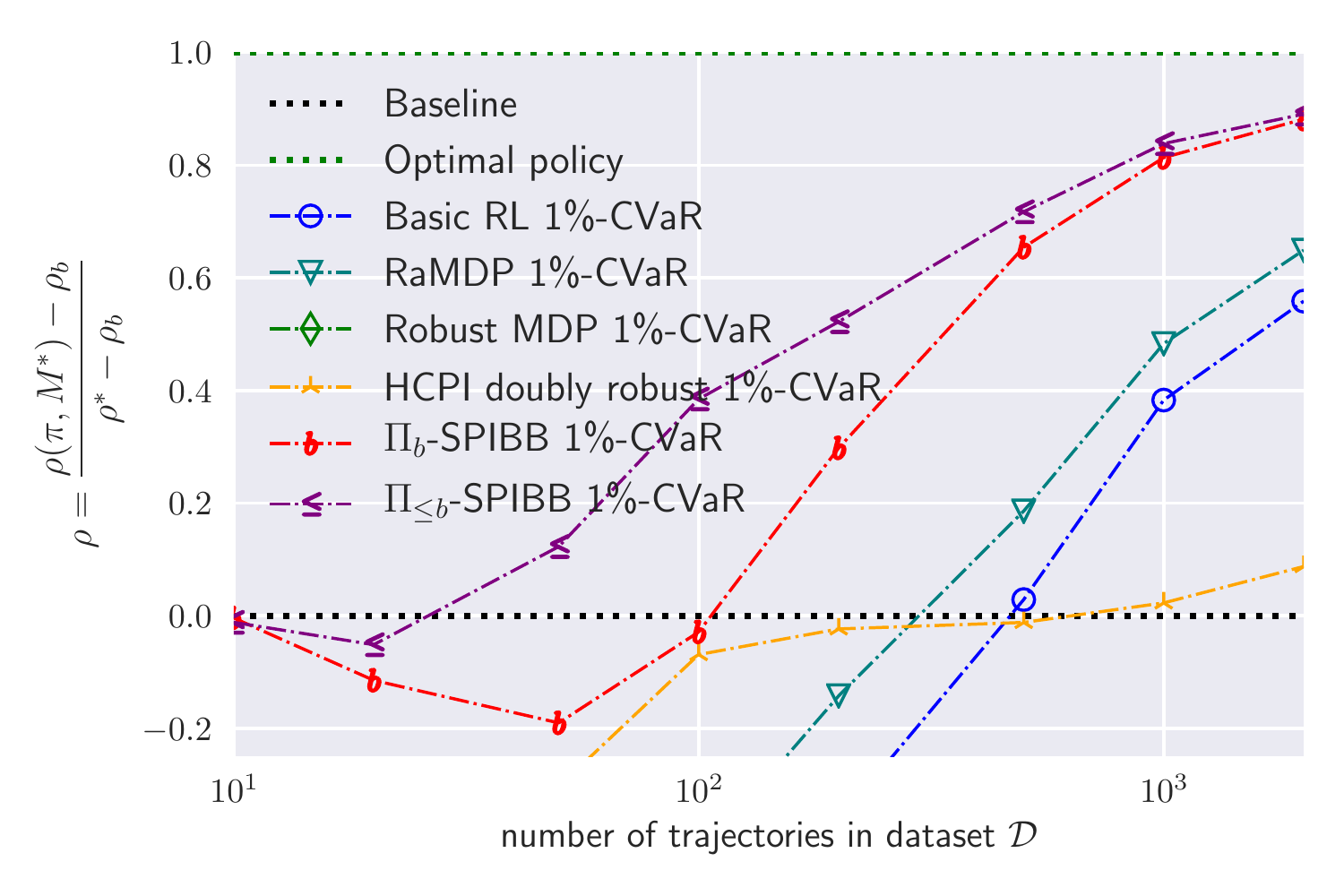}
			\label{fig:random_MDPs_1-CVaR_eta=0.7}
		} 
		\subfloat[10\%-CVaR: with $\eta=0.7$ and $N_\wedge=10$.]{
			\includegraphics[trim = 5pt 5pt 5pt 5pt, clip, width=0.33\textwidth]{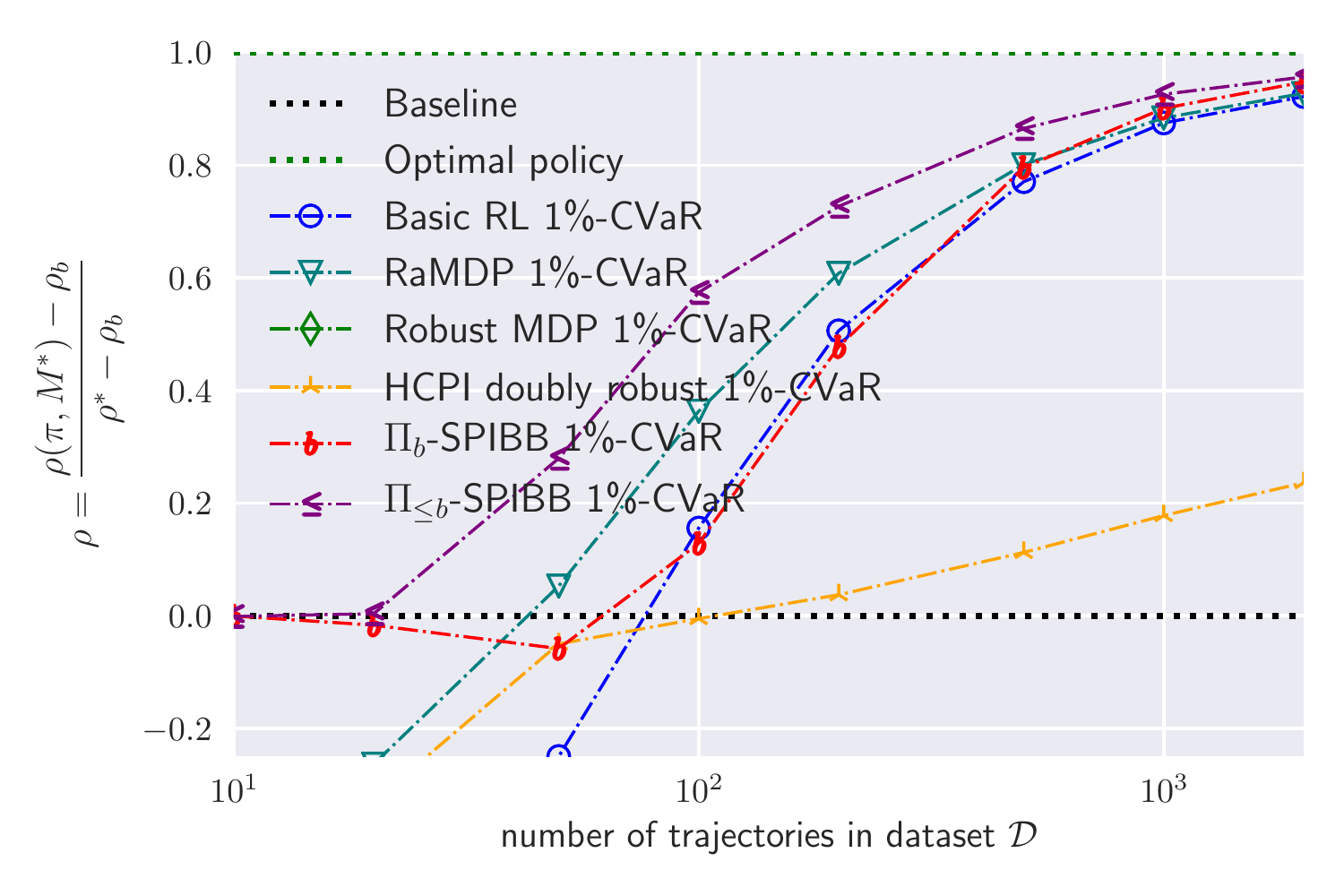}
			\label{fig:random_MDPs_10-CVaR_eta=0.7}
		}
		\subfloat[Mean: with $\eta=0.7$ and $N_\wedge=10$.]{
			\includegraphics[trim = 5pt 5pt 5pt 5pt, clip, width=0.33\textwidth]{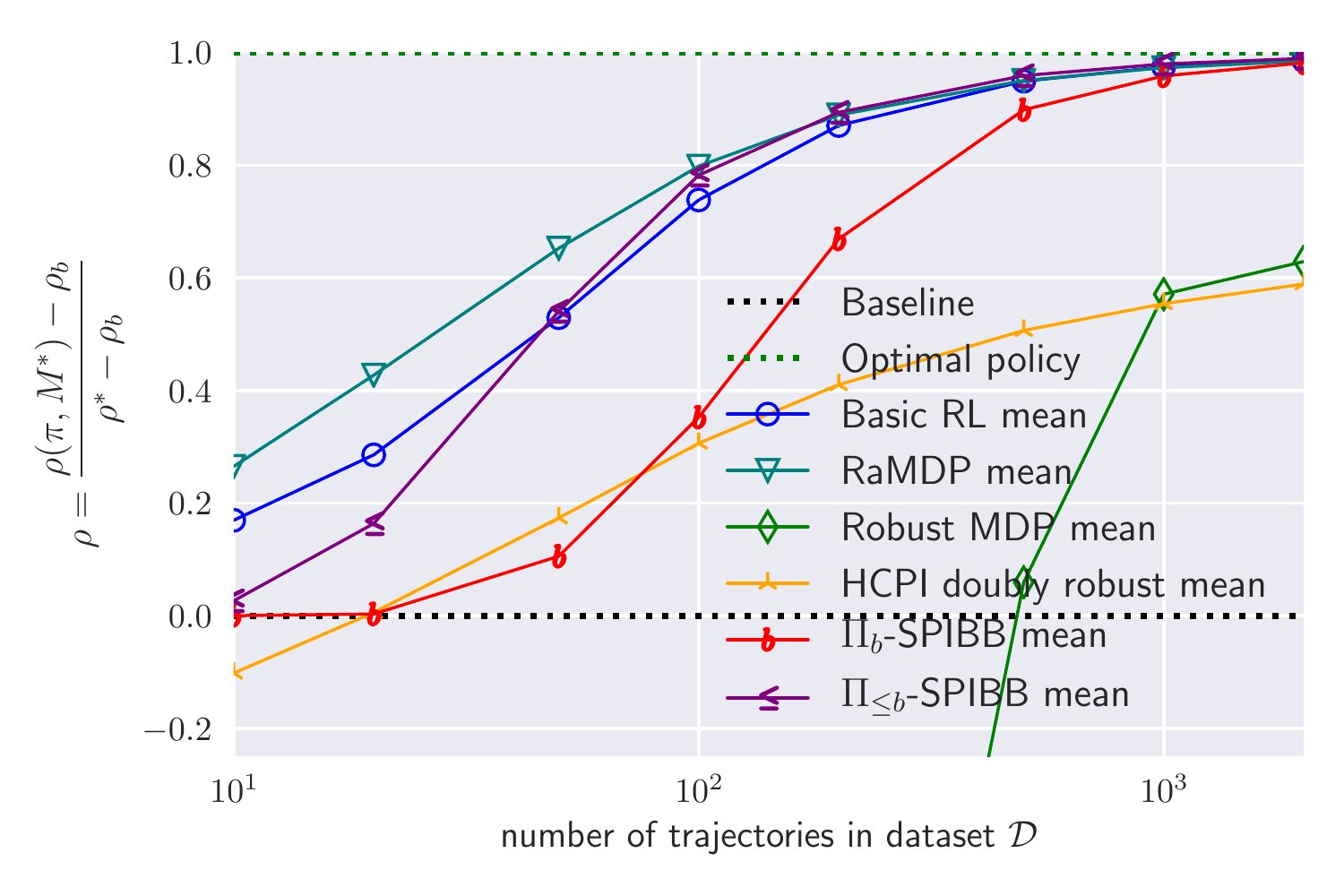}
			\label{fig:random_MDPs_mean_eta=0.7}
		}
		\caption{Random MDPs: 1\%-CVaR, 10\%-CVaR, and mean performance benchmarks for various $\eta$ values: respectively 0.1 (a-b), 0.9 (c), 0.3 (d-f), 0.5 (g-i), and 0.7 (j-l). The missing figures for $\eta=0.1$ and $\eta=0.9$ are in the main document. Figures for additional $\eta$ values: 0.2, 0.4, 0.6, and 0.8 may be found in the supplementary material package. The abscissae is the dataset size, the ordinate is the normalized performance.}
		\label{fig:randomMDP_sup_curves}
		\vspace{-10pt}
	\end{figure*}
	
	\clearpage
	
	\section{Helicopter Experiment Details}
	\label{sup:expe_helicopter}
	\subsection{Details about the helicopter environment}
	\label{sup:dummy-parameters}
	We consider the following helicopter environment, where:
	\begin{itemize}
	    \item The non terminal state space is the cross product of four features:
        	\begin{itemize}
        	    \item the abscissa position $s_x \in (0,1)$,
        	    \item the ordinate position $s_y \in (0,1)$,
        	    \item the abscissa velocity $v_x \in (-1,1)$,
        	    \item the ordinate velocity $v_y \in (-1,1)$,
        	    \item and the initial state is uniformly sampled in $(0,\frac{1}{3})\times(0,\frac{1}{3})\times(-1,1)\times(-1,1)$.
        	\end{itemize}
        \item The action set is a discrete thrust along each dimension:
        	\begin{itemize}
        	    \item the abscissa thrust $a_x \in \{-1,0,1\}$,
        	    \item and the ordinate thrust $a_y \in \{-1,0,1\}$.
        	\end{itemize}
        \item The transition function is independently applied on each dimension:
        	\begin{itemize}
        	    \item $s_x(t+1) = s_x(t) + v_x(t)\tau + \frac{1}{2}a_x(t) \tau^2 + \Gamma(0,\sigma_s)$,
        	    \item $s_y(t+1) = s_y(t) + v_y(t)\tau + \frac{1}{2}a_y(t) \tau^2 + \Gamma(0,\sigma_s)$,
        	    \item $v_x(t+1) = v_x(t) + a_x(t)\tau + \Gamma(0,\sigma_v)$,
        	    \item $v_y(t+1) = v_y(t) + a_y(t)\tau + \Gamma(0,\sigma_v)$,
        	    \item where $\tau=0.1$ is the time-step, $\Gamma(0,\sigma)$ is a centered Gaussian noise with standard deviation $\sigma$, $\sigma_s=0.025$ is the position-wise noise standard deviation, $\sigma_v=0.05$ is the velocity-wise noise standard deviation.
        	\end{itemize}
	    \item The reward function is set to:
        	\begin{itemize}
        	    \item $r(t) = 0$ in every non-terminal state,
        	    \item $r(t) = -1$ when one of the velocity features gets out of $(-1,1)$: the motor melts and the episode terminates,
        	    \item $r(t) = \min\left(10, \max\left(-1, \frac{1}{\sqrt{(s_x - 1)^2 + (s_y - 1)^2}} - 4\right)\right)$ when one of the position features leaves $(0,1)$: it landed and the episode terminates. It is good if it is close to the target coordinates $\{1,1\}$, bad otherwise, see Figure \subref*{fig:helicopter} for a visual representation of this final reward.
        	\end{itemize}
	    \item For the evaluation, similarly to what is commonly used in Atari or Go, the return is not discounted. Although, as next section specifies, the training of the SPIBB-DQN agents requires to set a discount factor lower than 1.
	\end{itemize}

	\subsection{Details about the experimental design}
	See Algorithm \ref{alg:helicopter_benchmark}.
	\begin{pseudocode}[ht!]
		\caption{Helicopter experimental process}
		\KwIn{List of hyper-parameter values for $N_\wedge$}
		\KwIn{List of dataset sizes}
% 		\KwIn{Randomized baseline policy $\pi_b$}
		\BlankLine
		\RepTimes{$20$}{
		    \For{each dataset size}{
		            
	            Generate a dataset.
	            
	            Compute the pseudo-counts.
	            
	            \RepTimes{$15$}{
		            
                    \For{each $N_\wedge$}{
                    
    		            Train a policy. ($N_\wedge=0$ amounts to vanilla DQN, and $N_\wedge= \infty$ amounts to reproducing the baseline)
    		                    
                        Evaluate the trained policy.
    		                    
                        Record the performance of the trained policy.
    		        }
    		    }
		    }
		}
		
		\label{alg:helicopter_benchmark}
	\end{pseudocode}
	\subsection{Details about the DQN and SPIBB-DQN implementations}
	\label{sup:dqn-parameters}
    The batch version of DQN simply consists in replacing the experience replay buffer by the dataset we are training on. Effectively, we are not sampling from the environment anymore but from the transitions collected a priori following the baseline. The same methodology applies for SPIBB, except that the targets we are using for our $Q$-values update verify the following modified Bellman equation:
    \begin{align}
	    y^{(i)}_j &= r_j + \gamma \max_{\pi\in\Pi_b} \sum_{a'\in\mathcal{A}} \pi(a'|x_j') Q^{(i)}(x_j',a') \nonumber \\
	    &= r_j + \gamma\sum_{a'|(x_j',a')\in\mathfrak{B}} \pi_b(a'|x_j') Q^{(i)}(x_j',a') + \gamma\left(\sum_{a'|(x_j',a')\notin\mathfrak{B}} \pi_b(a'|x_j')\right) \max_{a'|(x_j',a')\notin\mathfrak{B}} Q^{(i)}(x_j',a') \label{eq:spibb-DQN-app} \nonumber
	\end{align}
    We notice in particular that when $\mathfrak{B} = \emptyset$ the targets fall back to the traditional Bellman ones.
    We used the now classic target network trick \cite{Mnih2015}, combined with Double-DQN \cite{HasseltGS15}.

    The network used for the baseline and for the algorithms in the benchmark is a fully connected network with 3 hidden layers of $32$, $128$ and $32$ neurons, initialized using he\_uniform \cite{he2015delving}. The network has $9$ outputs corresponding to the $Q$-values of the $9$ actions in the game.
    We train the $Q$-networks with RMSProp \cite{tieleman2012lecture} with a momentum of $0.95$ and $\epsilon = 10^{-7}$ on mini-batches of size $32$. The learning rate is initialized at $0.01$ and is annealed every $20$k transitions or every pass on the dataset, whichever is larger. The networks are trained for $2$k passes on the dataset, and are fully converged by that time. We use the Keras framework \cite{chollet2015keras} with Tensorflow \cite{tensorflow2015-whitepaper} as backend. The policy is tested for $10$k steps at the end of training, with the initial states of each trajectory sampled as described in Section \ref{sup:dummy-parameters}.

	\begin{figure*}[t]
		\centering
		\subfloat[SPIBB-DQN with a single dataset in function of $N_\wedge$.]{
			\includegraphics[trim = 5pt 5pt 5pt 5pt, clip, width=0.5\textwidth]{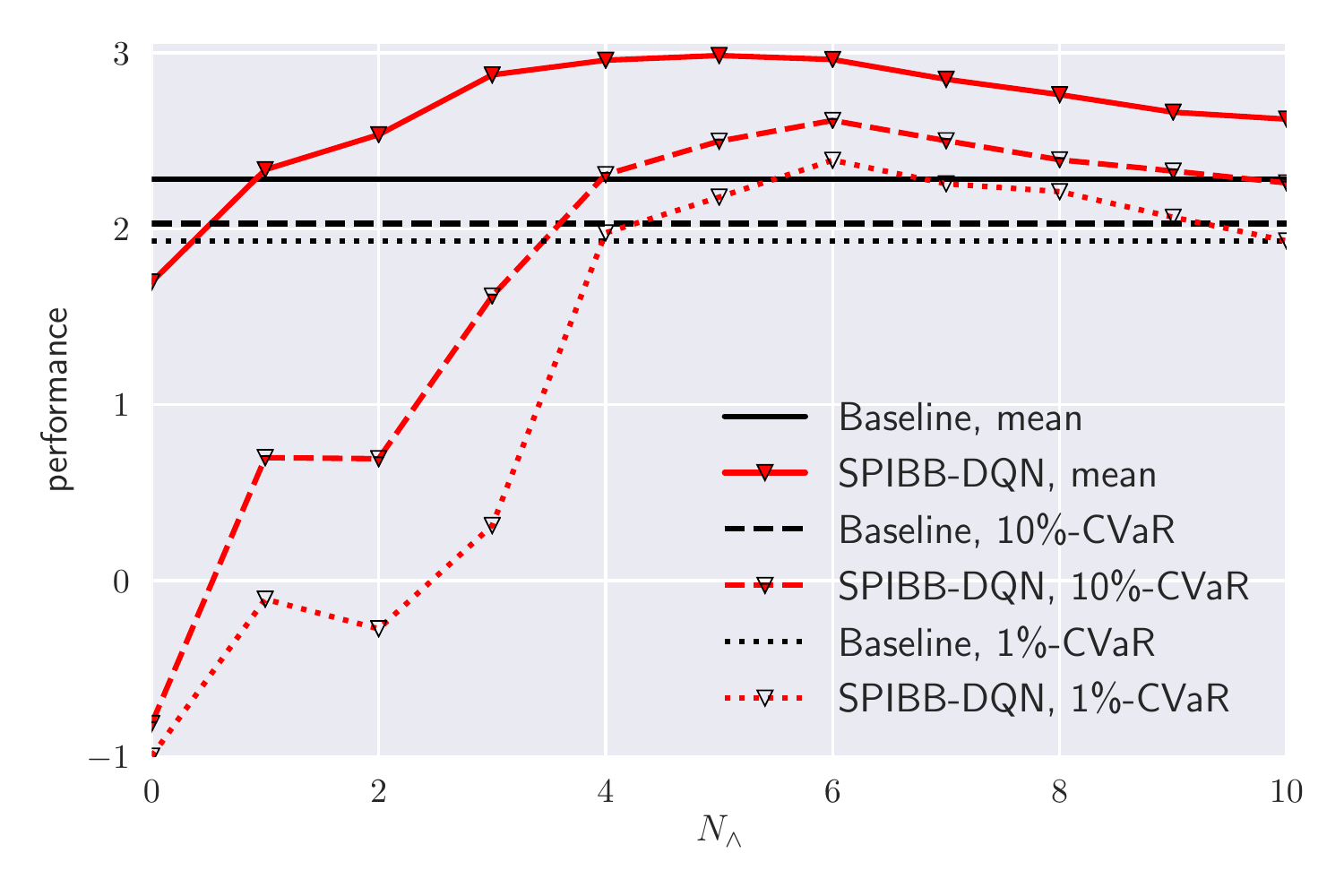}
			\label{fig:single-dataset-spibb-dqn}
		}
		\subfloat[RaMDP hyper-parameter $\kappa_{adj}$ search.]{
			\includegraphics[trim = 5pt 5pt 5pt 5pt, clip, width=0.5\textwidth]{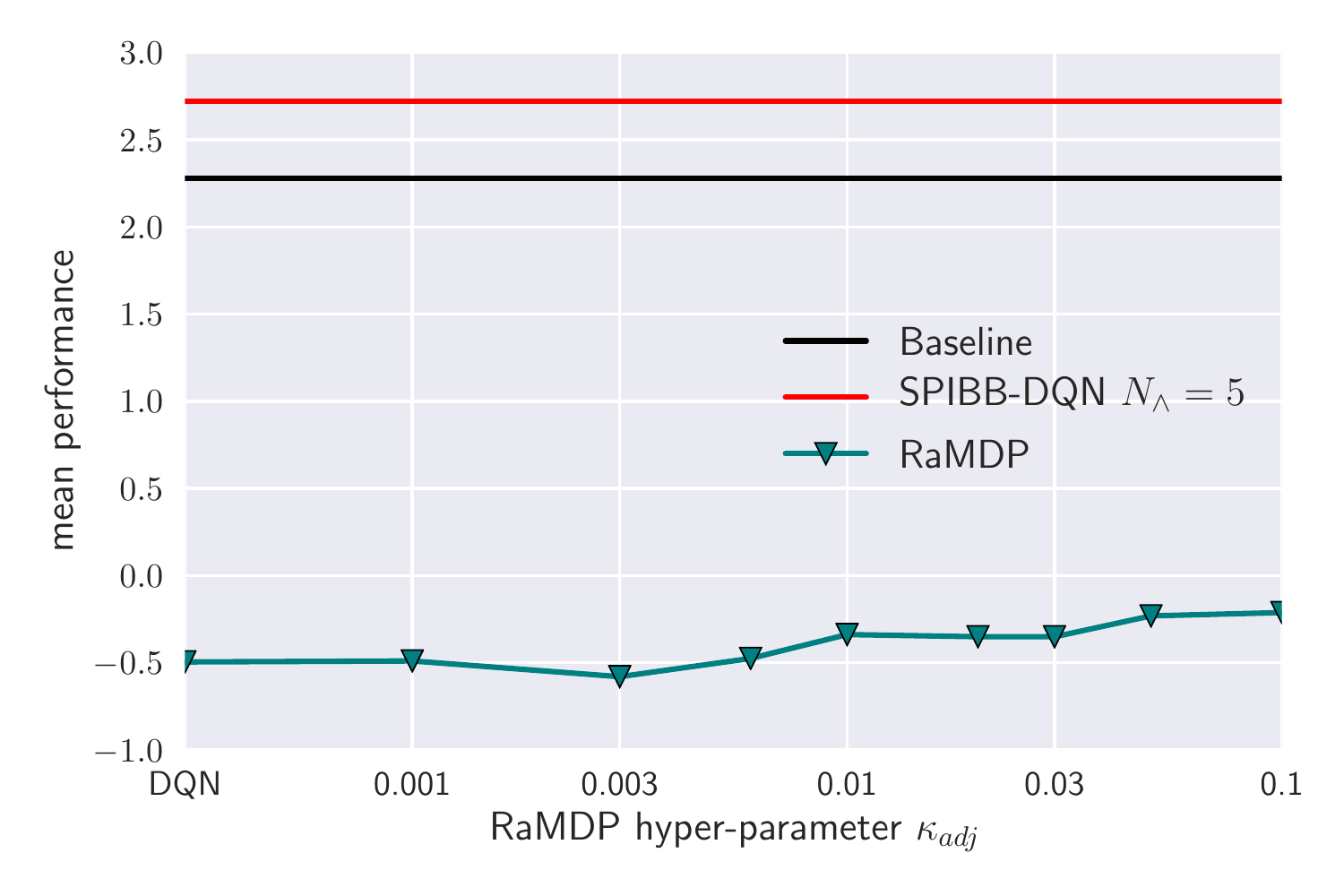}
			\label{fig:RaMDP-DQN-kappa}
		}
		\caption{Robust MDP hyper-parameter search results on the Gridworld domain.}
		\label{fig:preliminary-helicopter}
	\end{figure*}
	
	\subsection{Preliminary SPIBB-DQN experiments}
	\label{sup:DQN-figures}
	Before starting the experiments reported in the main document, Section \ref{sec:spibb-dqn-exp}, we led preliminary experiments with a single 10k-transition dataset. We found out, and report on Figure \subref*{fig:single-dataset-spibb-dqn}, that vanilla DQN trains very different $Q$-networks and therefore very different policies depending on the random seed, which influences the random initialization of the parameters of the network and the transitions sampled for the stochastic gradient. It is worth noticing a posteriori that this dataset was actually favorable to DQN on average (mean performance of 1.7 on this dataset vs. -0.5 reported in the main document), but that the reliability of DQN is still very low. In contrast, SPIBB-DQN shows stability for $N_\wedge\geq 4$.
	
	We also performed a hyper-parameter search on RaMDP on 10k-transition datasets. Given that the reward / value function amplitude is larger than on our previous experiments (Gridworld and Random MDPs), we expect the optimal $\kappa_{adj}$ value to be also larger than $0.003$. We thus considered the following hyper-parameter values: $\kappa_{adj}\in\left\{0.001, 0.003, 0.006,0.01, 0.02, 0.03, 0.05, 0.1\right\}$. To reduce the computational load, we only performed 75 runs per $\kappa_{adj}$ value. We also added $\kappa_{adj}=0$, which amounts to vanilla DQN. Figure \subref*{fig:RaMDP-DQN-kappa} shows that, although it slightly improves the DQN abysmal performance, the RaMDP performance is very limited, far under the baseline.
	
	\clearpage
	\section{Reproducible, Reusable, and Robust Reinforcement Learning}
	This paper's objective is to improve the robustness and the reliability of Reinforcement Learning algorithms. Inspired from Joelle Pineau's talk at NeurIPS 2018 about reproducible, reusable, and robust Reinforcement Learning\footnote{https://nips.cc/Conferences/2018/Schedule?showEvent=12486}, we intend to also make our work reusable and reproducible.
	
    \newlist{arrowlist}{itemize}{1}
    \setlist[arrowlist]{label=$\Rightarrow$}
	\subsection{Pineau's checklist (slide 33)}
	For all algorithms presented, check if you include:
	\begin{itemize}
	    \item A clear description of the algorithm.
            \begin{arrowlist}
                \item See Algorithm \ref{alg:Pibproj} for $\Pi_b$-SPIBB, Algorithm \ref{alg:pileqbproj} for $\Pi_{\leq b}$-SPIBB, and Equation \ref{eq:spibb-DQN} for SPIBB-DQN.
            \end{arrowlist}
	    \item An analysis of the complexity (time, space, sample size) of the algorithm.
            \begin{arrowlist}
                \item We do not provide formal analysis for the complexity of the finite MDP SPIBB algorithms as it depends on the policy iteration implementation, but it can be said that the complexity increase in comparison with standard policy iteration is insignificant: it does not change neither the order of magnitude nor the multiplying constant. For SPIBB-DQN, the pseudo-count computation may increase significantly the complexity of the algorithm. It is once more impossible to formally analyze since it depends on the pseudo-count implementation.
            \end{arrowlist}
	    \item A link to downloadable source code, including all dependencies.
            \begin{arrowlist}
                \item We provide all the code on github at these addresses: \url{https://github.com/RomainLaroche/SPIBB} and \url{https://github.com/rems75/SPIBB-DQN}. See Section \ref{sup:code}.
            \end{arrowlist}
	\end{itemize}
	
    For any theoretical claim, check if you include:
	\begin{itemize}
    	\item A statement of the result.
            \begin{arrowlist}
                \item See Theorems \ref{th:pib-spibb-convergence}, \ref{th:safepolicyimprovement-pi}, and \ref{th:free}.
            \end{arrowlist}
    	\item A clear explanation of any assumptions.
            \begin{arrowlist}
                \item See Sections \ref{sec:introduction} and \ref{sec:spibb}.
            \end{arrowlist}
    	\item A complete proof of the claim.
            \begin{arrowlist}
                \item See Section \ref{sup:proofs}.
            \end{arrowlist}
	\end{itemize}
	
    For all figures and tables that present empirical results, check if you include:
	\begin{itemize}
    	\item A complete description of the data collection process, including sample size.
            \begin{arrowlist}
                \item See Sections \ref{sec:results}, \ref{sup:MDPgen}, \ref{sup:baselinegen}, \ref{sup:datasetgen}, and \ref{sup:dummy-parameters}.
            \end{arrowlist}
    	\item A link to downloadable version of the dataset or simulation environment.
            \begin{arrowlist}
                \item See Section \ref{sup:code}.
            \end{arrowlist}
    	\item An explanation of how sample were allocated for training / validation / testing.
            \begin{arrowlist}
                \item The complete dataset is used for training. There is no need for validation set. Testing is performed in the true environment.
            \end{arrowlist}
    	\item An explanation of any data that was excluded.
            \begin{arrowlist}
                \item Does not apply to our simulated environments.
            \end{arrowlist}
    	\item The range of hyper-parameters considered, method to select the best hyper-parameter configuration, and specification of all hyper-parameters used to generate results.
            \begin{arrowlist}
                \item See Sections \ref{sec:results}, \ref{sup:MDPgen}, \ref{sup:baselinegen}, \ref{sup:datasetgen}, \ref{sup:benchmarkalgos}, and \ref{sup:dqn-parameters}.
            \end{arrowlist}
    	\item The exact number of evaluation runs.
            \begin{arrowlist}
                \item 100,000+ for finite MDPs experiments and 300 for SPIBB-DQN experiments.
            \end{arrowlist}
    	\item A description of how experiments were run.
            \begin{arrowlist}
                \item See Sections \ref{sec:results}, \ref{sup:expe_finite}, and \ref{sup:expe_helicopter}.
            \end{arrowlist}
    	\item A clear definition of the specific measure or statistics used to report results.
            \begin{arrowlist}
                \item Mean and X\% conditional value at risk (CVaR), described in Sections \ref{sec:results} and \ref{sup:evaluationgen}.
            \end{arrowlist}
    	\item Clearly defined error bars.
            \begin{arrowlist}
                \item Given the high number of runs we considered, the error bar are too thin to be displayed. Any difference visible with the naked eye is significant. We use CVaR everywhere instead to account for the uncertainty.
            \end{arrowlist}
    	\item A description of results including central tendency (e.g. mean) and variation (e.g. stddev).
            \begin{arrowlist}
                \item All our work is motivated and analyzed with respect to this matter.
            \end{arrowlist}
    	\item The computing infrastructure used.
            \begin{arrowlist}
                \item For the finite-MDPs experiment, we used clusters of CPUs. The full results were obtained by running the benchmarks with 100 CPUs running independently in parallel during 24h. For the helicopter experiment, we used a GPU cluster. However, only one GPU is necessary for a single run. Using a cluster allowed to launch several runs in parallel and considerably sped up the experiment. On a single GPU (a GTX 1080 Ti), a dataset of $|\mathcal{D}| = 10$k transitions is generated in $~5$ seconds. The dataset generation scales linearly in $|\mathcal{D}|$. Computing the counts for that dataset takes approximately $20$ minutes, it scales quadratically with the size of the dataset. As far as training is concerned, $2000$ passes on a dataset of $10$k transitions takes around $25$ minutes, it scales linearly in $N$. Finally, evaluation of the trained policy on $10$k trajectories takes $15$ minutes. It scales linearly in $|\mathcal{D}|$ as it requires the computation of the pseudo-count for each state encountered during the evaluation and this pseudo-count computation is linear in $|\mathcal{D}|$. Overall, a single run for a dataset of $10$k transitions takes around one hour.
            \end{arrowlist}
	\end{itemize}
	
	\subsection{Code attached to the submission}
	\label{sup:code}
    The attached code can be used to reproduce the experiments presented in the submitted paper. It is split into two projects: one for finite MDPs (Sections \ref{sec:gridworld_exp}, \ref{sec:gridworld_exp_random_baseline}, and \ref{sec:RandomMDPs_exp}), and one for SPIBB-DQN (Section \ref{sec:spibb-dqn-exp}).
    
    \subsubsection{Finite MDPs}
	\label{sup:code_finite}
	Found at this address: \url{https://github.com/RomainLaroche/SPIBB}.
    
    \paragraph{Prerequisites}
    The finite MDP project is implemented in Python 3.5 and only requires *numpy* and *scipy*.
    
    \paragraph{Content}
    We include the following:
    \begin{itemize}
        \item Libraries of the following algorithms:
            \begin{itemize}
                \item Basic RL,
                \item SPIBB:
                    \begin{itemize}
                        \item $\Pi_b$-SPIBB,
                        \item $\Pi_{\leq b}$-SPIBB,
                    \end{itemize}
                \item HCPI:
                    \begin{itemize}
                        \item doubly-robust,
                        \item importance sampling,
                        \item weighted importance sampling,
                        \item weighted per decision IS,
                        \item per decision IS,
                    \end{itemize}
                \item Robust MDP,
                \item and Reward-adjusted MDP.
            \end{itemize}
        \item Environments:
            \begin{itemize}
                \item Gridworld environment,
                \item Random MDPs environment.
            \end{itemize}
        \item Gridworld experiment of Section \ref{sec:gridworld_exp}. Run:
        
    		\texttt{python gridworld\_main.py \#name\_of\_experiment\# \#random\_seed\#}
    		
        \item Gridworld experiment with random behavioural policy of Section \ref{sec:gridworld_exp_random_baseline}. Run: 
        
    		\texttt{python gridworld\_random\_behavioural\_main.py \#name\_of\_experiment\# \#random\_seed\#}
    		
        \item Random MDPs experiment of Section \ref{sec:RandomMDPs_exp}. Run: 
        
    		\texttt{python randomMDPs\_main.py \#name\_of\_experiment\# \#random\_seed\#}
    		
    \end{itemize}

    \paragraph{Not included} We DO NOT include the following:
    \begin{itemize}
        \item The hyper-parameter search (Appendix \ref{sup:benchmarkalgos}): it should be easy to re-implement.
        \item The figure generator: it has too many specificities to be made understandable for a user at the moment. Also, it is not hard to re-implement with one's own visualization tools.
    \end{itemize}
    
    \paragraph{License} This project is BSD-licensed.
    
    \subsubsection{SPIBB-DQN}
	\label{sup:code_spibb_dqn}
	Found at this address: \url{https://github.com/rems75/SPIBB-DQN}.

    \paragraph{Prerequisites}
    SPIBB-DQN is implemented in Python 3 and requires the following libraries: Keras, Tensorflow, pickle, glob, yaml, argparse, numpy, yaml, pathlib, csv, scipy and click.
    
    \paragraph{Content}
    The SPIBB-DQN project contains the helicopter environment, the baseline used for our experiments and the code required to generate datasets and train vanilla DQN and SPIBB-DQN.
    
    \paragraph{Commands}

    To generate a dataset, use the following command:

     \texttt{python baseline.py baseline -{}-generate\_dataset -{}-dataset\_size 10000 -{}-dataset\_dir baseline/dataset -{}-seed 1}
        
    It will generate a dataset with 10000 transitions using the baseline defined in the baseline folder and save the dataset in the \texttt{baseline/dataset/10000/1/1\_0/dataset.pkl} folder. It will also compute the counts associated with each state-action pair in the dataset, and store those with the dataset in \texttt{baseline/dataset/10000/1/1\_0/dataset.pkl}.
    With other parameters, it creates a subfolder of the \texttt{dataset\_dir} you specify, the subfolder has the form: \texttt{{dataset\_size}/{seed}/{noise\_factor}} (\texttt{noise\_factor} is 1.0 by default, denoted as a 1\_0 folder).

    To train a policy using SPIBB-DQN with a parameter n\_wedge (denoted minimum\_count in the command) of 10, on a dataset generated following the method above, run the following command:

    \texttt{python train\_batch.py -{}-seed 1 -{}-dataset-path baseline/dataset/10000/1/1\_0/counts\_dataset.pkl  -{}-baseline-path baseline -{}-options minimum\_count 10}

    This will create, in the folder containing the dataset (\texttt{baseline/dataset/10000/1/1\_0} in that specific command), a csv file with the performance of the policy (one for each run on that dataset, 15 by default).

    To repeat the experiment, simply define a different seed for the dataset generation and train on that new dataset. The default values set in the code are the ones that produced the results from the paper. To run vanilla DQN, simply set the \texttt{minimum\_count} to 0.
    
    To run Reward-adjusted MDP on a dataset, simply add the following flag \texttt{--options learning\_type ramdp} and specify the value of kappa with e.g. \texttt{--options kappa 0.003}.
    
    \paragraph{Not included} We DO NOT include the following:
    \begin{itemize}
        \item The multi-CPU/multi-GPU implementation: its structure is too much dependent on the cluster tools. It would be useless for somebody from another lab.
    \end{itemize}
    
    \paragraph{License} This project is BSD-licensed.

\end{document}